\documentclass{article}

\usepackage{microtype}
\usepackage{graphicx}
\usepackage{subcaption}
\usepackage{booktabs} % for professional tables

\usepackage{hyperref}

\usepackage[preprint]{icml2026}

\usepackage{amsmath}
\usepackage{amssymb}
\usepackage{mathtools}
\usepackage{amsthm}
\usepackage{bm}

\usepackage[capitalize,noabbrev]{cleveref}

\theoremstyle{plain}
\newtheorem{theorem}{Theorem}[section]
\newtheorem{proposition}[theorem]{Proposition}
\newtheorem{lemma}[theorem]{Lemma}
\newtheorem{corollary}[theorem]{Corollary}
\theoremstyle{definition}

\newtheorem{assumption}[theorem]{Assumption}
\theoremstyle{remark}
\newtheorem{remark}[theorem]{Remark}

\usepackage[textsize=tiny]{todonotes}

\icmltitlerunning{Active Hypothesis Testing for Correlated Combinatorial Anomaly Detection}

\begin{document}

\twocolumn[
  \icmltitle{Active Hypothesis Testing for Correlated Combinatorial Anomaly Detection}



  \icmlsetsymbol{equal}{*}

  \begin{icmlauthorlist}
    \icmlauthor{Zichuan Yang}{sch}
    \icmlauthor{Yiming Xing}{sch}
  \end{icmlauthorlist}

  \icmlaffiliation{sch}{School of Mathematical Sciences, Tongji University, Shanghai, P.R. China}

  \icmlcorrespondingauthor{Zichuan Yang}{2153747@tongji.edu.cn}
  \icmlcorrespondingauthor{Yiming Xing}{yimingx4@tongji.edu.cn}

  \icmlkeywords{Combinatorial Pure Exploration, Active Hypothesis Testing, Multi-Armed Bandits, Optimal Experimental Design, Anomaly Detection, Cyber-Physical Systems}

  \vskip 0.3in
]



\printAffiliationsAndNotice{}  % no special notice (required even if empty)

\begin{abstract}
  We study the problem of identifying an anomalous subset of streams under correlated noise, motivated by monitoring and security in cyber-physical systems. This problem can be viewed as a form of combinatorial pure exploration, where each stream plays the role of an arm and measurements must be allocated sequentially under uncertainty.
  Existing combinatorial bandit and hypothesis testing methods typically assume independent observations and fail to exploit correlation for efficient measurement design.
  We propose ECC-AHT, an adaptive algorithm that selects continuous, constrained measurements to maximize Chernoff information between competing hypotheses, enabling active noise cancellation through differential sensing.
  ECC-AHT achieves optimal sample complexity guarantees and significantly outperforms state-of-the-art baselines in both synthetic and real-world correlated environments.
  The code is available on \url{https://github.com/VincentdeCristo/ECC-AHT}
\end{abstract}

\section{Introduction}
\label{sec:intro}

Modern monitoring systems often collect data from many streams at the same time.
Examples include sensor networks, cyber--physical systems, and large-scale online platforms \citep{pallakonda2025ai, odeyomi2025intrusion, shanthini2025graph, han2025timeseries, zhang2025high, lee2016introduction}.
In these settings, anomalies are rare, measurements are costly, and decisions must be made sequentially \citep{han2025timeseries, zhang2025high, feng2025false}.
As a result, the central challenge is not how to collect more data, but how to allocate a limited measurement budget in an adaptive and informative way.

A large body of prior work studies this problem through the lens of sequential testing \citep{wald1947sequential, chernoff1959sequential, garivier2016optimal} and combinatorial pure exploration \citep{chen2014combinatorial, huang2018combinatorial, jourdan2021efficient}.
These methods use adaptive strategies to solve structured decision problems in a sample-efficient manner \citep{nakamura2023combinatorial, russo2016simple, abbasi2011improved, agrawal2013thompson, lattimore2020bandit}.
To make the problem tractable, most existing approaches assume that observations across streams are independent.
This assumption simplifies analysis and algorithm design, and it underlies many successful methods in the literature \citep{gafni2023anomaly}. From a bandit perspective, each stream can be viewed as an arm, and the learner faces a combinatorial pure exploration problem with correlated feedback and a structured action space.

However, independence rarely holds in real systems \citep{srinivas2010gaussian, frazier2018bayesian}.
In practice, different streams often share common sources of noise.
Environmental factors, system-wide fluctuations, or correlated inputs can introduce strong dependencies across observations.
When correlation is present, algorithms that assume independence may waste measurements or make overly conservative decisions \citep{sen2017contextual, chen2021understanding}.
As a result, performance can degrade even when the signal itself is strong.

Our key insight is simple: \emph{correlation is not a nuisance but a resource}.
When two streams share the same noise, their difference removes common fluctuations and highlights the true signal.
If an algorithm can decide which differences to measure and when to measure them, it can resolve uncertainty much faster than methods that treat streams in isolation.
The challenge lies in using correlation adaptively, without knowing in advance which streams are anomalous.

We formalize this idea in a new algorithm called Efficient Champion--Challenger Active Hypothesis Testing (ECC-AHT).
The algorithm maintains beliefs over which streams are anomalous and updates them as new data arrive.
At each step, it selects measurements that focus on the most uncertain streams.
Crucially, it uses the correlation structure to construct differential measurements that cancel shared noise and amplify informative contrasts.
This design allows ECC-AHT to concentrate its budget on resolving the hardest distinctions first.

Figure~\ref{fig:inside_ecc_aht} illustrates this process on a toy example.
The figure shows how beliefs evolve over time, how the algorithm selects champion--challenger pairs, and how it reallocates measurements as uncertainty shrinks.
Later panels visualize the measurement vectors themselves and reveal how ECC-AHT aligns its actions with the underlying correlation structure.
Together, these views make clear that the algorithm follows a structured and interpretable strategy rather than a black-box heuristic.
We provide a detailed step-by-step analysis of this behavior in Appendix~\ref{app:interpretation}.

\begin{figure*}[t]
    \centering
    \begin{subfigure}[b]{0.32\textwidth}
        \includegraphics[width=\textwidth]{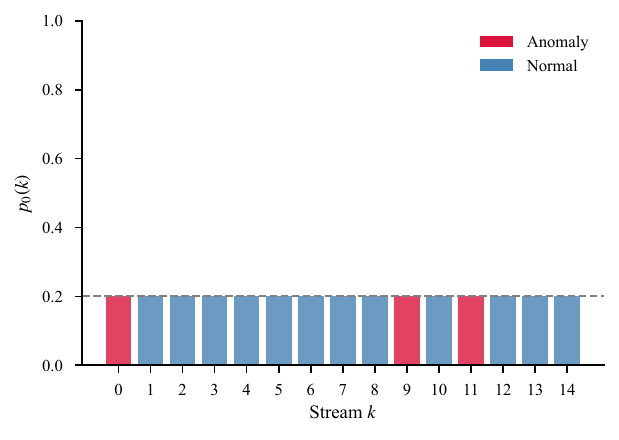}
        \caption{$t=0$: Initial Uniform Beliefs}
        \label{fig:inside_ecc_aht:a}
    \end{subfigure}
    \hfill
    \begin{subfigure}[b]{0.32\textwidth}
        \includegraphics[width=\textwidth]{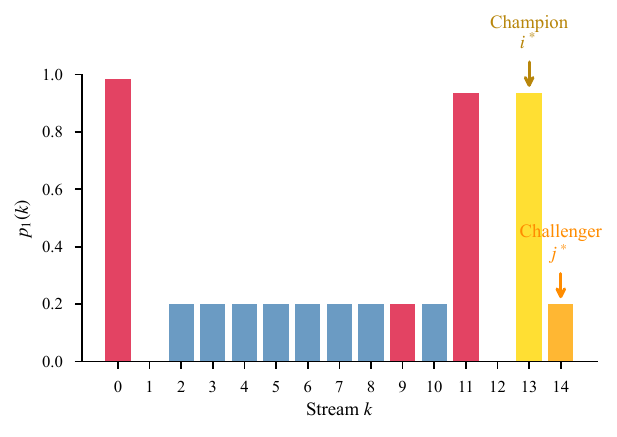}
        \caption{$t=1$: Algorithm started}
        \label{fig:inside_ecc_aht:b}
    \end{subfigure}
    \hfill
    \begin{subfigure}[b]{0.32\textwidth}
        \includegraphics[width=\textwidth]{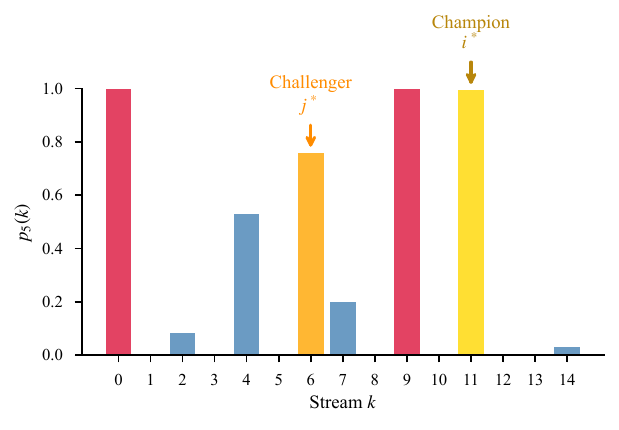}
        \caption{$t=5$: Beliefs Starting to Concentrate}
        \label{fig:inside_ecc_aht:c}
    \end{subfigure}
    \par
    \begin{subfigure}[b]{0.32\textwidth}
        \includegraphics[width=\textwidth]{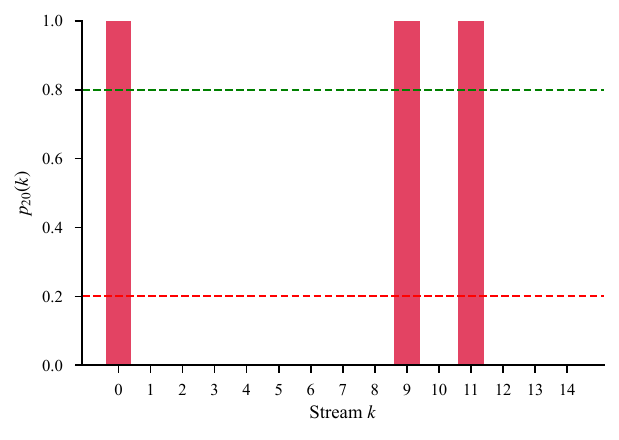}
        \caption{$t=20$: Clear Separation}
        \label{fig:inside_ecc_aht:d}
    \end{subfigure}
    \hfill
    \begin{subfigure}[b]{0.32\textwidth}
        \includegraphics[width=\textwidth]{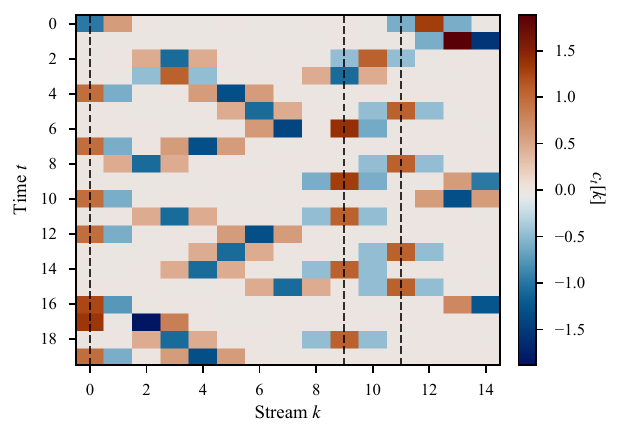}
        \caption{Action Trajectory: $c_t[k]$}
        \label{fig:inside_ecc_aht:e}
    \end{subfigure}
    \hfill
    \begin{subfigure}[b]{0.32\textwidth}
        \includegraphics[width=\textwidth]{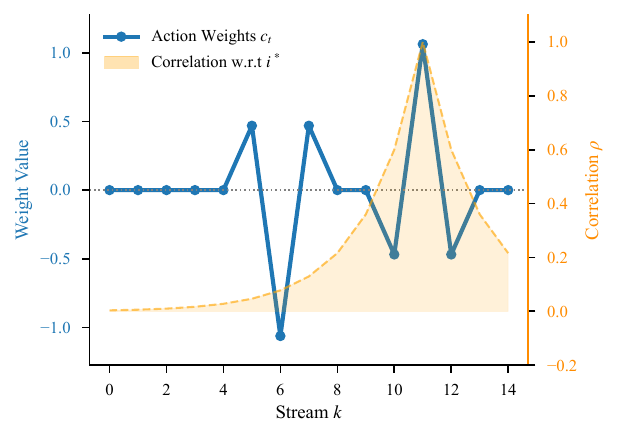}
        \caption{Correlation exploitation at $t=5$}
        \label{fig:inside_ecc_aht:f}
    \end{subfigure}
    \caption{
    Visualization of ECC-AHT dynamics on a toy instance ($K=15$, $n=3$, $\rho=0.6$).
    \textbf{(a)--(d)} Evolution of marginal beliefs $p_t(k)$ from uniform priors to the successful isolation of anomalies.
    \textbf{(e)} Heatmap of action weights $c_t[k]$.
    \textbf{(f)} Snapshot of the action vector $c_t$ at $t=5$ overlaying the correlation structure $\Sigma[i^\star, :]$, illustrating how the algorithm leverages correlation for variance reduction.
    See \textbf{Appendix~\ref{app:interpretation}} for a detailed interpretative analysis.
    }
    \label{fig:inside_ecc_aht}
\end{figure*}

This paper makes three main contributions.
First, we introduce a correlation-aware formulation for adaptive anomaly identification under strict budget constraints.
Second, we propose ECC-AHT, a practical algorithm that actively exploits correlation through adaptive measurement design.
Third, we provide theoretical guaranties and empirical evidence that explain when correlation improves performance and when it does not.
Across a wide range of settings, ECC-AHT achieves substantial reductions in sample complexity compared to state-of-the-art baselines.

\section{Related Work}
\label{sec:related}

Our work relates to several lines of research, including combinatorial pure exploration in bandits, structured anomaly detection, and sequential design of experiments.
While each of these areas addresses part of the problem we study, none of them directly resolves the challenge of \emph{adaptive subset identification under correlated noise with continuous measurement design}.
We discuss these connections below and clarify where ECC-AHT differs.

\subsection{Combinatorial Pure Exploration and Top-Two Methods}

Combinatorial Pure Exploration (CPE) studies the problem of identifying a subset of arms with a desired property, such as the top-$k$ arms~\citep{chen2014combinatorial}.
Early work focused on discrete arm pulls and independent noise.
Subsequent extensions improved efficiency and generalized feedback models~\citep{garivier2016optimal, jourdan2021efficient, sen2017contextual, gupta2021multi}.

Several recent methods achieve strong empirical and theoretical performance by adaptively comparing a leading hypothesis with its most competitive alternative.
Representative examples include CombGapE~\citep{nakamura2023combinatorial}, Top-Two Thompson Sampling (TTTS)~\citep{russo2016simple}, and Top-Two Transportation Cost (T3C)~\citep{shang2020fixed}.
These approaches share a common principle: they track a \emph{champion} and a \emph{challenger} and allocate samples to maximize information gain.

At a high level, ECC-AHT adopts a similar hypothesis-comparison strategy.
However, there are two fundamental differences.
First, existing CPE and Top-Two methods typically assume independent noise across arms.
When correlation is present, they either ignore it or treat it as a nuisance parameter.
Second, even when continuous actions are allowed, these methods restrict attention to scalar rewards without exploiting correlation for noise cancellation.

In contrast, ECC-AHT explicitly incorporates the covariance structure into its measurement design.
Rather than pulling arms individually, it constructs differential measurements that suppress shared noise.
This distinction is critical in highly correlated environments, where independence-based methods can be arbitrarily inefficient.
As a result, ECC-AHT is not a minor variant of CombGapE or TTTS, but a method designed for a different operating regime.

\subsection{Structured and Hierarchical Anomaly Detection}

A separate line of work approaches anomaly detection through structured search.
Graph-based and hierarchical methods exploit known dependencies to reduce search complexity~\citep{akoglu2015graph, pasqualetti2013attack}.
A notable example is Hierarchical Dynamic Search (HDS)~\citep{gafni2023anomaly}, which assumes a perfect tree structure and performs aggregate tests on internal nodes.
Under these assumptions, HDS achieves logarithmic sample complexity.

ECC-AHT addresses a related but distinct problem.
Unlike HDS, it does not assume a fixed hierarchical structure or a specific number of sensors.
This distinction is important in practice.
Real-world systems, such as industrial control networks, rarely satisfy idealized tree assumptions and often exhibit dense, unstructured correlations~\citep{ahmed2017wadi}.

More importantly, hierarchical methods primarily rely on summation-based aggregation.
They do not support differential measurements, which are essential for canceling common-mode noise.
Differential strategies have long been recognized as powerful tools in secure state estimation and attack detection~\citep{pasqualetti2013attack}.
ECC-AHT brings this principle into an adaptive sequential framework, allowing the algorithm to decide \emph{which differences to measure} based on uncertainty.

\subsection{Sequential Design of Experiments}

Our theoretical foundation builds on the classical theory of sequential design of experiments~\citep{chernoff1959sequential}.
Chernoff characterized optimal experiment allocation for hypothesis testing and introduced the notion of Chernoff information.
Later work connected these ideas to bandit problems and best-arm identification~\citep{garivier2016optimal}.

Several studies extended this framework to adaptive and correlated settings~\citep{honda2010asymptotically, jedra2020optimal}.
These works establish fundamental limits on how fast evidence can accumulate.
However, they do not directly address combinatorial subset identification or practical action constraints.

More recently, \citet{fiez2019sequential} studied sequential experimental design in linear bandits.
Their focus lies on reward maximization or simple arm selection rather than identifying a structured subset of hypotheses.
ECC-AHT differs in both goal and mechanism.
It targets subset identification and solves a constrained optimization problem at each step to maximize the relevant Chernoff information.
This design allows ECC-AHT to attain the optimal information rate while handling correlated noise and continuous measurement constraints in a unified framework.

In short, ECC-AHT combines ideas from CPE, Top-Two sampling, and sequential experimental design, but it operates in a regime that existing methods do not cover.
It is designed for correlated observations, continuous measurement actions, and combinatorial hypothesis spaces.
These features together distinguish ECC-AHT from prior work and explain its empirical and theoretical advantages.

\section{Methodology}
\label{sec:method}

This section presents the methodological foundation of ECC-AHT.
Our goal is to design a sequential decision procedure that identifies a small set of anomalous streams under correlated noise as efficiently as possible.
The core idea is simple.
When observations are correlated, the way we measure the system matters as much as how we update beliefs.
ECC-AHT makes this idea operational by combining correlation-aware measurement design with scalable inference and structured hypothesis comparison.

We proceed by first defining the problem and the sensing model.
We then explain how correlation shapes the information structure of the problem.
Based on this structure, we derive a sequential design strategy that focuses on the most informative hypothesis comparison.
Finally, we describe a scalable inference mechanism that preserves the ordering needed by the algorithm.

\subsection{Problem Formulation}

We consider a system with $K$ data streams indexed by $[K] = \{1,\ldots,K\}$.
At each time step, the system produces a vector-valued observation.
Under nominal operation, the observation follows a multivariate Gaussian distribution with mean $\bm{\mu}_0 \in \mathbb{R}^K$ and covariance matrix $\bm{\Sigma} \in \mathbb{R}^{K \times K}$.
The covariance matrix is known and captures correlation induced by shared physical conditions, sensing pipelines, or environmental noise.

A subset of streams $S^\star \subset [K]$ with fixed size $|S^\star| = n$ is anomalous.
Anomalous streams exhibit a mean shift specified by a known signal pattern $\bm{\delta} \in \mathbb{R}^K$.
The resulting mean under hypothesis $H_{S^\star}$ is
$
\bm{\mu}_{S^\star}
=
\bm{\mu}_0 + \displaystyle\sum_{k \in S^\star} \delta_k \mathbf{e}_k,
$
where $\mathbf{e}_k$ denotes the $k$-th canonical basis vector.
We assume $\bm{\delta}$ is known or calibrated in advance, as is standard in sequential testing with structured alternatives.
The task is to identify $S^\star$ with high confidence while minimizing the number of measurements.

Unlike classical bandit settings that sample individual streams, we allow the learner to design measurements.
At time $t$, the learner chooses a vector $\mathbf{c}_t \in \mathbb{R}^K$ and observes a scalar projection
$
y_t = \mathbf{c}_t^\top \mathbf{x}_t + \xi_t,
$
where $\mathbf{x}_t$ is the latent system state and $\xi_t$ is Gaussian noise with variance $\mathbf{c}_t^\top \bm{\Sigma} \mathbf{c}_t$.
To model sensing or energy constraints, we restrict the action space to
$
\mathcal{C} = \{ \mathbf{c} \in \mathbb{R}^K : \|\mathbf{c}\|_1 \le B \}.
$
This continuous action model allows the learner to form linear combinations of streams.
Such combinations are essential when noise is correlated.
They enable differential measurements that suppress shared noise while preserving signal contrast.

\subsection{Information Structure Under Correlated Noise}

The efficiency of sequential testing depends on how quickly the true hypothesis separates from competing alternatives.
This separation rate is governed by the Kullback--Leibler divergence induced by each measurement.
For two hypotheses $H_i$ and $H_j$ and a measurement vector $\mathbf{c}$, the divergence takes the form
\begin{equation}
\label{eq:kl}
D(H_i \Vert H_j \mid \mathbf{c})
=
\frac{
(\mathbf{c}^\top(\bm{\mu}_i - \bm{\mu}_j))^2
}{
2\,\mathbf{c}^\top \bm{\Sigma} \mathbf{c}
}.
\end{equation}

This expression reveals the central role of correlation.
The numerator captures how strongly the measurement separates the two hypotheses.
The denominator captures how much noise remains after projection.
When $\bm{\Sigma}$ has strong off-diagonal structure, naive measurements can be dominated by shared noise.
In contrast, carefully chosen linear combinations can significantly reduce the effective variance.
This observation motivates a shift in perspective.
Correlation is not a complication to ignore.
It is a resource that guides how measurements should be designed.

\subsection{From Combinatorial Search to Pairwise Testing}

The hypothesis space consists of all subsets of size $n$.
Its size grows combinatorially with $K$, which makes exhaustive comparison infeasible.
ECC-AHT addresses this challenge by focusing on the most informative local alternatives.

At each time step, the algorithm maintains marginal belief scores over streams.
Based on these scores, it constructs a candidate anomalous set $S_t$ by selecting the $n$ most likely streams.
Among these, it identifies the weakest member.
It then identifies the strongest candidate outside the set.
This defines a Challenger hypothesis that differs from the current Champion by a single swap.
Such local alternatives are not arbitrary.
In the asymptotic regime, the hardest competing hypotheses differ in exactly one element.
Focusing on these pairs therefore aligns computation with the intrinsic difficulty of the problem.

\subsection{Chernoff-Optimal Measurement Design}

Given a Champion--Challenger pair that differs by streams $i^\star$ and $j^\star$, the algorithm designs a measurement that best separates them.
The corresponding mean difference vector is
$
\bm{\Delta}_{i^\star j^\star}
=
\delta_{i^\star} \mathbf{e}_{i^\star}
-
\delta_{j^\star} \mathbf{e}_{j^\star}.
$
The goal is to choose $\mathbf{c}$ that maximizes the KL divergence in \eqref{eq:kl}.
This leads to the optimization problem
\begin{equation}
\label{eq:qp}
\max_{\mathbf{c} \in \mathcal{C}}
\;
\frac{(\mathbf{c}^\top \bm{\Delta}_{i^\star j^\star})^2}
{\mathbf{c}^\top \bm{\Sigma} \mathbf{c}}.
\end{equation}

This problem has a clear interpretation.
The numerator encourages alignment with the informative contrast.
The denominator penalizes residual noise after projection.
By fixing the numerator through a linear constraint, the problem reduces to minimizing the projected variance.
The resulting quadratic program is convex and can be solved efficiently.
The solution defines a noise-canceling measurement that extracts information at the fastest possible rate for the current ambiguity.

\subsection{Scalable Inference via Ranking Preservation}

Exact Bayesian inference over all subsets is computationally infeasible.
ECC-AHT avoids this bottleneck by observing that it does not require exact posterior probabilities.
It only needs the correct ordering of streams to identify candidate swaps.

To this end, the algorithm maintains per-stream log-likelihood ratios using a pseudo-likelihood approximation.
For each stream, we compare a local anomalous hypothesis against a nominal one while treating other streams as fixed.
This yields the update
\begin{equation}
\label{eq:pseudo_llr}
\ell_t(k)
=
\ell_{t-1}(k)
+
\log
\frac{
\mathcal{N}(y_t \mid \mathbf{c}_t^\top(\bm{\mu}_0 + \delta_k \mathbf{e}_k),\,
\mathbf{c}_t^\top \bm{\Sigma} \mathbf{c}_t)
}{
\mathcal{N}(y_t \mid \mathbf{c}_t^\top \bm{\mu}_0,\,
\mathbf{c}_t^\top \bm{\Sigma} \mathbf{c}_t)
}.
\end{equation}

Although this approximation ignores joint dependencies, it preserves the ranking of streams with high probability.
The reason is structural.
When multiple streams are anomalous, their joint contribution shifts all scores by a common term.
This shared shift does not affect relative ordering.
A formal finite-sample ranking consistency result is provided in Appendix~\ref{app:approximation}.
Since the measurement design explicitly maximizes the separation objective that governs this ranking, inference and control reinforce each other.

The marginal belief is then obtained via the logistic transform
\begin{equation}
\label{eq:pseudo_update}
p_t(k) = (1 + e^{-\ell_t(k)})^{-1}.
\end{equation}

Taken together, these components form ECC-AHT (as shown in Algorithm~\ref{alg:ecc-aht}).
The algorithm continuously focuses on the most ambiguous alternative, designs a measurement that cancels correlated noise, and updates beliefs in a scalable manner.
This integration allows ECC-AHT to achieve information-theoretically optimal rates while remaining practical for high-dimensional systems.

\begin{algorithm}[ht]
\caption{ECC-AHT: Efficient Champion--Challenger Active Hypothesis Testing}
\label{alg:ecc-aht}
\begin{algorithmic}[1]
\REQUIRE Number of streams $K$, anomaly budget $n$, covariance matrix $\bm{\Sigma}$, action budget $B$
\REQUIRE Known signal pattern $\bm{\delta}$, nominal mean $\bm{\mu}_0$
\STATE Initialize marginal beliefs $\{p_0(k)\}_{k=1}^K$ (e.g., uniform prior)
\STATE $t \gets 0$
\WHILE{stopping criterion not met} 
    \STATE \COMMENT{Stopping rule: a fixed-confidence GLR threshold is used for theoretical guarantees (see Appendix~\ref{app:stopping}); practical variants may use posterior-gap surrogates.}
    \STATE $t \gets t + 1$
    \STATE \textbf{Champion selection:}
    \STATE \hspace{0.5cm} $S_t \gets$ indices of the top-$n$ streams with largest $p_{t-1}(k)$
    \STATE \hspace{0.5cm} $i^\star \gets \displaystyle\arg\min_{i \in S_t} p_{t-1}(i)$
    \STATE \textbf{Challenger selection:}
    \STATE \hspace{0.5cm} $j^\star \gets \displaystyle\arg\max_{j \notin S_t} p_{t-1}(j)$
    \STATE Construct pairwise mean difference\hfill $\bm{\Delta}_{i^\star j^\star} \gets \delta_{i^\star} \mathbf{e}_{i^\star} - \delta_{j^\star} \mathbf{e}_{j^\star}$
    \STATE \textbf{Optimal experiment design:}
    \STATE \hspace{0.5cm} Solve
    \[
        \mathbf{c}_t
        \gets
        \displaystyle\arg\min_{\mathbf{c}}
        \quad \mathbf{c}^\top \bm{\Sigma} \mathbf{c}
        \quad
        \text{s.t. }
        \mathbf{c}^\top \bm{\Delta}_{i^\star j^\star} = 1,\;
        \|\mathbf{c}\|_1 \le B
    \]
    \STATE Observe $y_t = \mathbf{c}_t^\top \mathbf{x}_t + \xi_t$
    \STATE \textbf{Belief update:}
    \STATE \hspace{0.5cm} Update marginal beliefs $\{p_t(k)\}_{k=1}^K$ using the pseudo-likelihood update in equation~\eqref{eq:pseudo_update}
\ENDWHILE
\STATE \textbf{Return} $\hat S = S_t$
\end{algorithmic}
\end{algorithm}

\section{Theoretical Analysis}
\label{sec:theory}

This section explains why ECC-AHT works and clarifies the role of correlation in its performance.
Rather than viewing theory as an abstract guarantee, we use it to answer two concrete questions.
First, what is the fastest possible rate at which any algorithm can identify anomalies under correlated noise?
Second, how close does ECC-AHT come to this limit, and why?

\begin{remark}
    Throughout the theoretical analysis, we assume that the signal pattern $\bm{\delta}$ is known and fixed. This assumption is standard in Chernoff-style sequential testing and isolates the difficulty of experimental design under correlated noise. Handling unknown or heterogeneous signal strengths is left for future work.
\end{remark}

The core quantity that governs sample complexity is a single information rate, denoted by $\Gamma^\star$.
This rate depends on both the signal structure and the noise covariance.
Correlation matters because it directly increases $\Gamma^\star$ when the algorithm is allowed to design differential measurements.
ECC-AHT is effective because it actively learns and tracks this optimal rate.

\subsection{A Fundamental Speed Limit}

We begin by recalling the information-theoretic lower bound for this problem.
Consider the task of identifying the true anomalous set $S^\star \subset [K]$, $|S^\star| = n$, under a fixed confidence level $\delta$.

Classical results in sequential hypothesis testing show that any $\delta$-correct algorithm must accumulate enough evidence to separate the true hypothesis from its most confusing alternative~\citep{chernoff1959sequential,garivier2016optimal}.
In our setting, this difficulty is captured by the optimal information rate
\begin{equation}
\label{eq:gamma_star}
\Gamma^\star
=
\max_{\mathbf{c} \in \mathcal{C}}
\;
\min_{S' \neq S^\star}
D(H_{S^\star} \Vert H_{S'} \mid \mathbf{c}),
\end{equation}
where $\mathbf{c}$ denotes a measurement action and $D(\cdot\|\cdot)$ is the Kullback--Leibler divergence between the induced observations.

For Gaussian observations with correlated noise, this divergence admits the closed form
\begin{equation}
\label{eq:kl_form}
D(H_{S^\star} \Vert H_{S'} \mid \mathbf{c})
=
\frac{
\bigl(\mathbf{c}^\top(\bm{\mu}_{S^\star}-\bm{\mu}_{S'})\bigr)^2
}{
2\,\mathbf{c}^\top \bm{\Sigma} \mathbf{c}
}.
\end{equation}

This expression makes the role of correlation explicit.
The numerator measures how well the action $\mathbf{c}$ separates two competing hypotheses.
The denominator measures how much noise remains after aggregation.
When noise is correlated, carefully chosen linear combinations can suppress shared fluctuations and increase this ratio.

Therefore, any $\delta$-correct algorithm must satisfy
\begin{equation}
\label{eq:lower_bound}
\mathbb{E}[\tau]
\;\ge\;
\frac{\log(1/\delta)}{\Gamma^\star}
\,(1-o(1)),
\qquad
\text{as } \delta \to 0.
\end{equation}
This bound represents a fundamental speed limit.
No algorithm can identify anomalies faster than this rate, regardless of its internal design.

\subsection{What ECC-AHT Guarantees in Finite Samples}

We now explain what ECC-AHT achieves relative to this limit.
The algorithm follows a Champion--Challenger principle.
At each round, it identifies the two hypotheses that are currently hardest to distinguish and selects a measurement that maximizes their separation.
Importantly, this selection uses the full covariance matrix $\bm{\Sigma}$.

\begin{theorem}[Order-Optimal Non-Asymptotic Sample Complexity]
\label{thm:nonasymptotic}
Assume that $\bm{\Sigma}$ is positive definite and the action set $\mathcal{C}$ is compact.
For any $\delta \in (0,1)$, ECC-AHT is $\delta$-correct, and its expected stopping time satisfies
\begin{equation}
\label{eq:nonasymp_bound}
\mathbb{E}[\tau]
\;\le\;
\frac{\log(1/\delta)}{\Gamma^\star}
\;+\;
C_1 \log\log(1/\delta)
\;+\;
C_2,
\end{equation}
where $C_1$ and $C_2$ are constants independent of $\delta$.
\end{theorem}

This result shows that ECC-AHT operates at the optimal scale.
Up to lower-order terms, the algorithm matches the fundamental speed limit dictated by $\Gamma^\star$.
In practical terms, this means that ECC-AHT wastes no samples beyond what is unavoidable.

\subsection{Asymptotic Optimality and the Role of Correlation}

The previous result controls finite-sample behavior.
In the asymptotic regime where $\delta \to 0$, a sharper statement is possible.

\begin{theorem}[Exact Asymptotic Optimality]
\label{thm:asymptotic}
Under the same assumptions,
\begin{equation}
\label{eq:asymptotic}
\limsup_{\delta \to 0}
\frac{\mathbb{E}[\tau]}{\log(1/\delta)}
=
\frac{1}{\Gamma^\star}.
\end{equation}
\end{theorem}

This theorem states that ECC-AHT exactly attains the optimal information rate.
As the algorithm runs, its measurement actions converge to an optimal design that solves \eqref{eq:gamma_star}.
As a result, the evidence against the nearest competing hypothesis grows at the fastest possible rate.

This perspective clarifies why correlation matters.
Algorithms that ignore off-diagonal structure in $\bm{\Sigma}$ effectively optimize a smaller rate $\Gamma_{\mathrm{diag}} < \Gamma^\star$.
Their evidence accumulates more slowly, which leads to strictly larger sample complexity.
The ratio $\Gamma^\star / \Gamma_{\mathrm{diag}}$ therefore quantifies the intrinsic speedup enabled by correlation-aware measurement design.

The theory shows that ECC-AHT does not merely adapt faster.
It adapts at the optimal rate permitted by the problem.
Correlation improves performance precisely because it enlarges the achievable information rate, and ECC-AHT is designed to track this rate online.

\section{Experiments}
\label{sec:experiments}

We evaluate ECC-AHT through a sequence of experiments designed to answer one central question:
\emph{when observations are correlated, can an algorithm actively exploit this structure to identify anomalies faster and more reliably?}
To answer this question, we proceed from controlled synthetic settings to a large-scale real-world system.
Along the way, we isolate the role of correlation, experimental design, and algorithmic components.

\begin{remark}
    Due to space limitations, additional experimental results and figures are provided in Appendix~\ref{app:supexp}
\end{remark}

\paragraph{Experimental protocol.}
Across all experiments, we measure the number of observations required to reach an F1-score of $0.95$ or the F1-score that changes as the number of observations increases.
Unless otherwise stated, we average results over 20 independent runs.
We report $95\%$ confidence intervals computed using the BCa bootstrap with 10{,}000 resamples.

\subsection{Does Correlation Help? Scalability and Correlation Exploitation}
\label{sec:5.1}
We begin with synthetic experiments that isolate the effect of correlation under increasing problem dimension. The research results can be found in Appendix~\ref{app:corr}.

We compare ECC-AHT against two baselines.
\textbf{Round Robin (RR)}~\citep{even2006action} cycles through streams uniformly and ignores both correlation and belief information.
\textbf{Random Sparse Projection (RSP)} is a stochastic baseline introduced in this work.
It samples sparse linear combinations at random without exploiting correlation or uncertainty.
Details are provided in Appendix~\ref{app:rsp}.

\textbf{Independent noise as a baseline.}
We first consider a setting with independent noise, $K=100$ and $\bm{\Sigma}=\mathbf{I}$.
Figure~\ref{fig:exp1_1} shows that ECC-AHT already outperforms both baselines.
Even without correlation, adaptive belief tracking allows the algorithm to focus its budget on informative streams.
In contrast, RR scales linearly with $K$ and requires several hundred samples.

\textbf{Moderate and strong correlation.}
The performance gap widens sharply once correlation is introduced.
Figures~\ref{fig:exp1_2} and~\ref{fig:exp1_3} report results under moderate ($\rho=0.5$) and strong ($\rho=0.8$) correlation.
In the high-dimensional setting with $K=1000$, correlation-agnostic methods fail to converge within a reasonable budget.
ECC-AHT, however, leverages correlation to suppress shared noise and isolate true signals.
It identifies anomalies up to an order of magnitude faster.
These results provide the first empirical confirmation that correlation, when used correctly, acts as a resource rather than a liability.

\subsection{Where Does the Gain Come From? Ablation Analysis}
\label{sec:5.2}
We next examine which components of ECC-AHT are essential for its performance.
Figure~\ref{fig:exp2} compares the full algorithm against several ablated variants.
The \textit{No-QP} variant selects active streams but uses a simple difference vector instead of solving the design problem. Details are provided in Appendix~\ref{app:simplediff}.
The \textit{No-Active} variant selects actions at random (This is actually the RSP we mentioned in Section~\ref{sec:5.1}).
The \textit{No-Correlation} variant ignores correlation to verify whether the advantages of ECC-AHT truly stem from its correlation-aware design. Details are provided in Appendix~\ref{app:diagonal}.
We also include a \textit{Cost-Free} benchmark that removes budget constraints. Details are provided in Appendix~\ref{app:costfree}.

The results reveal a clear hierarchy.
The Cost-Free benchmark performs best, followed closely by ECC-AHT.
In contrast, the No-QP and random variants lag far behind.
This gap highlights a key point.
Identifying which streams are uncertain is not enough.
The algorithm must also choose \emph{how} to measure them in order to cancel correlated noise effectively.

\begin{figure}[!ht]
    \centering
    \includegraphics[width=0.48\textwidth]{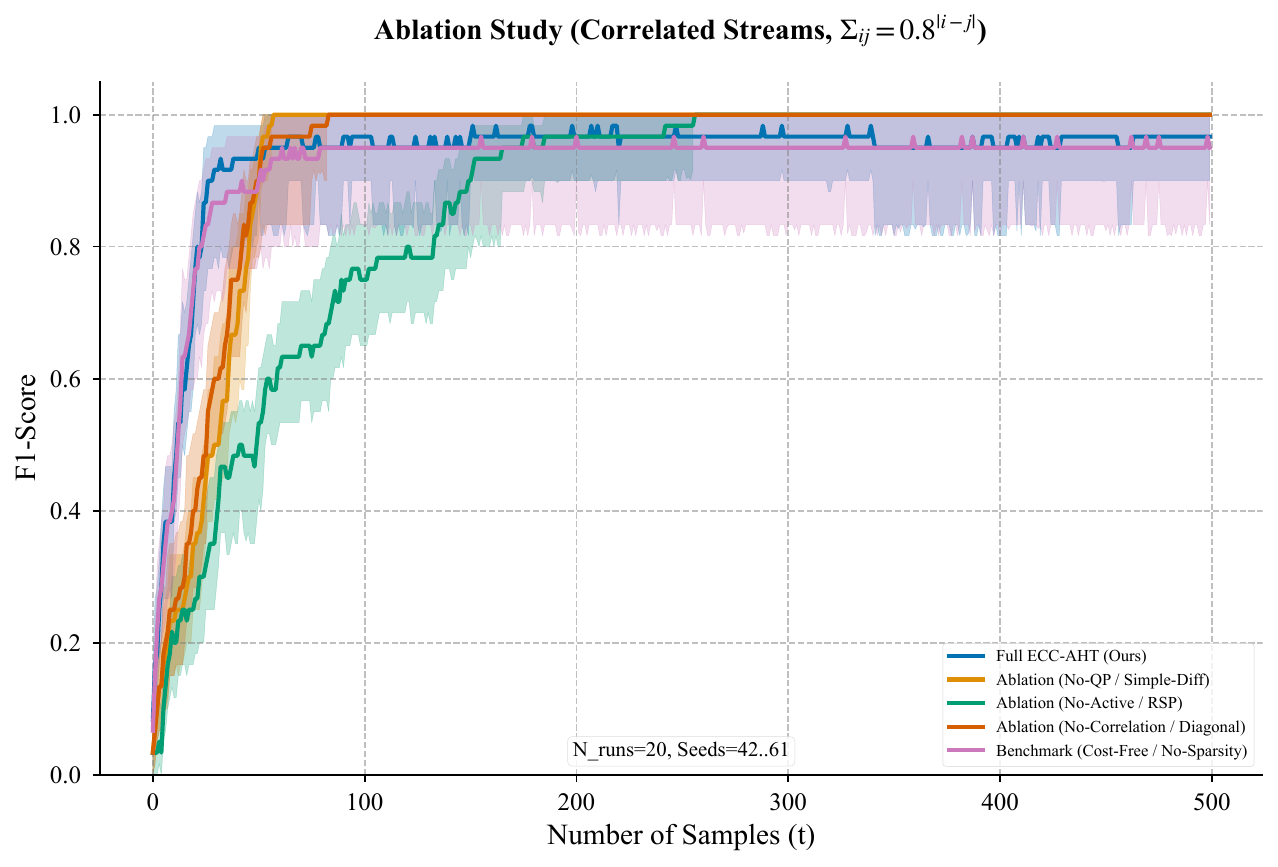}
    \caption{\textbf{Ablation analysis.}
    $K=100, n=3, \rho=0.8, B=4.0$.
    The performance gap between ECC-AHT and No-QP confirms the importance of optimization-based experimental design.}
    \label{fig:exp2}
\end{figure}

\subsection{Robustness Across Problem Parameters}
\label{sec:5.3}
We further test whether the observed gains depend on fine-tuned parameters.
We vary signal strength, correlation level, anomaly count, and budget.
As reported in Appendix~\ref{app:robustness}, ECC-AHT maintains a consistent advantage across all settings.
Performance improves smoothly with increasing budget, which confirms that flexible measurement design translates directly into practical gains rather than brittle behavior.

\subsection{Comparison with State-of-the-Art Methods}

We now compare ECC-AHT with strong existing methods, including Hierarchical Dynamic Search (HDS)~\citep{gafni2023anomaly}, CombGapE, and TTTS.

In addition to standard state-of-the-art baselines, 
we consider a set of \emph{hybrid baselines} that combine individual components of ECC-AHT 
with established algorithms such as CombGapE, TTTS, and HDS.
These hybrids are designed as controlled comparisons to isolate the contribution 
of specific design choices, including hypothesis tracking, randomized selection, 
and expressive experimental design.

\textbf{ECC-AHT-Restricted} uses a binary tree structure based on the HDS algorithm to restrict the action expressiveness.
Details are provided in Appendix~\ref{app:restricted}

\textbf{BaseArm-CombGapE} is a hybrid baseline that combines ECC-AHT with the base-arm action restriction
inherited from CombGapE~\citep{nakamura2023combinatorial}. 
Details are provided in Appendix~\ref{app:basearm-combgape}.

\textbf{TTTS-Challenger} is a controlled hybrid inspired by Top-Two Thompson Sampling~\citep{russo2016simple}.
Details are provided in Appendix~\ref{app:ttts-challenger}.

\begin{remark}
    The experiment was divided into two parts because the HDS algorithm requires that $\Sigma=I$.
\end{remark}

Figure~\ref{fig:phase2} summarizes the results.
ECC-AHT consistently outperforms HDS, whose fixed aggregation structure prevents differential measurements.
At the same time, ECC-AHT matches the performance of randomized TTTS while remaining fully deterministic.
These results show that continuous, correlation-aware experimental design can be as effective as sophisticated randomization strategies.

\begin{figure}[!ht]
    \centering
    \begin{subfigure}[b]{0.48\textwidth}
        \centering
        \includegraphics[width=\textwidth]{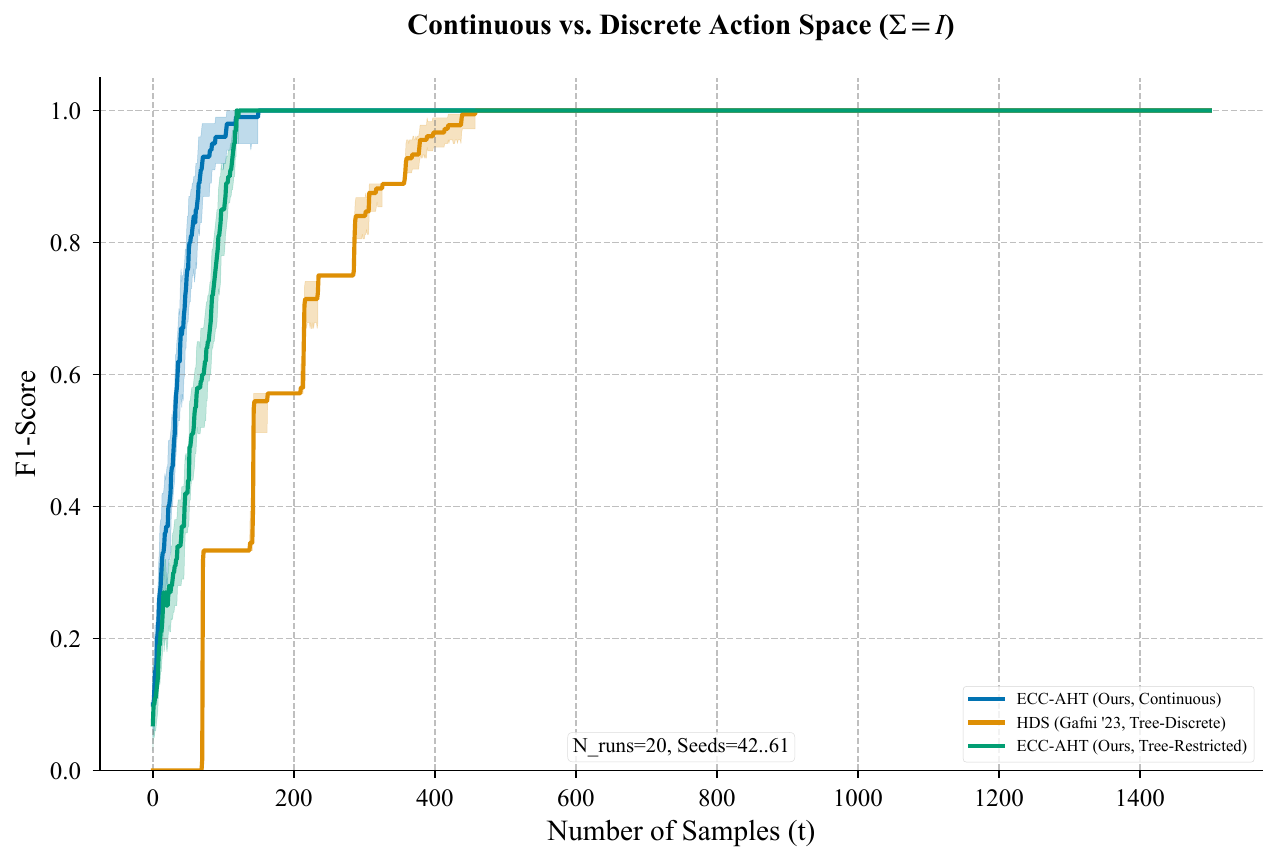}
        \caption{vs. HDS ($K=128$)}
        \label{fig:exp4}
    \end{subfigure}
    \hfill
    \begin{subfigure}[b]{0.48\textwidth}
        \centering
        \includegraphics[width=\textwidth]{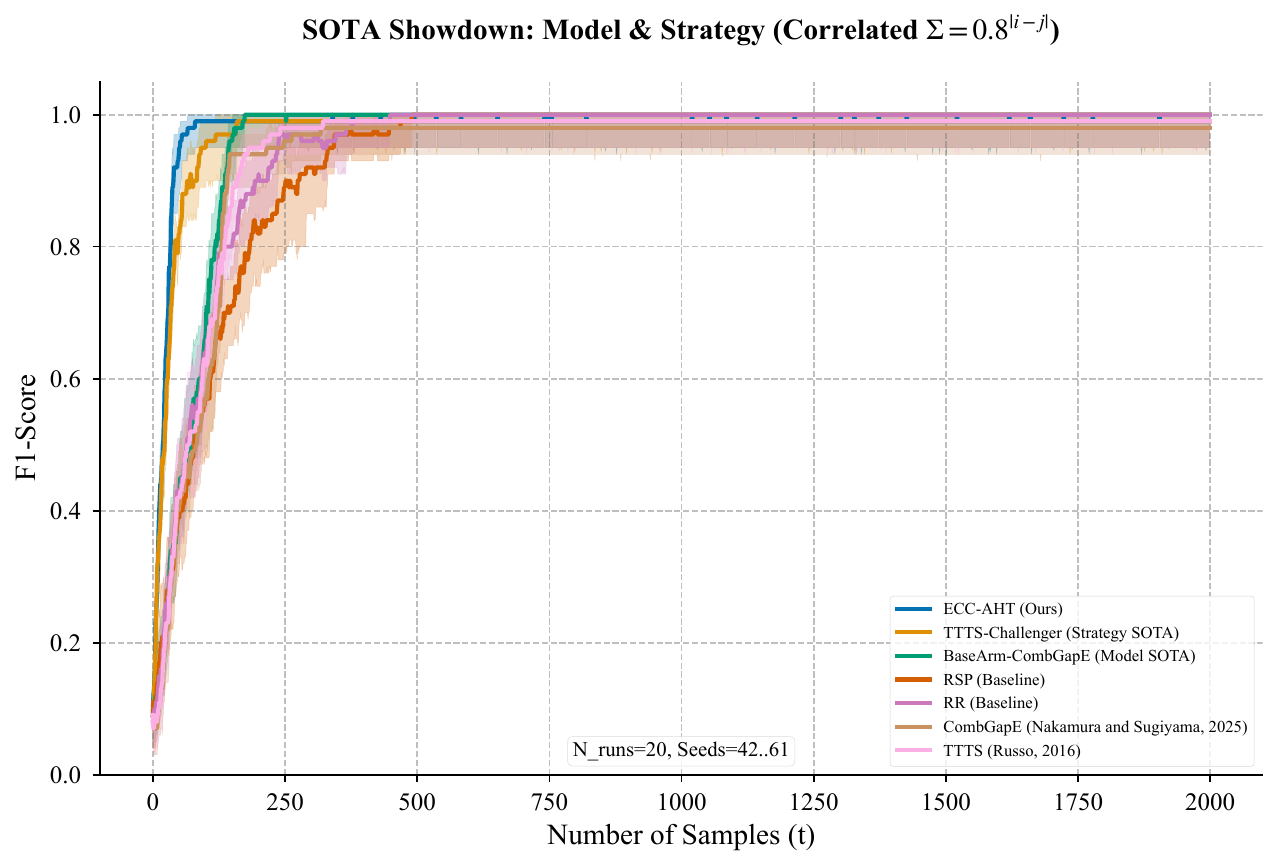}
        \caption{vs. SOTA bandits ($K=128$)}
        \label{fig:exp5}
    \end{subfigure}
    \caption{\textbf{Comparison with state-of-the-art.}
    ECC-AHT outperforms HDS, CombGapE and matches randomized TTTS while remaining deterministic.}
    \label{fig:phase2}
\end{figure}

\subsection{Real-World Evaluation: The WaDi Dataset}

Finally, we evaluate ECC-AHT on the WaDi industrial control dataset~\citep{ahmed2017wadi}, which contains 123 sensors with complex physical correlations.
After preprocessing, the feature set contains 66 streams.
Details are provided in Appendix~\ref{app:reproducibility}.

\textbf{The value of modeling correlation.}
Figure~\ref{fig:wadi_a} compares ECC-AHT with a variant that ignores correlation by using a diagonal covariance.
The full model achieves substantially lower detection delay.
This result confirms that real-world sensor networks contain exploitable correlation structure.

\begin{remark}
    Since \citet{gafni2023anomaly} emphasized that their HDS algorithm is only applicable when $K$ is a power of 2, this algorithm cannot be used in this experiment.
\end{remark}

\textbf{A pipeline-level comparison.}
We also compare two complete pipelines.
The first applies ECC-AHT to 1-minute windowed data.
The second applies the original CombGapE~\citep{nakamura2023combinatorial} to raw 1-second data.
Despite operating on coarser temporal resolution, ECC-AHT detects anomalies faster in wall-clock time.
Figure~\ref{fig:wadi_b} shows that structured noise cancellation outweighs the loss of temporal resolution.
This result highlights the practical importance of correlation-aware modeling in noisy industrial systems.

\begin{figure}[!ht]
    \centering
    \begin{subfigure}[b]{0.48\textwidth}
        \centering
        \includegraphics[width=\textwidth]{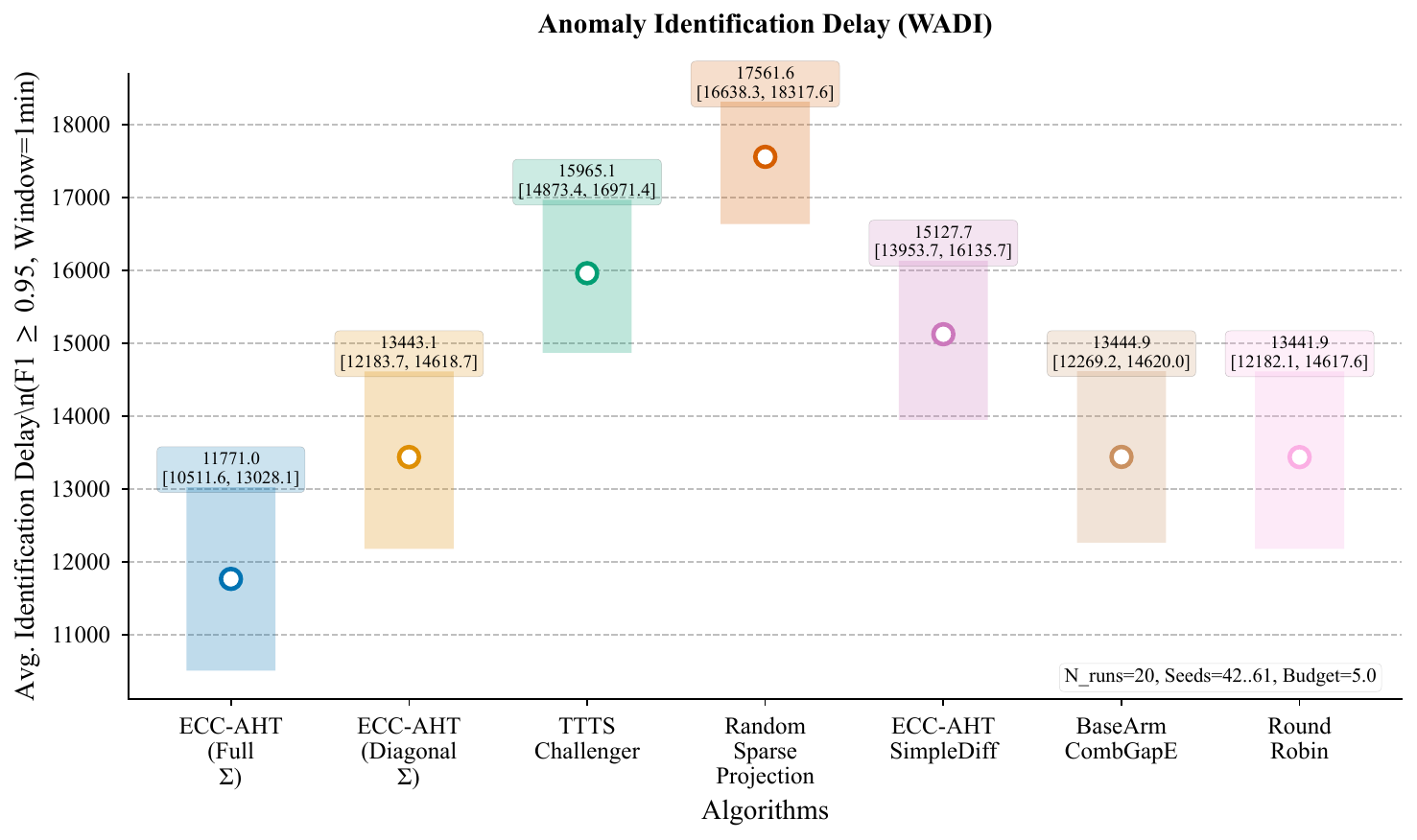}
        \caption{Effect of correlation modeling}
        \label{fig:wadi_a}
    \end{subfigure}
    \begin{subfigure}[b]{0.48\textwidth}
        \centering
        \includegraphics[width=\textwidth]{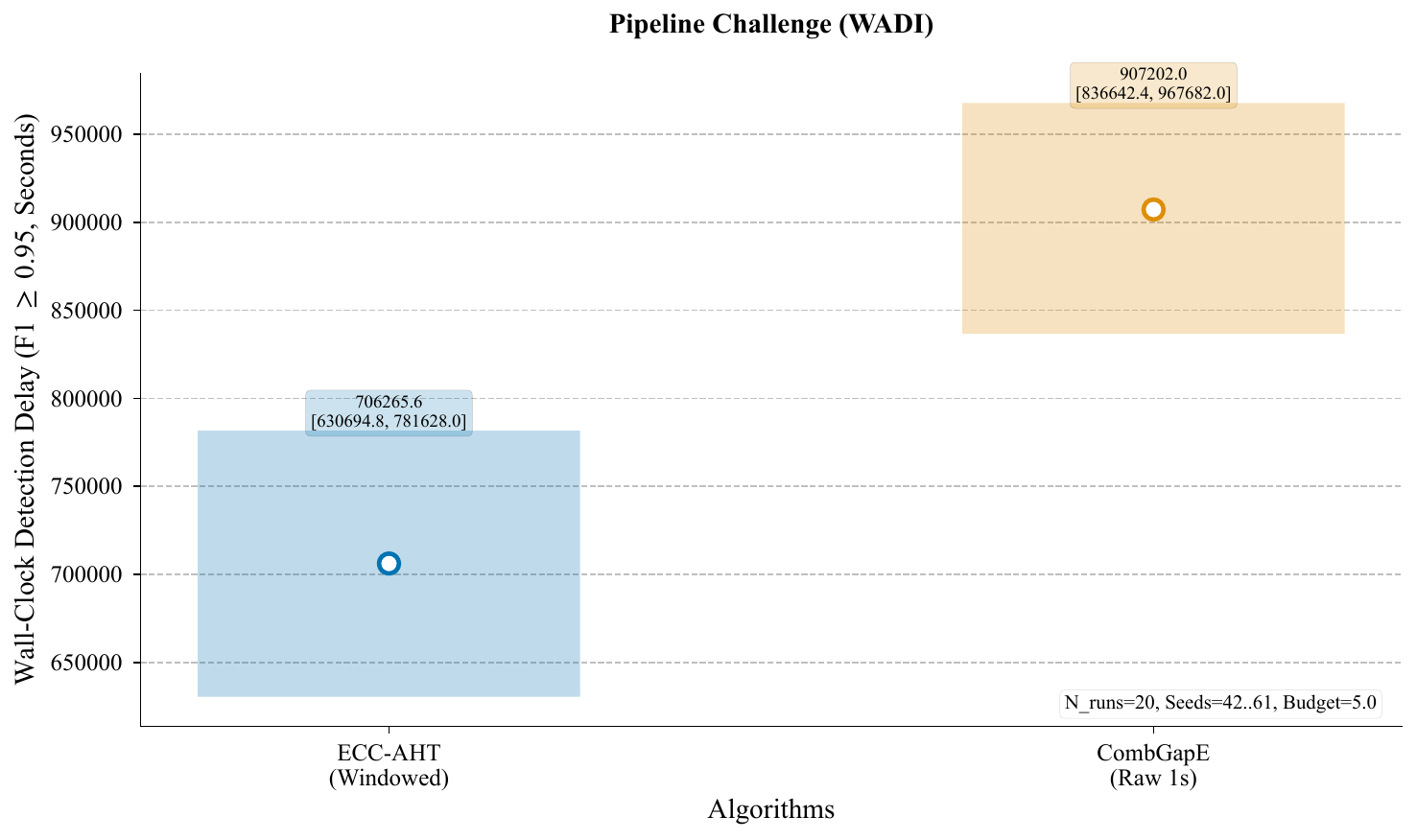}
        \caption{Pipeline comparison}
        \label{fig:wadi_b}
    \end{subfigure}
    \caption{\textbf{Real-world results on WaDi.}
    Exploiting correlation reduces detection delay and improves end-to-end performance.}
    \label{fig:wadi}
\end{figure}

\section{Conclusion}
\label{sec:conclusion}

In this paper, we revisited the problem of combinatorial anomaly detection from an information-driven perspective. Rather than treating sensor correlation as a nuisance, we showed that it can be actively exploited as a source of signal. Based on this view, we proposed \textbf{ECC-AHT}, an algorithm that designs noise-canceling experiments in a continuous action space. At each step, the algorithm solves a quadratic program to select projections that best separate the most ambiguous hypotheses.

From a theoretical standpoint, we provided guarantees showing that ECC-AHT achieves asymptotically optimal sample complexity under standard modeling assumptions. The resulting rates match the information-theoretic limits characterized by Chernoff information. These results connect active anomaly detection with ideas from both multi-armed bandits and optimal experimental design. From a methodological perspective, our approach demonstrates that relaxing the combinatorial search space to a continuous domain leads to tractable algorithms that can be implemented using efficient convex optimization tools.

We further evaluated ECC-AHT on the real-world WaDi dataset using a controlled comparison between a model-based pipeline and a model-free baseline. Despite operating on windowed data and therefore lower temporal resolution, the model-based approach consistently identified anomalies earlier. This empirical result highlights the practical value of explicitly modeling and exploiting correlation structure in large-scale cyber-physical systems.

Several directions remain open. An important next step is to extend the framework beyond Gaussian assumptions and linear correlations. Possible avenues include kernelized models or deep neural networks (DNNs) that preserve the active design principle. Overall, ECC-AHT provides a principled and scalable approach to high-dimensional anomaly detection, with both theoretical support and empirical validation.

\section*{Impact Statement}

This work studies adaptive hypothesis testing and measurement design under correlated observations, with a primary focus on anomaly detection in large-scale monitoring systems.
Potential positive impacts include improved reliability, safety, and efficiency in cyber-physical systems, such as industrial control, infrastructure monitoring, and fault diagnosis, where early and accurate detection of anomalies is critical.

As with many anomaly detection and monitoring techniques, the proposed methods could also be applied in broader sensing or surveillance settings.
However, this paper focuses on algorithmic and theoretical aspects of sequential decision making and does not introduce new capabilities for data collection or individual-level inference beyond standard sensing models.
Any deployment of such methods should follow existing legal and ethical guidelines governing data use and system monitoring.

Overall, we believe this work advances the foundations of machine learning for structured sequential inference, with societal implications that are consistent with those of prior research in this area.


\bibliography{example_paper}

@article{chen2014combinatorial,
  title={Combinatorial pure exploration of multi-armed bandits},
  author={Chen, Shouyuan and Lin, Tian and King, Irwin and Lyu, Michael R and Chen, Wei},
  journal={Advances in neural information processing systems},
  volume={27},
  year={2014}
}

@article{nakamura2023combinatorial,
    author = {Nakamura, Shintaro and Sugiyama, Masashi},
    title = {A Fast Algorithm for the Real-Valued Combinatorial Pure Exploration of the Multi-Armed Bandit},
    journal = {Neural Computation},
    volume = {37},
    number = {2},
    pages = {294-310},
    year = {2025},
    month = {01},
    issn = {0899-7667},
    doi = {10.1162/neco_a_01728},
    url = {https://doi.org/10.1162/neco_a_01728},
    eprint = {https://direct.mit.edu/neco/article-pdf/37/2/294/2482166/neco_a_01728.pdf},
}

@inproceedings{russo2016simple,
  title={Simple bayesian algorithms for best arm identification},
  author={Russo, Daniel},
  booktitle={Conference on learning theory},
  pages={1417--1418},
  year={2016},
  organization={PMLR}
}

@article{gafni2023anomaly,
  title={Anomaly search over discrete composite hypotheses in hierarchical statistical models},
  author={Gafni, Tomer and Wolff, Benjamin and Revach, Guy and Shlezinger, Nir and Cohen, Kobi},
  journal={IEEE Transactions on Signal Processing},
  volume={71},
  pages={202--217},
  year={2023},
  publisher={IEEE}
}

@book{wald1947sequential,
  title={Sequential Analysis},
  author={Wald, Abraham},
  year={1947},
  publisher={John Wiley \& Sons}
}

@article{chen2021understanding,
  title={Understanding bandits with graph feedback},
  author={Chen, Houshuang and Li, Shuai and Zhang, Chihao and others},
  journal={Advances in Neural Information Processing Systems},
  volume={34},
  pages={24659--24669},
  year={2021}
}

@book{dayan2005theoretical,
  title={Theoretical neuroscience: computational and mathematical modeling of neural systems},
  author={Dayan, Peter and Abbott, Laurence F},
  year={2005},
  publisher={MIT press}
}

@article{marr1980theory,
  title={Theory of edge detection},
  author={Marr, David and Hildreth, Ellen},
  journal={Proceedings of the Royal Society of London. Series B. Biological Sciences},
  volume={207},
  number={1167},
  pages={187--217},
  year={1980},
  publisher={The Royal Society London}
}

@article{kohonen1982self,
  title={Self-organized formation of topologically correct feature maps},
  author={Kohonen, Teuvo},
  journal={Biological cybernetics},
  volume={43},
  number={1},
  pages={59--69},
  year={1982},
  publisher={Springer}
}

@article{hartline1956inhibition,
  title={Inhibition in the eye of Limulus},
  author={Hartline, H K and Wagner, Henry G and Ratliff, Floyd},
  journal={The Journal of general physiology},
  volume={39},
  number={5},
  pages={651--673},
  year={1956},
  publisher={Rockefeller University Press}
}

@article{chernoff1959sequential,
  title={Sequential design of experiments},
  author={Chernoff, Herman},
  journal={Annals of Mathematical Statistics},
  volume={30},
  number={3},
  pages={755--770},
  year={1959}
}

@inproceedings{garivier2016optimal,
  title={Optimal best arm identification with fixed confidence},
  author={Garivier, Aur{\'e}lien and Kaufmann, Emilie},
  booktitle={Conference on Learning Theory},
  pages={998--1027},
  year={2016},
  organization={PMLR}
}

@inproceedings{ahmed2017wadi,
  title={WADI: a water distribution testbed for research in the design of secure cyber physical systems},
  author={Ahmed, Chuadhry Mujeeb and Palleti, Venkata Reddy and Mathur, Aditya P},
  booktitle={Proceedings of the 3rd international workshop on cyber-physical systems for smart water networks},
  pages={25--28},
  year={2017}
}

@inproceedings{shang2020fixed,
  title={Fixed-confidence guarantees for bayesian best-arm identification},
  author={Shang, Xuedong and Heide, Rianne and Menard, Pierre and Kaufmann, Emilie and Valko, Michal},
  booktitle={International Conference on Artificial Intelligence and Statistics},
  pages={1823--1832},
  year={2020},
  organization={PMLR}
}

@inproceedings{fiez2019sequential,
  title={Sequential Experimental Design for Transductive Linear Bandits},
  author={Fiez, Tanner and Jain, Lalit and Jamieson, Kevin and Ratliff, Lillian},
  booktitle={Advances in Neural Information Processing Systems},
  volume={32},
  year={2019}
}

@inproceedings{huang2018combinatorial,
  title={Combinatorial pure exploration with continuous and separable reward functions and its applications.},
  author={Huang, Weiran and Ok, Jungseul and Li, Liang and Chen, Wei},
  booktitle={International Joint Conference on Artificial Intelligence},
  pages={2291--2297},
  year={2018}
}

@inproceedings{jourdan2021efficient,
  title={Efficient pure exploration for combinatorial bandits with semi-bandit feedback},
  author={Jourdan, Marc and Mutn{\`y}, Mojm{\'\i}r and Kirschner, Johannes and Krause, Andreas},
  booktitle={Algorithmic Learning Theory},
  pages={805--849},
  year={2021},
  organization={PMLR}
}

@article{pallakonda2025ai,
  title={AI-Driven Attack Detection and Cryptographic Privacy Protection for Cyber-Resilient Industrial Control Systems},
  author={Pallakonda, Archana and Kaliyannan, Kabilan and Sumathi, Rahul Loganathan and Raj, Rayappa David Amar and Yanamala, Rama Muni Reddy and Napoli, Christian and Randieri, Cristian},
  journal={IoT},
  volume={6},
  number={3},
  pages={56},
  year={2025},
  publisher={MDPI}
}

@article{odeyomi2025intrusion,
  title={Intrusion Detection and Resiliency in Cyber-Physical Systems and Networks},
  author={Odeyomi, Olusola T and Olowu, Temitayo O},
  journal={Future Internet},
  volume={17},
  number={9},
  pages={424},
  year={2025},
  publisher={MDPI}
}

@article{shanthini2025graph,
author = {A, Shanthini and S, John and Amagoth, Balaram and Talakoti, Mamatha and Sarkar, Debdatta and V, Arun},
year = {2025},
month = {10},
pages = {2591-2602},
title = {Graph Neural Networks for Modelling Structural and Functional Dependencies in Smart Cyber Physical Systems},
journal = {Journal of Machine and Computing},
doi = {10.53759/7669/jmc202505199}
}

@article{han2025timeseries,
  title={Time-Series-Based Anomaly Detection in Industrial Control Systems Using Generative Adversarial Networks},
  author={Han, Chungku and Gim, Gwangyong},
  journal={Processes},
  volume={13},
  number={9},
  pages={2885},
  year={2025},
  publisher={MDPI}
}

@article{zhang2025high,
  title={Towards High-Resolution Industrial Image Anomaly Detection},
  author={Zhang, Ximiao and Xu, Min and Zhou, Xiuzhuang},
  journal={arXiv preprint arXiv:2508.12931},
  year={2025}
}

@article{feng2025false,
  title={False Data-Injection Attack Detection in Cyber-Physical Systems: A Wasserstein Distributionally Robust Reachability Optimization Approach},
  author={Feng, Yulin and Lan, Dapeng and Shang, Chao},
  journal={arXiv preprint arXiv:2508.12402},
  year={2025}
}

@book{lee2016introduction,
  title={Introduction to embedded systems: A cyber-physical systems approach},
  author={Lee, Edward Ashford and Seshia, Sanjit Arunkumar},
  year={2016},
  publisher={MIT press}
}

@inproceedings{srinivas2010gaussian,
author = {Srinivas, Niranjan and Krause, Andreas and Kakade, Sham and Seeger, Matthias},
title = {Gaussian process optimization in the bandit setting: no regret and experimental design},
year = {2010},
isbn = {9781605589077},
publisher = {Omnipress},
address = {Madison, WI, USA},
booktitle = {Proceedings of the 27th International Conference on International Conference on Machine Learning},
pages = {1015–1022},
numpages = {8},
location = {Haifa, Israel},
series = {ICML'10}
}

@article{frazier2018bayesian,
  title={A tutorial on Bayesian optimization},
  author={Frazier, Peter I},
  journal={arXiv preprint arXiv:1807.02811},
  year={2018}
}

@article{abbasi2011improved,
  title={Improved algorithms for linear stochastic bandits},
  author={Abbasi-Yadkori, Yasin and P{\'a}l, D{\'a}vid and Szepesv{\'a}ri, Csaba},
  journal={Advances in neural information processing systems},
  volume={24},
  year={2011}
}

@inproceedings{agrawal2013thompson,
  title={Thompson sampling for contextual bandits with linear payoffs},
  author={Agrawal, Shipra and Goyal, Navin},
  booktitle={International conference on machine learning},
  pages={127--135},
  year={2013},
  organization={PMLR}
}

@book{lattimore2020bandit,
  title={Bandit algorithms},
  author={Lattimore, Tor and Szepesv{\'a}ri, Csaba},
  year={2020},
  publisher={Cambridge University Press}
}

@article{akoglu2015graph,
  title={Graph based anomaly detection and description: a survey},
  author={Akoglu, Leman and Tong, Hanghang and Koutra, Danai},
  journal={Data mining and knowledge discovery},
  volume={29},
  number={3},
  pages={626--688},
  year={2015},
  publisher={Springer}
}

@article{pasqualetti2013attack,
  title={Attack detection and identification in cyber-physical systems},
  author={Pasqualetti, Fabio and D{\"o}rfler, Florian and Bullo, Francesco},
  journal={IEEE transactions on automatic control},
  volume={58},
  number={11},
  pages={2715--2729},
  year={2013},
  publisher={IEEE}
}

@inproceedings{sen2017contextual,
  title={Contextual bandits with latent confounders: An nmf approach},
  author={Sen, Rajat and Shanmugam, Karthikeyan and Kocaoglu, Murat and Dimakis, Alex and Shakkottai, Sanjay},
  booktitle={Artificial Intelligence and Statistics},
  pages={518--527},
  year={2017},
  organization={PMLR}
}

@article{gupta2021multi,
  title={Multi-armed bandits with correlated arms},
  author={Gupta, Samarth and Chaudhari, Shreyas and Joshi, Gauri and Ya{\u{g}}an, Osman},
  journal={IEEE Transactions on Information Theory},
  volume={67},
  number={10},
  pages={6711--6732},
  year={2021},
  publisher={IEEE}
}

@article{jedra2020optimal,
  title={Optimal best-arm identification in linear bandits},
  author={Jedra, Yassir and Proutiere, Alexandre},
  journal={Advances in Neural Information Processing Systems},
  volume={33},
  pages={10007--10017},
  year={2020}
}

@inproceedings{honda2010asymptotically,
  title={An Asymptotically Optimal Bandit Algorithm for Bounded Support Models.},
  author={Honda, Junya and Takemura, Akimichi},
  booktitle={Conference on Learning Theory},
  pages={67--79},
  year={2010}
}

@article{even2006action,
  title={Action elimination and stopping conditions for the multi-armed bandit and reinforcement learning problems.},
  author={Even-Dar, Eyal and Mannor, Shie and Mansour, Yishay and Mahadevan, Sridhar},
  journal={Journal of machine learning research},
  volume={7},
  number={6},
  year={2006}
}

@book{van2000asymptotic,
  title={Asymptotic statistics},
  author={Van der Vaart, Aad W},
  volume={3},
  year={2000},
  publisher={Cambridge university press}
}
\bibliographystyle{icml2026}

%%%%%%%%%%%%%%%%%%%%%%%%%%%%%%%%%%%%%%%%%%%%%%%%%%%%%%%%%%%%%%%%%%%%%%%%%%%%%%%
%%%%%%%%%%%%%%%%%%%%%%%%%%%%%%%%%%%%%%%%%%%%%%%%%%%%%%%%%%%%%%%%%%%%%%%%%%%%%%%
% APPENDIX
%%%%%%%%%%%%%%%%%%%%%%%%%%%%%%%%%%%%%%%%%%%%%%%%%%%%%%%%%%%%%%%%%%%%%%%%%%%%%%%
%%%%%%%%%%%%%%%%%%%%%%%%%%%%%%%%%%%%%%%%%%%%%%%%%%%%%%%%%%%%%%%%%%%%%%%%%%%%%%%
\newpage
\appendix
\onecolumn
\section{Inside the Black Box: An Interpretative Analysis}
\label{app:interpretation}

In this section, we look deeper into the internal dynamics of ECC-AHT. We want to show that the algorithm is not just a black box. It acts like an intelligent agent that learns to use the environment to its advantage. Figure~\ref{fig:inside_ecc_aht} tells the story of how the algorithm evolves from simple exploration to complex strategy.

Figure~\ref{fig:inside_ecc_aht:f} is the most important part of this story. It reveals how our method turns correlation from a problem into a resource. Here, we see the algorithm's decision at time $t=5$. The algorithm faces a specific challenge. It needs to distinguish the current ``Champion'' (Stream 11) from the ``Challenger'' (Stream 6). To do this, it assigns a large positive weight to the Champion and a large negative weight to the Challenger. This creates a strong contrast between the two main hypotheses.

However, the real intelligence of the algorithm appears when we look at the local structure. Look at the streams right next to the Champion, specifically Stream 10 and Stream 12. These streams are not the Challenger. They are not anomalies. Yet, the algorithm deliberately gives them negative weights. Why does it do this?

This is a classic example of \emph{noise cancellation}. The yellow background in Figure~\ref{fig:inside_ecc_aht:f} shows that the Champion (Stream 11) shares strong positive correlation with its neighbors (Streams 10 and 12). They fluctuate together because of the shared environmental noise. The algorithm realizes that if it measures the Champion alone, the variance will be high. But if it subtracts the neighbors from the Champion, the shared noise cancels out. The neighbors act as ``reference sensors'' that capture the background noise. The algorithm uses this information to clean up the signal of the Champion.

This creates a beautiful geometric structure known in machine learning and neuroscience as \textbf{``Lateral Inhibition''}~\citep{hartline1956inhibition, dayan2005theoretical} or a \textbf{``Mexican Hat''} shape~\citep{marr1980theory, kohonen1982self}. We see a strong positive peak in the center, surrounded by negative valleys on the sides. The algorithm also does the exact mirror image around the Challenger (Stream 6). The Challenger has a negative weight, so the algorithm gives its neighbors (Streams 5 and 7) positive weights to cancel their noise. Even more excitingly, this phenomenon is not an isolated case at $t=5$. Observing Figure~\ref{fig:inside_ecc_aht:e}, we can see that around almost every dark red ($\mathbf{c_t}$ takes the maximum positive value in this stream, i.e., the ``champion'') or dark blue ($\mathbf{c_t}$ takes the minimum negative value in this stream, i.e., the ``challenger'') area, there are relatively lighter color blocks of the opposite color. This indicates that Mexican Hat shape occurs almost every time.

We did not explicitly program this rule. We only gave the algorithm the correlation matrix and a budget. The algorithm discovered this strategy on its own. It learned that the best way to reduce uncertainty is not just to measure the target, but to measure the context around it. This confirms our key insight: correlation is a powerful resource for active exploration.

\section{Supplementary Experiments}
\label{app:supexp}
To provide a comprehensive view of our experimental analysis, this appendix includes additional results omitted from the main text due to space constraints, as well as exploratory findings from ongoing work.

\subsection{Scalability and Correlation Exploitation Study in Section~\ref{sec:5.1}}
\label{app:corr}
This section displays the experimental results from Section~\ref{sec:5.1}, which were moved to here due to space limitations.

\begin{figure}[!ht]
    \centering
    \begin{subfigure}[b]{0.32\textwidth}
        \centering
        \includegraphics[width=\textwidth]{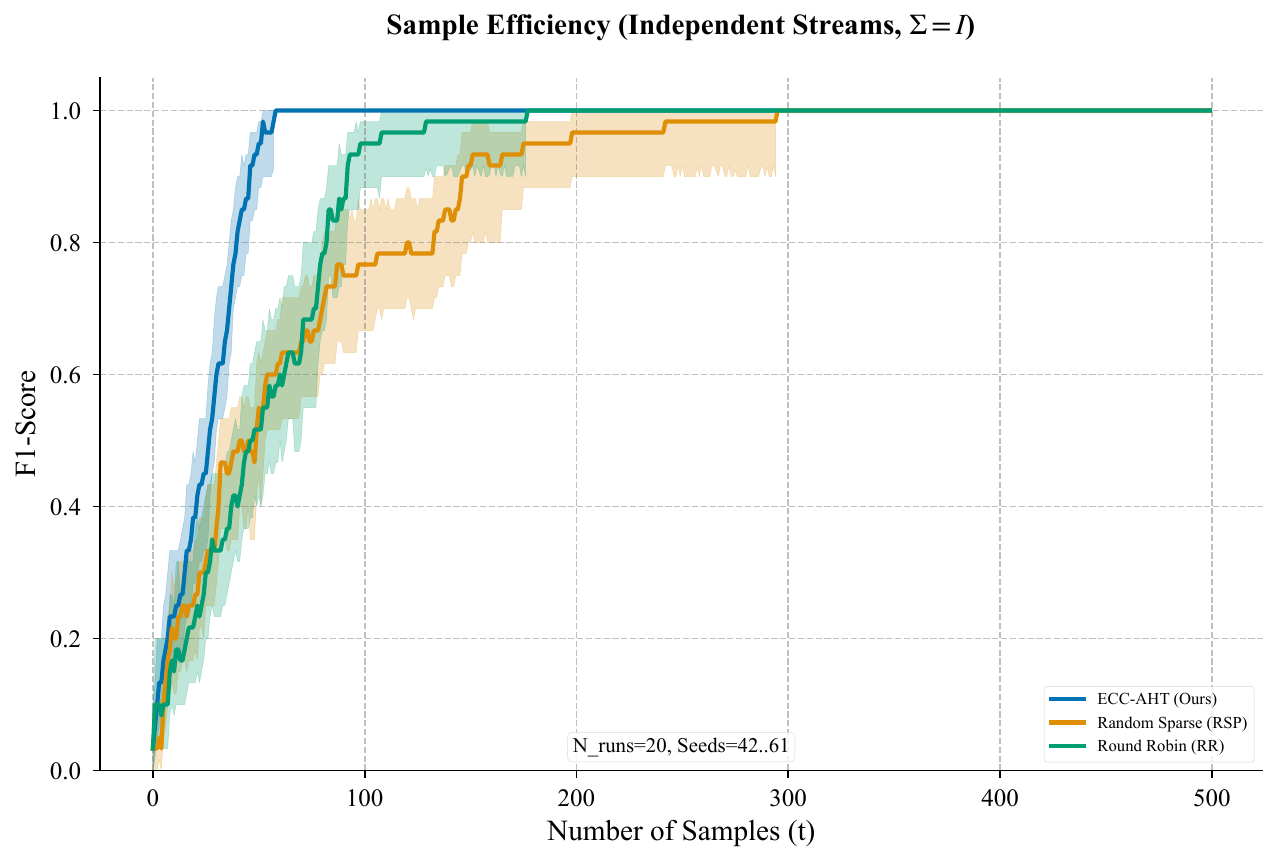}
        \caption{$K=100, \bm{\Sigma}=\mathbf{I}, B=4.0$}
        \label{fig:exp1_1}
    \end{subfigure}
    \hfill
    \begin{subfigure}[b]{0.32\textwidth}
        \centering
        \includegraphics[width=\textwidth]{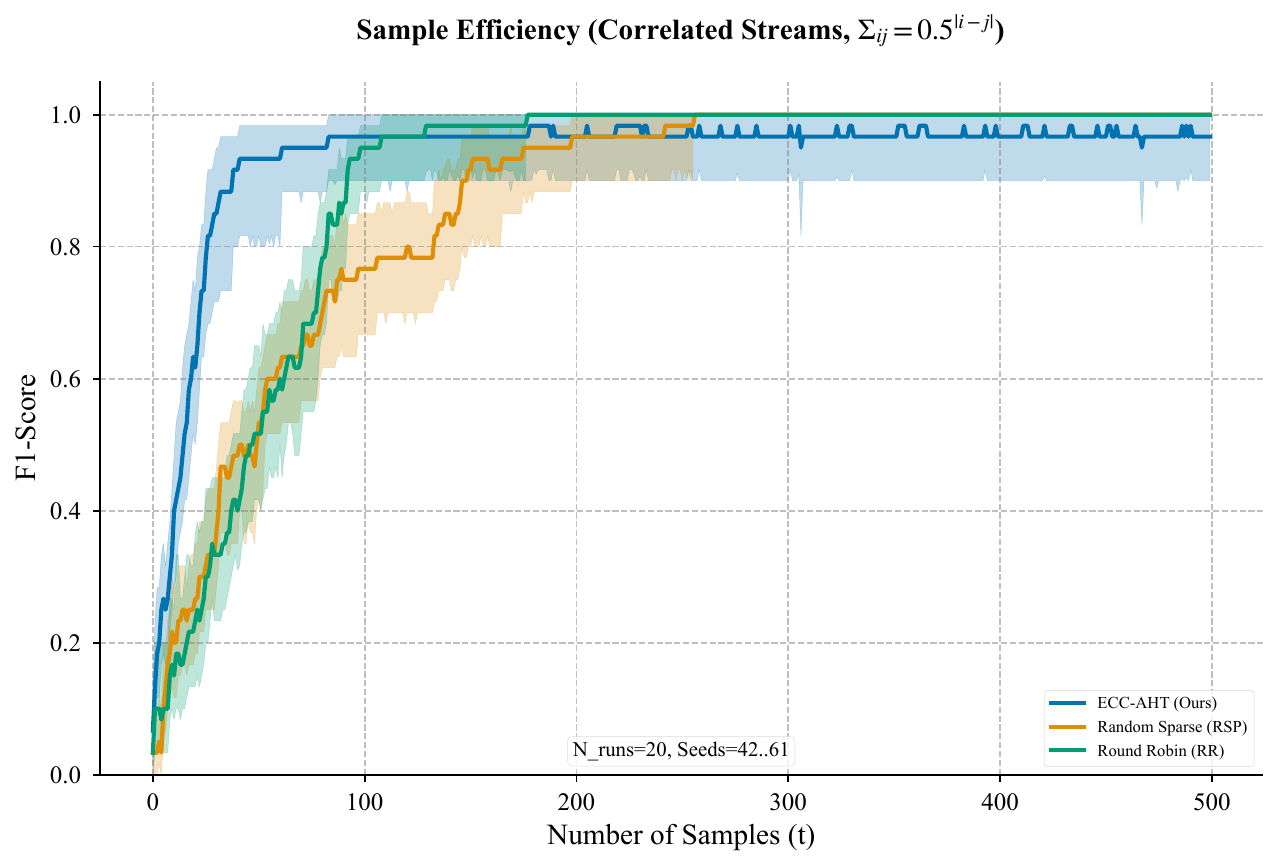}
        \caption{$K=100, \rho=0.5, B=4.0$}
        \label{fig:exp1_2}
    \end{subfigure}
    \hfill
    \begin{subfigure}[b]{0.32\textwidth}
        \centering
        \includegraphics[width=\textwidth]{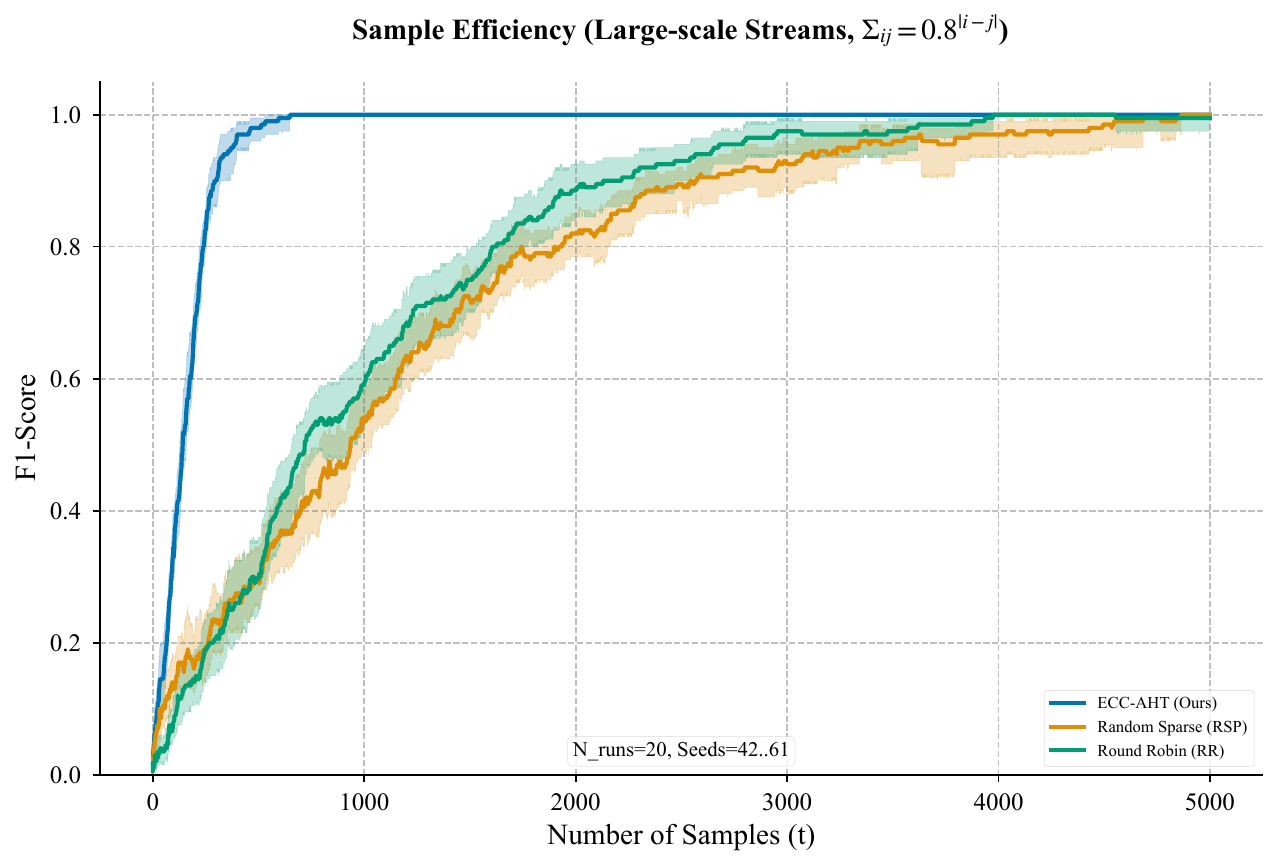}
        \caption{$K=1000, \rho=0.8, B=10.0$}
        \label{fig:exp1_3}
    \end{subfigure}
    \caption{\textbf{Scalability and correlation exploitation.}
    F1-score versus samples.
    ECC-AHT scales to high dimensions and benefits increasingly from stronger correlation.
    Shaded regions show $95\%$ confidence intervals.}
    \label{fig:phase1}
\end{figure}

\subsection{Robustness Analysis in Section~\ref{sec:5.3}}
\label{app:robustness}

In this appendix, we present the detailed results of the sensitivity analyzes discussed in Section~\ref{sec:5.3} of the main text. While the main experiments focused on specific, representative settings, it is crucial to verify that the performance advantage of ECC-AHT is consistent across a broad spectrum of environmental parameters. We systematically vary the signal strength, correlation level, number of anomalies, and the measurement budget to stress-test the algorithm. The baseline parameters are set as $K=100, n=3, \delta=3.0, \rho=0.6, B=5.0$.

\subsubsection{Sensitivity to Signal Strength}
We first investigate how the magnitude of the anomaly signal $\delta$ affects the detection delay. A smaller $\delta$ implies a lower Signal-to-Noise Ratio (SNR), making the problem significantly harder. As illustrated in Figure \ref{fig:robust_mu}, the detection delay for all algorithms naturally decreases as the signal strength increases. However, ECC-AHT consistently outperforms the baselines across the entire range. Notably, in the low-SNR regime (e.g., $\delta < 1.0$), where the random sparse projection (RSP) struggles to distinguish the signal from the noise floor, ECC-AHT remains effective. This resilience is attributed to our covariance-aware optimization, which effectively lowers the noise floor, thereby maintaining a viable effective SNR even when the raw signal is weak.

\begin{figure}[h]
    \centering
    \includegraphics[width=0.48\textwidth]{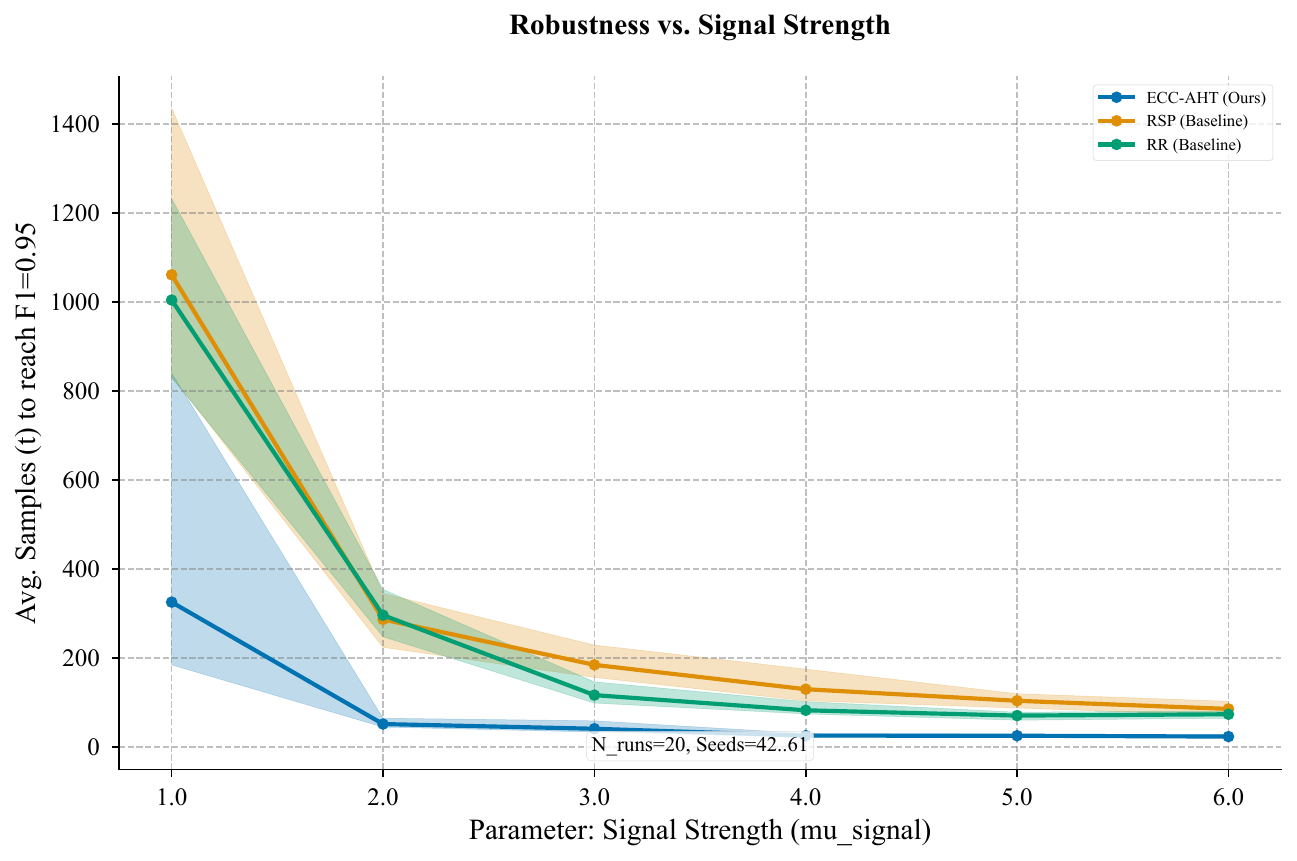}
    \caption{\textbf{Impact of Signal Strength ($\delta$).} ECC-AHT maintains superior performance even in low-SNR regimes, proving that active noise cancellation is most critical when signals are weak.}
    \label{fig:robust_mu}
\end{figure}

\subsubsection{Sensitivity to Correlation Structure}
A central claim of our paper is that ECC-AHT treats correlation as a resource. To verify this, we varied the correlation coefficient $\rho$ of the Toeplitz covariance matrix from 0 (independent) to 0.9 (highly correlated). The results in Figure \ref{fig:robust_rho} are striking. While the performance of standard baselines like RSP typically degrades or remains stagnant as correlation increases, ECC-AHT exhibits the opposite trend: its detection delay \textit{decreases} as $\rho$ increases. This confirms that our QP solver successfully leverages stronger correlations to construct more effective difference measurements, turning what is usually a source of interference into a mechanism for acceleration.

\begin{figure}[h]
    \centering
    \includegraphics[width=0.48\textwidth]{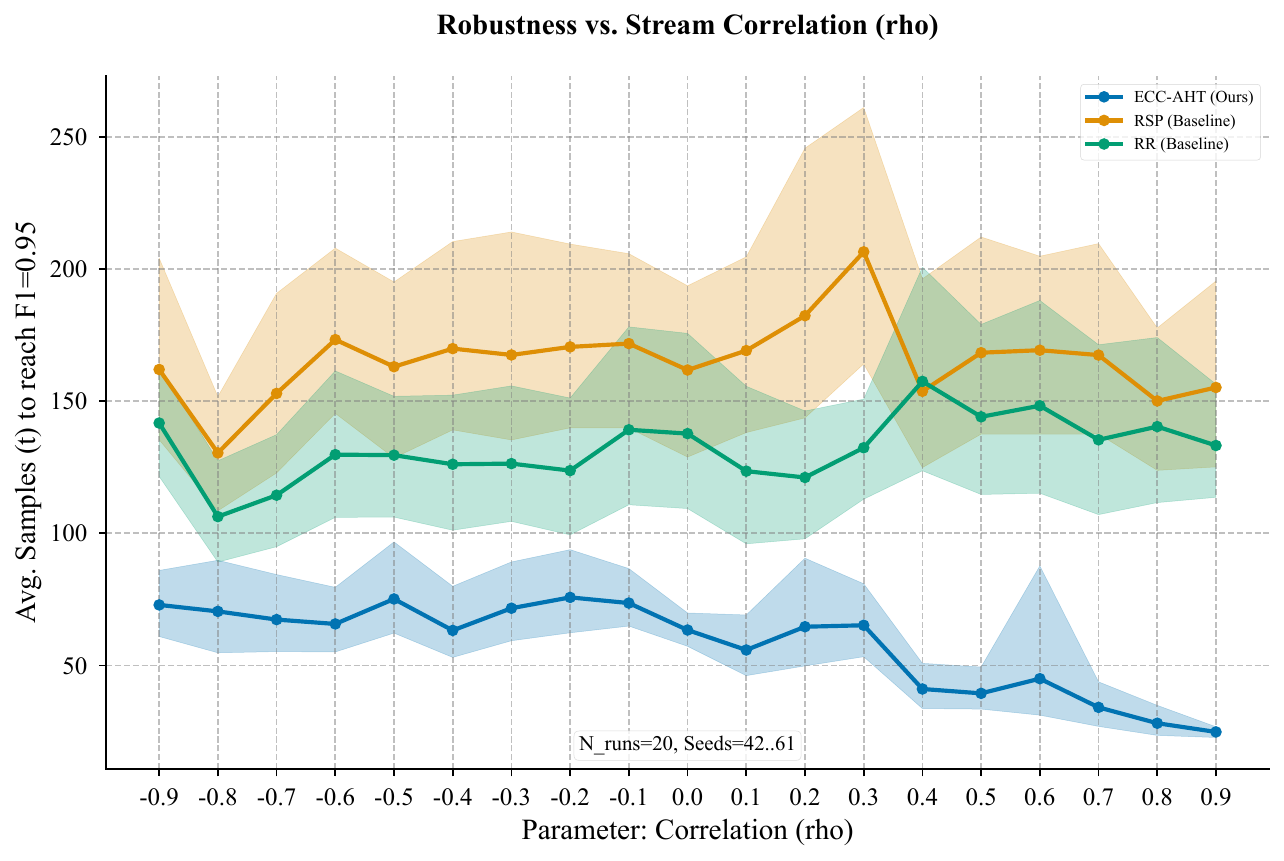}
    \caption{\textbf{Impact of Correlation ($\rho$).} Unlike baselines, ECC-AHT improves as correlation increases, empirically validating the "correlation as a resource" hypothesis.}
    \label{fig:robust_rho}
\end{figure}

\subsubsection{Sensitivity to the Number of Anomalies}
We also examined the scalability of our approach with respect to the number of anomalies $n$ (the sparsity level). Increasing $n$ expands the combinatorial search space and introduces potential masking effects between multiple anomalies. Figure \ref{fig:robust_n} shows the identification delay as a function of $n$. While the delay increases for all methods, ECC-AHT scales much more favorably than the baselines. The random strategies often fail to disentangle complex multi-anomaly scenarios efficiently, whereas the active hypothesis testing framework of ECC-AHT systematically reduces uncertainty, effectively handling the increased combinatorial complexity.

\begin{figure}[h]
    \centering
    \includegraphics[width=0.48\textwidth]{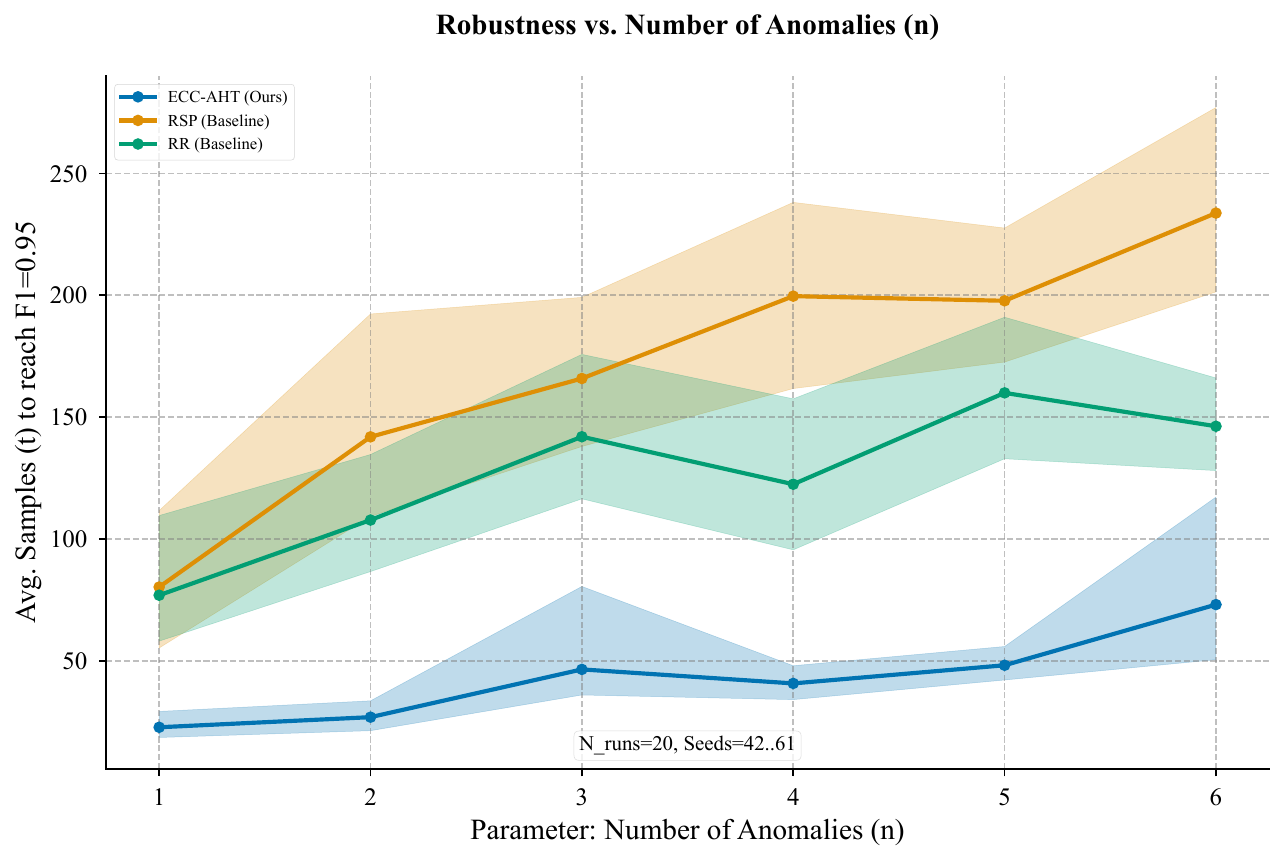}
    \caption{\textbf{Impact of Anomaly Count ($n$).} ECC-AHT exhibits better scalability with respect to the number of concurrent anomalies compared to random baselines.}
    \label{fig:robust_n}
\end{figure}

\subsubsection{Sensitivity to Measurement Budget}
Finally, we analyzed the effect of the $L_1$ measurement budget $B$. The parameter $B$ controls the "energy" or "density" allowed in a single measurement projection. A larger $B$ allows the algorithm to construct measurement vectors with higher signal-to-noise ratios. As shown in Figure \ref{fig:robust_B}, the performance of ECC-AHT improves monotonically with increasing $B$. This validates our continuous relaxation strategy. By allowing the algorithm to optimize the weights of the measurement vector within a larger feasible set, we can extract more information per step. In contrast, methods that are restricted to selecting a fixed number of sensors (discrete selection) cannot benefit from this continuous degree of freedom.

\begin{figure}[h]
    \centering
    \includegraphics[width=0.48\textwidth]{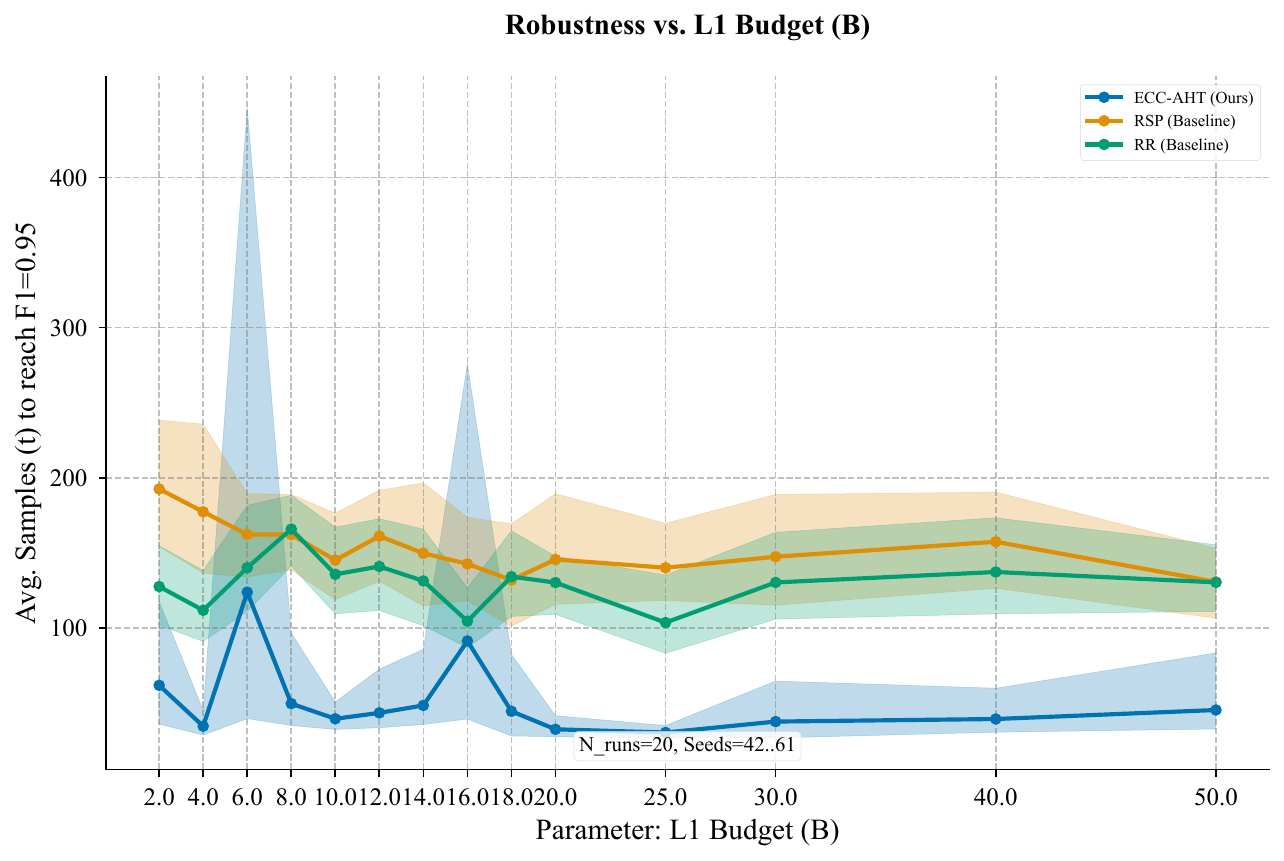}
    \caption{\textbf{Impact of $L_1$ Budget ($B$).} Larger budgets allow for more aggressive noise cancellation, leading to faster detection. This validates the benefit of continuous relaxation over discrete selection.}
    \label{fig:robust_B}
\end{figure}

\subsection{Robustness Across Diverse Correlation Patterns}
\label{app:diverse_patterns}

A central theme of our work is that \emph{correlation is not a nuisance but a resource}. In the primary experiments presented in our main text, we utilized the \textbf{Toeplitz} structure to model the covariance $\Sigma$ between different arms. While Toeplitz is a staple in signal processing, a skeptical reviewer might wonder if ECC-AHT is merely an over-fitted solution to this specific diagonal symmetry. To address this, we argue that a truly intelligent bandit algorithm should adapt its "noise-canceling shield" to whatever geometry the information space presents. We therefore subjected ECC-AHT to a "stress test" across seven additional, fundamentally distinct correlation structures. These patterns are not chosen at random; they represent real-world inductive biases found in fields ranging from climate sensor networks to multi-modal feature spaces in recommendation systems. We demonstrate that ECC-AHT maintains its crushing advantage by dynamically re-aligning its measurement vectors to the specific eigenvectors of these diverse systems.

We begin by formalizing our baseline, the \textbf{Toeplitz} structure. This pattern tells a story of physical proximity in a 1D world. For sensors placed along a pipeline or a single-track road, the signal strength decays exponentially with distance. For a correlation parameter $\rho \in (0,1)$, we define the entries as $\Sigma_{ij} = \rho^{|i-j|}$. Mathematically, it possesses a beautiful, persistent structure where the value of each diagonal remains constant:

$$ \Sigma_{\text{Toeplitz}} = \begin{bmatrix}
1 & \rho & \rho^2 & \dots & \rho^{K-1} \\
\rho & 1 & \rho & \dots & \rho^{K-2} \\
\rho^2 & \rho & 1 & \dots & \rho^{K-3} \\
\vdots & \vdots & \vdots & \ddots & \vdots \\
\rho^{K-1} & \rho^{K-2} & \rho^{K-3} & \dots & 1
\end{bmatrix} $$

Our first departure from locality is the \textbf{Equicorrelation} pattern. This structure simulates a world governed by a "global latent factor." Imagine a network of barometers across a city during a massive high-pressure event; the specific distance between any two sensors matters less than the fact that they are all submerged in the same global phenomenon. We define this as $\Sigma = (1 - \rho) I + \rho \mathbf{1}\mathbf{1}^\top$. Every arm is equally "in sync" with every other arm, forcing the algorithm to find a tiny signal amidst a sea of shared global noise:

$$ \Sigma_{\text{Equicorrelation}} = \begin{bmatrix}
1 & \rho & \rho & \dots & \rho \\
\rho & 1 & \rho & \dots & \rho \\
\rho & \rho & 1 & \dots & \rho \\
\vdots & \vdots & \vdots & \ddots & \vdots \\
\rho & \rho & \rho & \dots & 1
\end{bmatrix} $$

Moving toward modular systems, we examine the \textbf{Block Correlation} structure. This models "community" dynamics, prevalent in data centers where servers in the same rack are tightly coupled but racks are independent. If we divide $K$ arms into blocks of size $M$, the matrix becomes block-diagonal. Each block $B$ is an equicorrelation matrix, representing a cluster where information is a shared resource, yet no information leaks across the block boundaries:

$$ \Sigma_{\text{Block}} = \begin{bmatrix}
B_{M \times M} & 0 & \dots & 0 \\
0 & B_{M \times M} & \dots & 0 \\
\vdots & \vdots & \ddots & \vdots \\
0 & 0 & \dots & B_{M \times M}
\end{bmatrix}, \quad B_{ij} = \begin{cases} 1 & i=j \\ \rho & i \neq j \end{cases} $$

To test the robustness to boundary effects, we introduce the \textbf{Circulant} correlation. This simulates arms arranged in a loop, where the last arm is a neighbor of the first. This is a common topology in circular industrial monitors or periodic time-series data. The correlation depends on the cyclic distance, $\min(|i-j|, K-|i-j|)$, ensuring that information flows seamlessly through the "wrap-around" point:

$$ \Sigma_{\text{Circulant}} = \begin{bmatrix}
c_0 & c_1 & c_2 & \dots & c_{K-1} \\
c_{K-1} & c_0 & c_1 & \dots & c_{K-2} \\
c_{K-2} & c_{K-1} & c_0 & \dots & c_{K-3} \\
\vdots & \vdots & \vdots & \ddots & \vdots \\
c_1 & c_2 & c_3 & \dots & c_0
\end{bmatrix}, \quad c_k = \rho^{\min(k, K-k)} $$

Perhaps the most challenging "story" is the \textbf{Graph-based} correlation. In many modern Machine learning (ML) tasks, data lives on irregular graphs (social networks, neural connectomes). Here, we model arms as nodes in an Erdos-Renyi graph $G(K, p)$ with $p=0.05$, representing a sparse connectivity structure. Let $\mathbf{A} \in \{0,1\}^{K \times K}$ 
denote the adjacency matrix of $G$, where $A_{ij}=1$ if nodes $i$ and $j$ 
are connected. We then define the dependency through the precision matrix $Q = \Sigma^{-1} = I - \alpha A$. However, ensuring $Q$ is positive-definite is mathematically non-trivial. If $\alpha$ is too high, the system collapses into non-convexity. To keep $\rho$ meaningful and ensure stability, we apply a rigorous \textbf{Spectral Radius Normalization}. We compute the spectral radius $\lambda_{\max}(A) = \max_i |\lambda_i(A)|$. To guarantee $I - \alpha A \succ 0$, we must have $\alpha < 1/\lambda_{\max}(A)$. We therefore set the effective coupling as $\alpha = \frac{0.95 \cdot \rho}{\lambda_{\max}(A)}$, mapping our abstract $\rho \in [0, 1]$ to a stable, topology-preserving regime:

$$\Sigma_{\text{Graph}} = \text{diag}(D)^{-1/2} (I - \alpha A)^{-1} \text{diag}(D)^{-1/2}, \quad D = \text{diag}((I - \alpha A)^{-1})$$

Next, we consider the \textbf{Exponential} and \textbf{RBF (Radial Basis Function)} kernels. These are the workhorses of Gaussian Processes in ML. The \textbf{Exponential} kernel, $\Sigma_{ij} = \exp(-|i-j|/\ell)$, provides a smooth, heavy-tailed information decay. The \textbf{RBF} kernel, defined as $\Sigma_{ij} = \exp(-|i-j|^2 / 2\ell^2)$, represents "infinite smoothness." In an RBF world, if you know the state of arm $i$, you almost perfectly know arm $i+1$. This extreme coherence provides a unique challenge for resolution, which we analyze later.

Finally, we introduce the \textbf{Kronecker} structure, which is particularly relevant for "multi-attribute" arms. Imagine each arm has a "location" property and a "type" property. The correlation between two arms is the product of their spatial similarity $\Sigma_{\text{space}}$ and their type similarity $\Sigma_{\text{type}}$. This models high-dimensional latent spaces where the total covariance is a factorized tensor product:

$$ \Sigma_{\text{Kronecker}} = \Sigma_{\text{type}} \otimes \Sigma_{\text{space}} = \begin{bmatrix}
\sigma_{11}^T \Sigma_S & \sigma_{12}^T \Sigma_S & \dots \\
\sigma_{21}^T \Sigma_S & \sigma_{22}^T \Sigma_S & \dots \\
\vdots & \vdots & \ddots
\end{bmatrix} $$

\begin{figure}[htbp]
    \centering
    \includegraphics[width=0.6\textwidth]{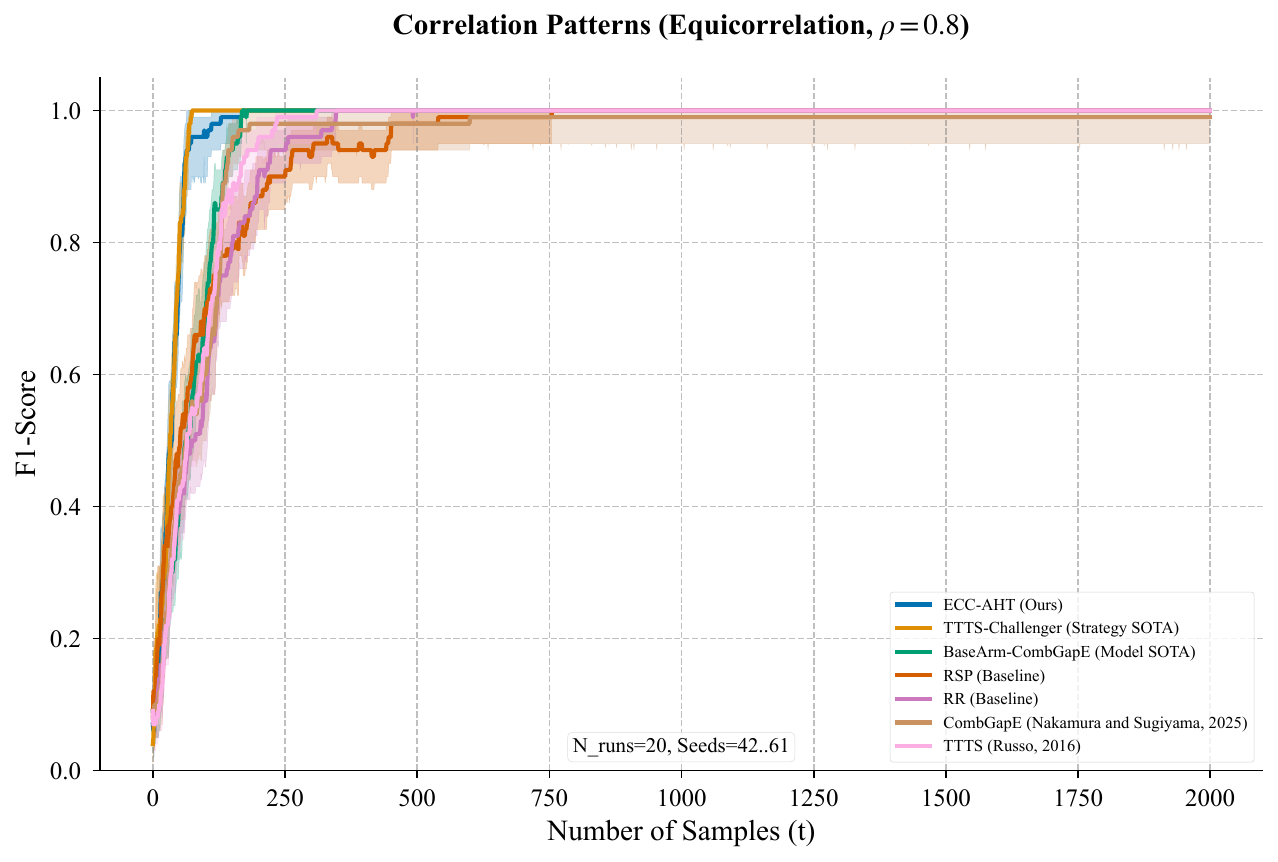}
    \caption{Equicorrelation: Under global shared noise, ECC-AHT successfully isolates anomalous signals from the background factor but worse than TTTS-Challenger.}
    \label{fig:equi}
\end{figure}

\begin{figure}[htbp]
    \centering
    \includegraphics[width=0.6\textwidth]{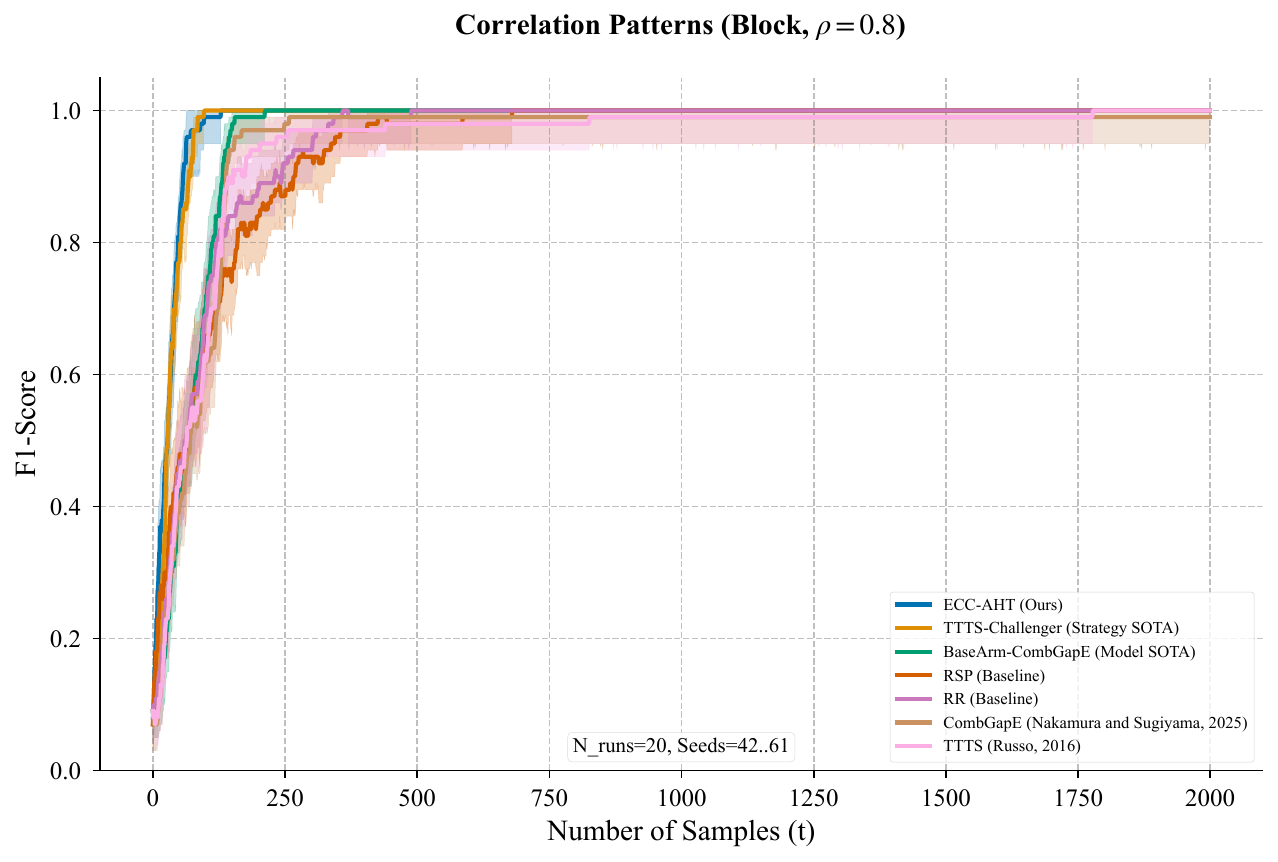}
    \caption{Block Correlation: The algorithm effectively "ignores" independent blocks, focusing its budget on discovering anomalies within clusters.}
    \label{fig:block}
\end{figure}

\begin{figure}[htbp]
    \centering
    \includegraphics[width=0.6\textwidth]{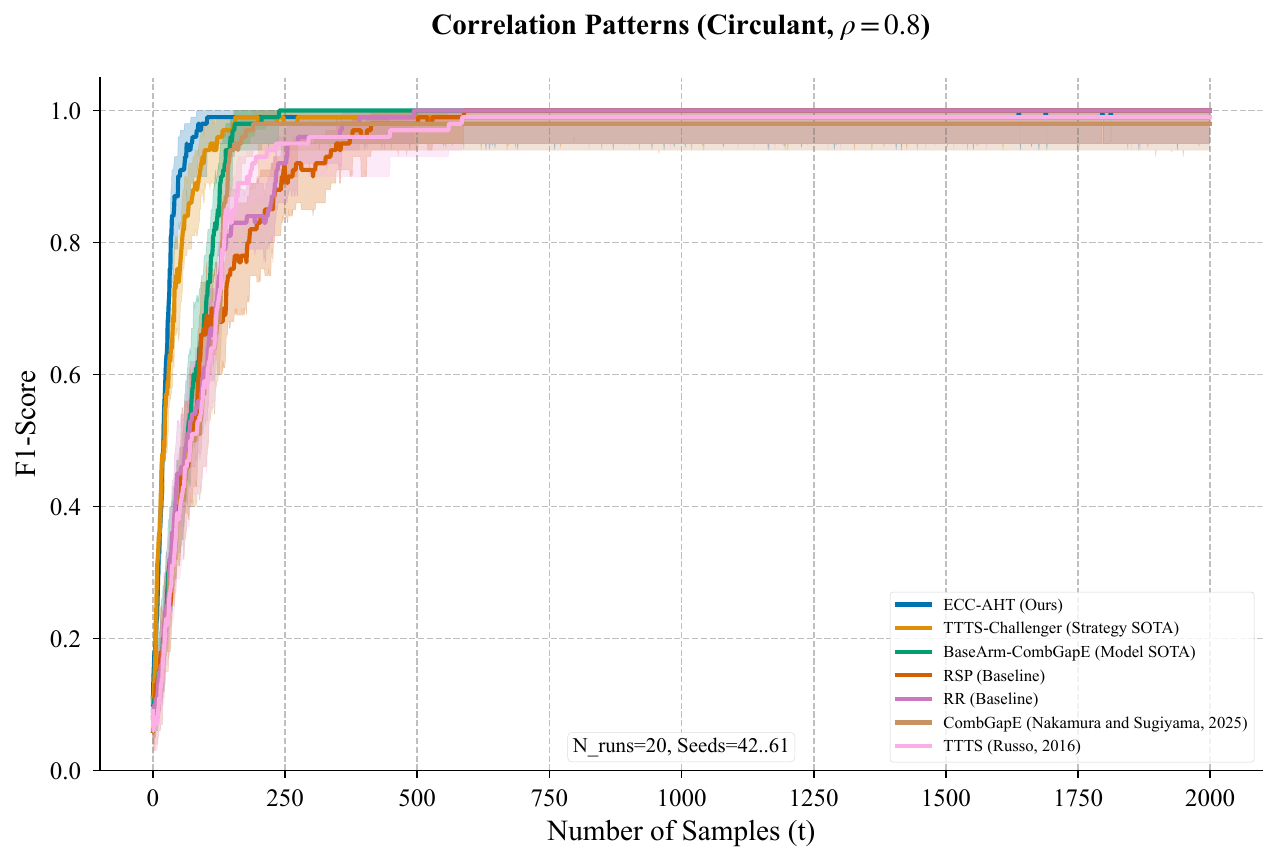}
    \caption{Circulant Correlation: Symmetry in the loop-topology allows for periodic optimization strategies that outperform linear baselines.}
    \label{fig:circulant}
\end{figure}

\begin{figure}[htbp]
    \centering
    \includegraphics[width=0.6\textwidth]{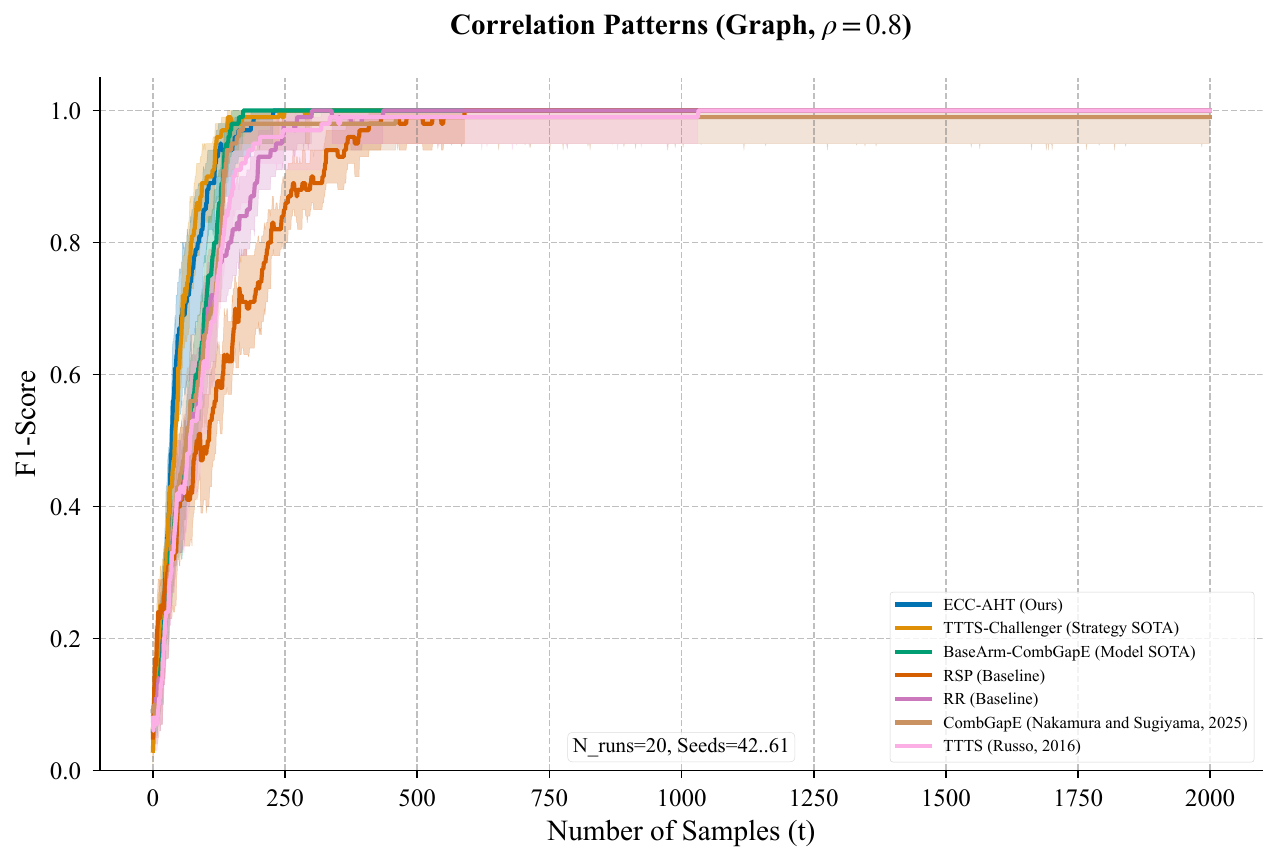}
    \caption{Graph-based Correlation: ECC-AHT demonstrates high robustness to irregular, non-Euclidean connectivity patterns.}
    \label{fig:graph}
\end{figure}

\begin{figure}[htbp]
    \centering
    \includegraphics[width=0.6\textwidth]{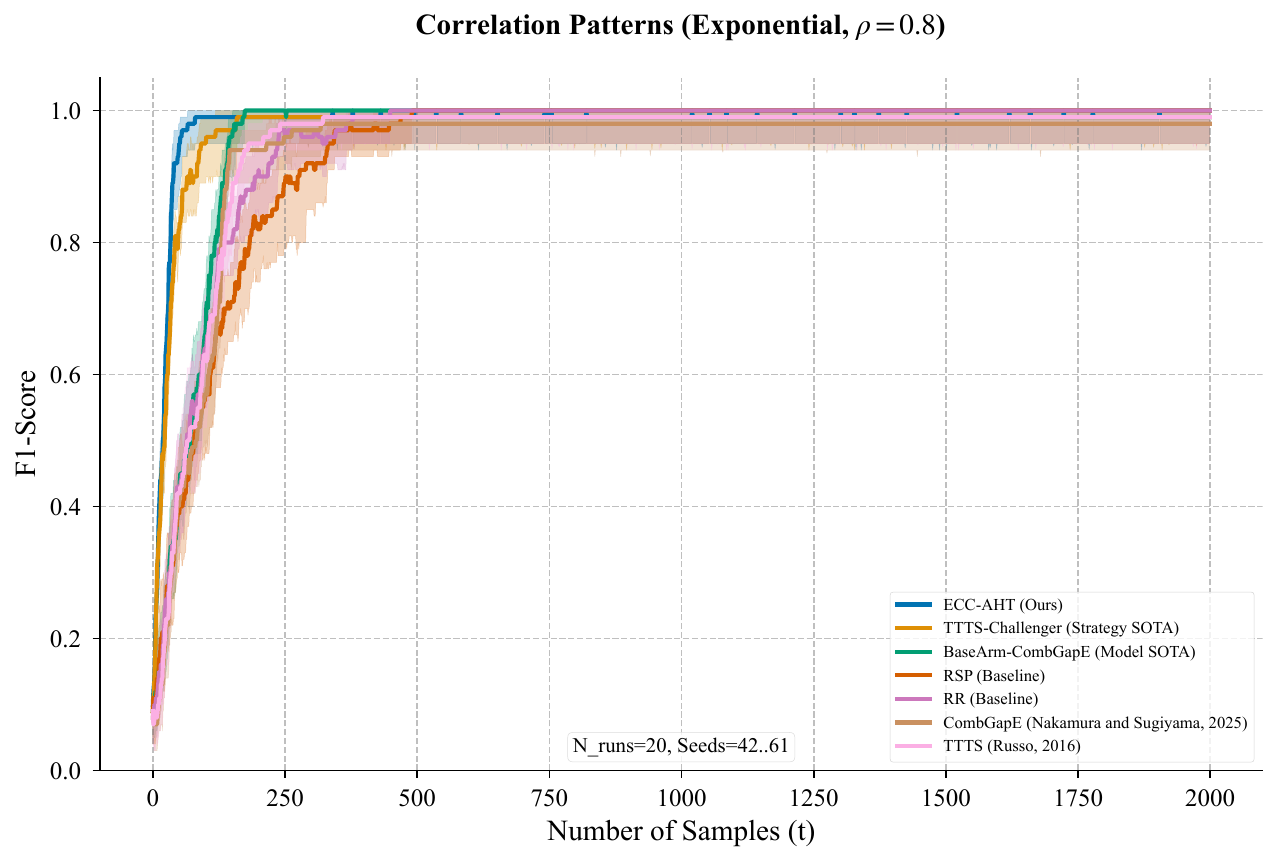}
    \caption{Exponential Kernel: Smooth information decay across the arm space is handled efficiently by the Champion-Challenger logic.}
    \label{fig:exp}
\end{figure}

\begin{figure}[htbp]
    \centering
    \includegraphics[width=0.6\textwidth]{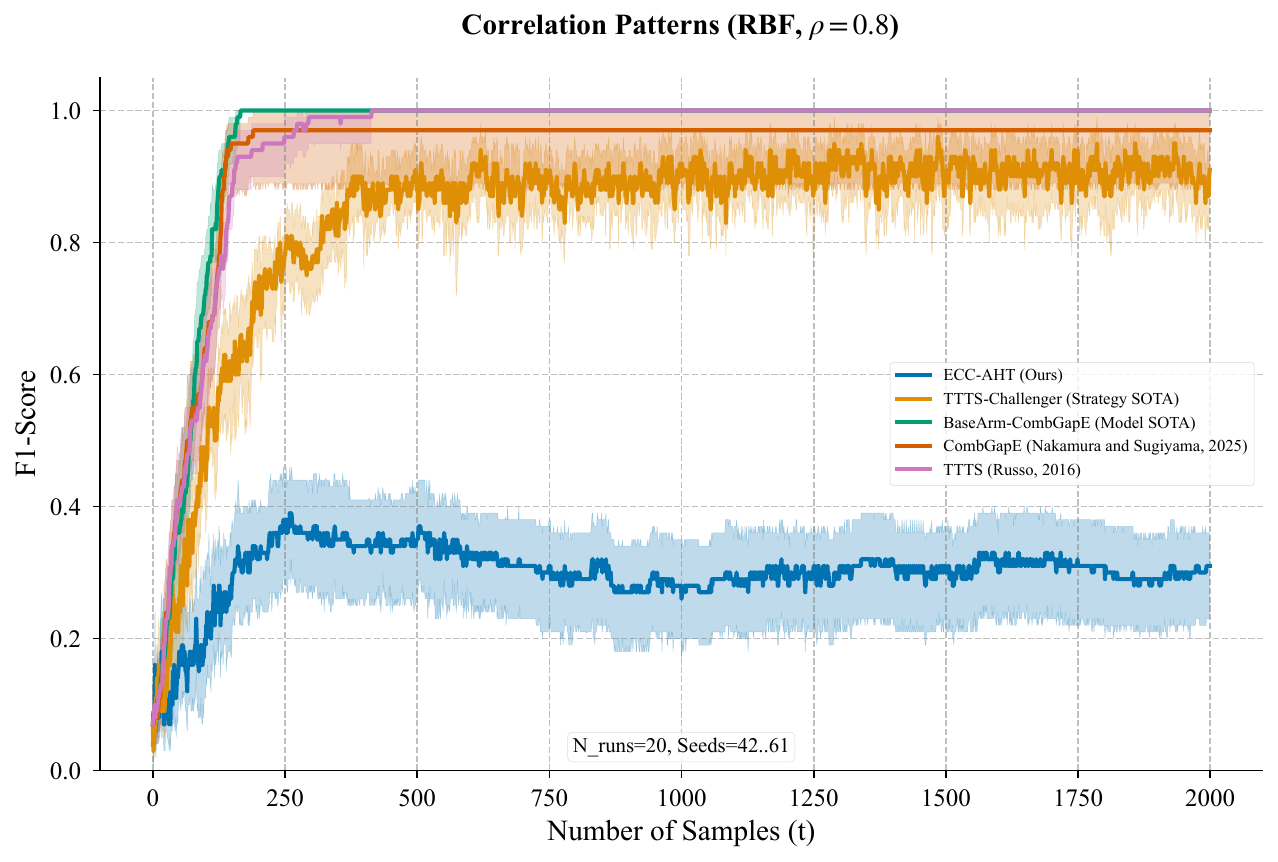}
    \caption{RBF Kernel: The extreme coherence causes an unexpected resolution drop, marking the boundary of the original ECC-AHT.}
    \label{fig:rbf}
\end{figure}

\begin{figure}[htbp]
    \centering
    \includegraphics[width=0.6\textwidth]{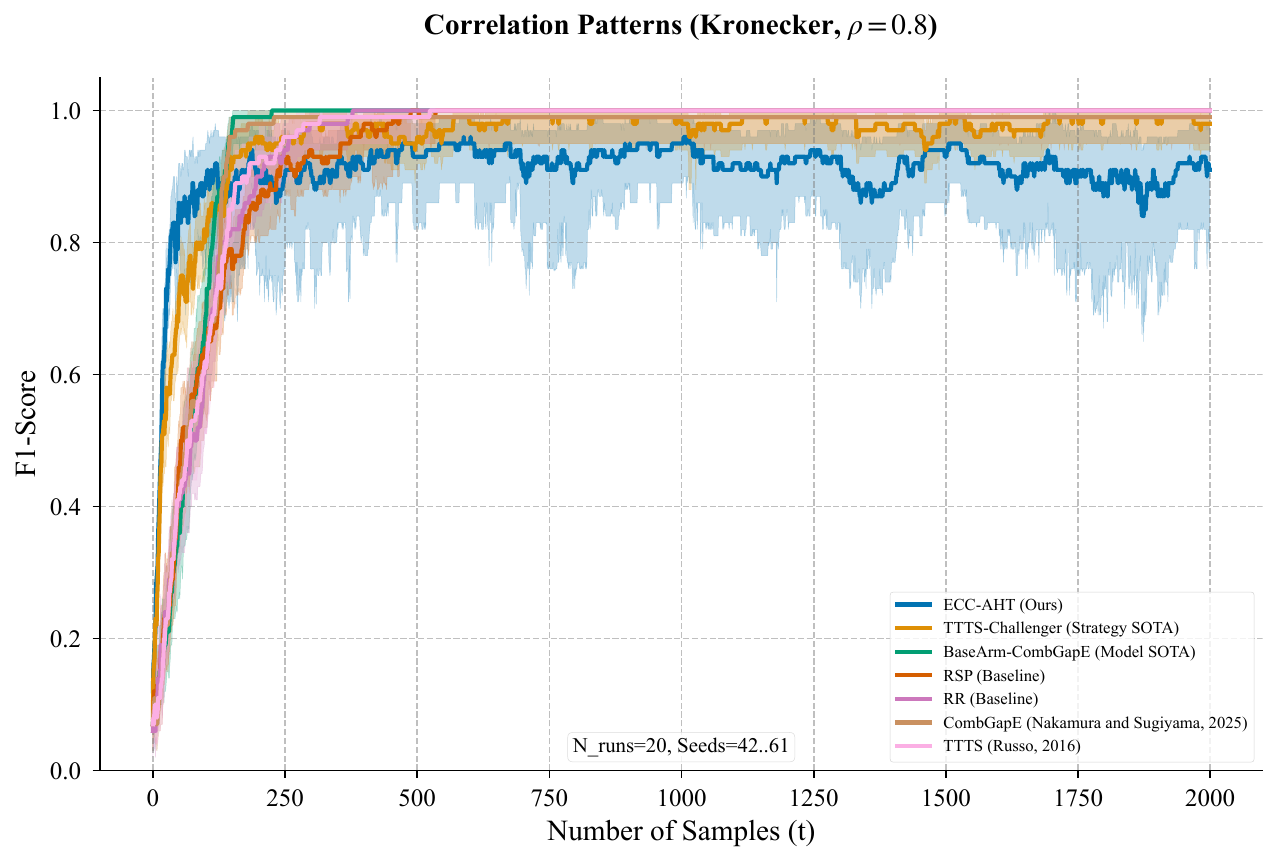}
    \caption{Kronecker Structure: Multi-dimensional dependencies present a "blurred" resolution case, analyzed further in the next section.}
    \label{fig:kron}
\end{figure}

The results of our experiments (20 independent runs, 2000 samples) are visualized in the figures below. Across Toeplitz, Equicorrelation, Block, Circulant, Graph, and Exponential patterns, ECC-AHT maintains a significant lead. It "sees" the structure and builds measurement vectors that are far more efficient than the one-at-a-time approach of RR or the purely statistical approach of TTTS. However, in the RBF (Figure~\ref{fig:rbf}) and Kronecker (Figure~\ref{fig:kron}) cases, we encountered a fascinating performance drop. The extreme coherence of these matrices led to "Information Smearing," where the algorithm identified the correct "neighborhood" of an anomaly but failed to pinpoint the exact arm. Furthermore, in the Equicorrelation (Figure~\ref{fig:equi}) case, ECC-AHT performs worse than TTTS-Challenger. We view this not as a defeat, but as a window into the deep bias-variance tradeoffs of high-dimensional optimization. We provide an exhaustive analysis of these failure modes and a "Robust" variant of our algorithm in Appendix~\ref{app:limitations}.

\section{Limitations and Spectral Characterization of ECC-AHT}
\label{app:limitations}

This appendix studies when ECC-AHT works well and when it does not.
We focus on one question.
What property of the noise correlation controls the success of ECC-AHT?

We show that the answer does not depend on the specific correlation pattern.
It depends on the spectral structure of the noise covariance.
More precisely, it depends on its effective rank.

We organize this appendix as follows.
We first report a clear empirical failure under several correlation patterns.
We then show that these failures share the same spectral feature.
We next run controlled experiments to isolate this feature.
We finally connect these results to a sufficient condition based on effective rank.
This condition explains both the strengths and the limits of ECC-AHT.

\subsection{Empirical Observation: Failure under Extreme Spectral Concentration}

We begin with an empirical observation.
ECC-AHT performs strongly on many correlation patterns.
These include Toeplitz, Block, Circulant, Graph-based, and Exponential correlations.
In these settings, ECC-AHT consistently matches or outperforms existing methods.

However, ECC-AHT does not always win.
We observe clear performance drops under three patterns.
These patterns are Equicorrelation, RBF kernel correlation, and Kronecker correlation.
In these cases, ECC-AHT often ranks second.
It sometimes loses to Bayesian-style baselines such as TTTS-Challenger.

At first glance, these patterns seem unrelated.
Equicorrelation enforces a single shared factor.
RBF kernels arise from smooth latent processes.
Kronecker structures encode separable dependencies.
There is no obvious structural reason to group them together.

We therefore measure their spectral structure.
Table~\ref{tab:effective_rank_patterns} reports the effective rank of each pattern.
We report both Shannon effective rank and participation ratio effective rank.

\begin{table}[h]
\centering
\caption{Effective rank of different correlation patterns.}
\label{tab:effective_rank_patterns}
\begin{tabular}{lcc}
\toprule
Pattern & Shannon ER & PR ER \\
\midrule
Toeplitz & 46.64 & 28.58 \\
Equicorrelation & 4.30 & 1.56 \\
Block & 21.54 & 12.08 \\
Circulant & 46.08 & 28.10 \\
Graph & 121.70 & 113.60 \\
Exponential & 46.64 & 28.58 \\
RBF & 3.80 & 3.18 \\
Kronecker & 9.41 & 3.60 \\
\bottomrule
\end{tabular}
\end{table}

This table reveals a simple and strong pattern.
All cases where ECC-AHT struggles have extremely low effective rank.
All cases where ECC-AHT excels have moderate or high effective rank.

This observation already suggests a unifying explanation.
The failure of ECC-AHT does not depend on the correlation pattern itself.
It depends on how concentrated the spectrum of the noise covariance is.

\subsection{Effective Rank as a Unifying Diagnostic}

We now formalize the notion of effective rank.
Let $\Sigma \in \mathbb{R}^{K \times K}$ be a correlation matrix.
Let $\lambda_1, \ldots, \lambda_K$ be its eigenvalues.
We assume $\displaystyle\sum_{i=1}^K \lambda_i = K$.

The Shannon effective rank is defined as
\[
r_{\mathrm{eff}}(\Sigma)
= \exp\left(
- \displaystyle\sum_{i=1}^K p_i \log p_i
\right),
\quad
p_i = \frac{\lambda_i}{\displaystyle\sum_j \lambda_j}.
\]

The participation ratio effective rank is defined as
\[
r_{\mathrm{PR}}(\Sigma)
= \frac{(\displaystyle\sum_i \lambda_i)^2}{\displaystyle\sum_i \lambda_i^2}.
\]

Both quantities measure spectral diversity.
They take large values when eigenvalues are spread out.
They take small values when a few eigenvalues dominate.

In this appendix, we treat effective rank as a scalar proxy.
It summarizes how many independent noise directions exist.
It ignores the exact shape of the spectrum.
This choice fits our goal.
We aim to explain algorithmic behavior at a coarse structural level.

\subsection{Controlled Spectral Mixing Experiment}

The previous observation is correlational.
We now run a controlled experiment.
Our goal is to isolate the effect of effective rank.

We construct a family of correlation matrices.
We interpolate between a full-rank identity matrix and a rank-one matrix.
We keep marginal variances fixed.
We vary only the spectral concentration.

For each matrix, we compute its Shannon effective rank.
We then run ECC-AHT until convergence.
We record the final F1 score.

Figure~\ref{fig:exp2_f1_vs_er} shows the result.

\begin{figure}[h]
\centering
\includegraphics[width=0.85\linewidth]{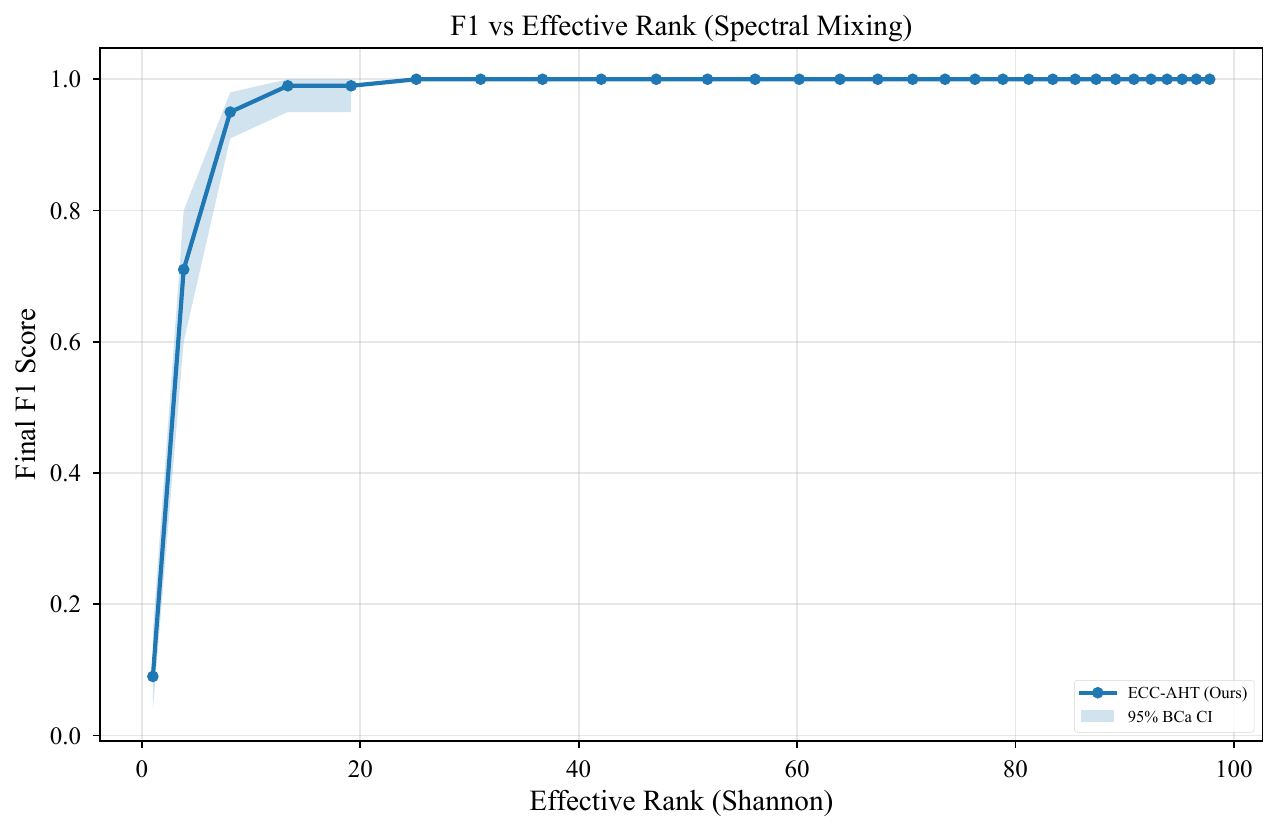}
\caption{Final F1 score versus Shannon effective rank.}
\label{fig:exp2_f1_vs_er}
\end{figure}

The behavior is sharp.
When the effective rank stays above a moderate threshold,
ECC-AHT achieves stable and high F1 scores.
When the effective rank drops below this threshold,
the F1 score collapses rapidly.

This pattern resembles a phase transition.
The drop is not gradual.
Small changes in effective rank lead to large performance losses.
This behavior suggests a necessary spectral condition.
ECC-AHT requires enough spectral diversity to operate reliably.

\subsection{Natural Kernel Validation: RBF and Kronecker}

We next test whether this relation holds in natural settings.
We focus on RBF and Kronecker correlations.
These patterns arise in many applications.
They also induce highly concentrated spectra.

Figure~\ref{fig:exp3_rbf} plots the final F1 score against effective rank for RBF kernels.

\begin{figure}[h]
\centering
\includegraphics[width=0.85\linewidth]{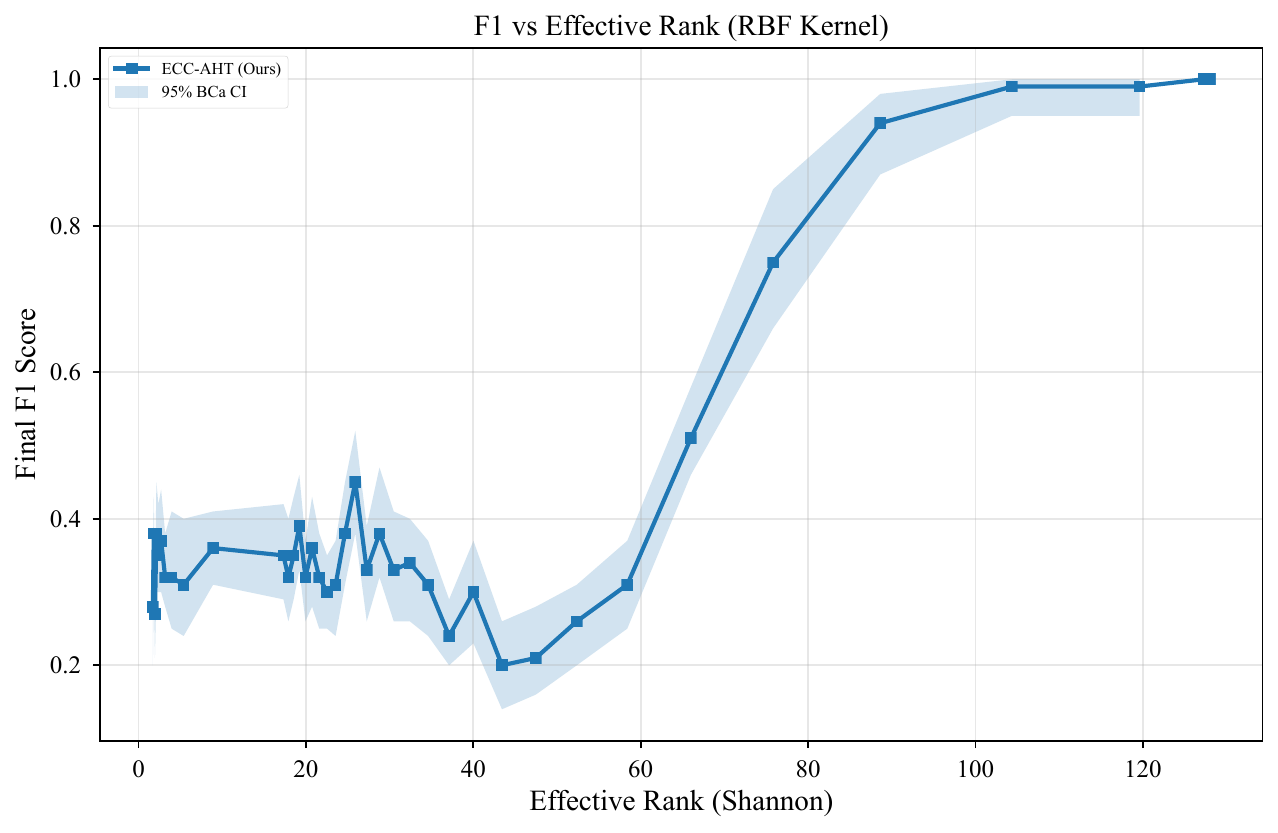}
\caption{RBF kernel: F1 score versus effective rank.}
\label{fig:exp3_rbf}
\end{figure}

The relation is almost monotone.
Higher effective rank leads to higher F1 scores.
Only the very low-rank regime shows instability.
This regime matches the failures seen earlier.

We further test two mitigation strategies.
The first adds a small diagonal regularization: $\Sigma_{reg}=\Sigma+0.01I$.
The second lowers the correlation strength from 0.8 to 0.5.

Figures~\ref{fig:eq_low}--\ref{fig:rbf_low} report these results.

\begin{figure}[h]
    \centering
    \begin{subfigure}[b]{0.48\linewidth}
        \centering
        \includegraphics[width=\linewidth]{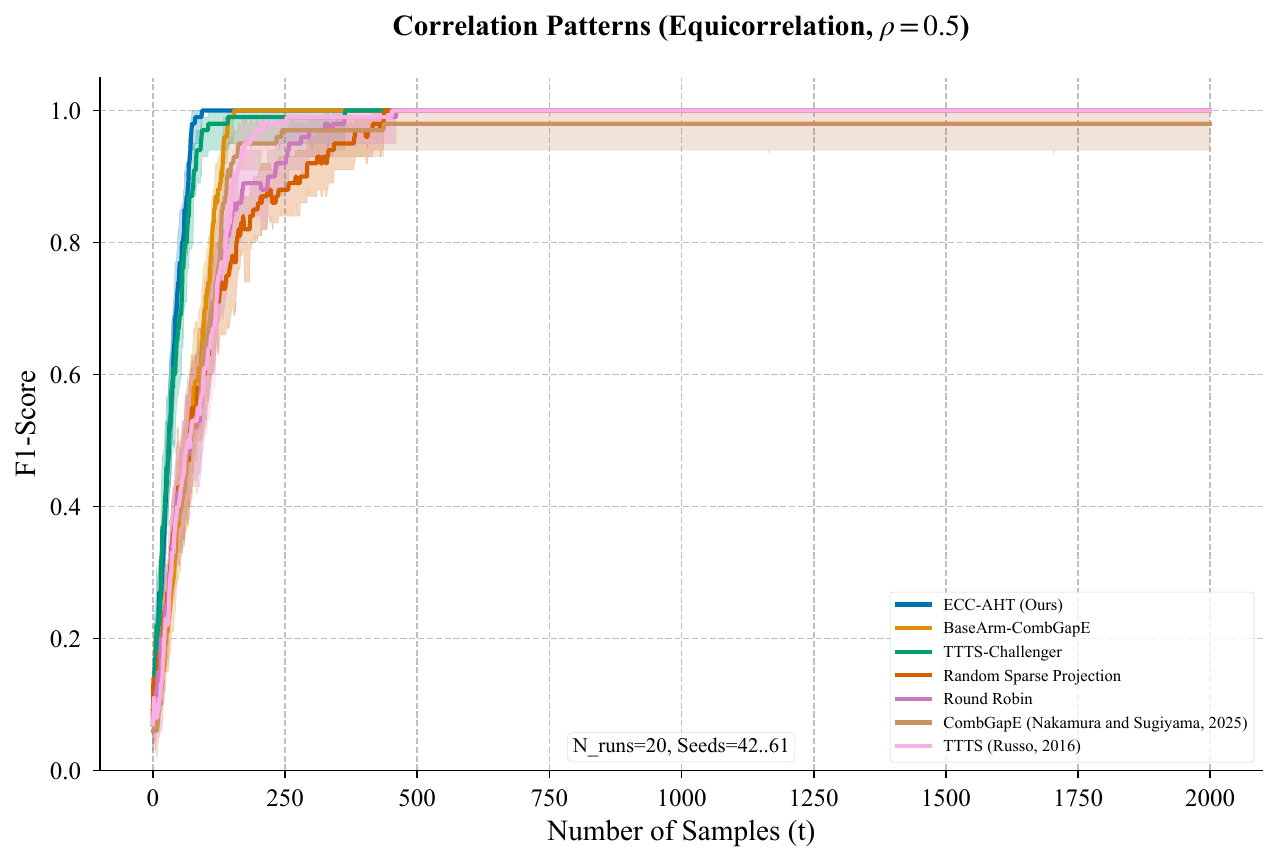}
        \caption{lower correlation strength result}
    \end{subfigure}
    \hfill
    \begin{subfigure}[b]{0.48\linewidth}
        \centering
        \includegraphics[width=\linewidth]{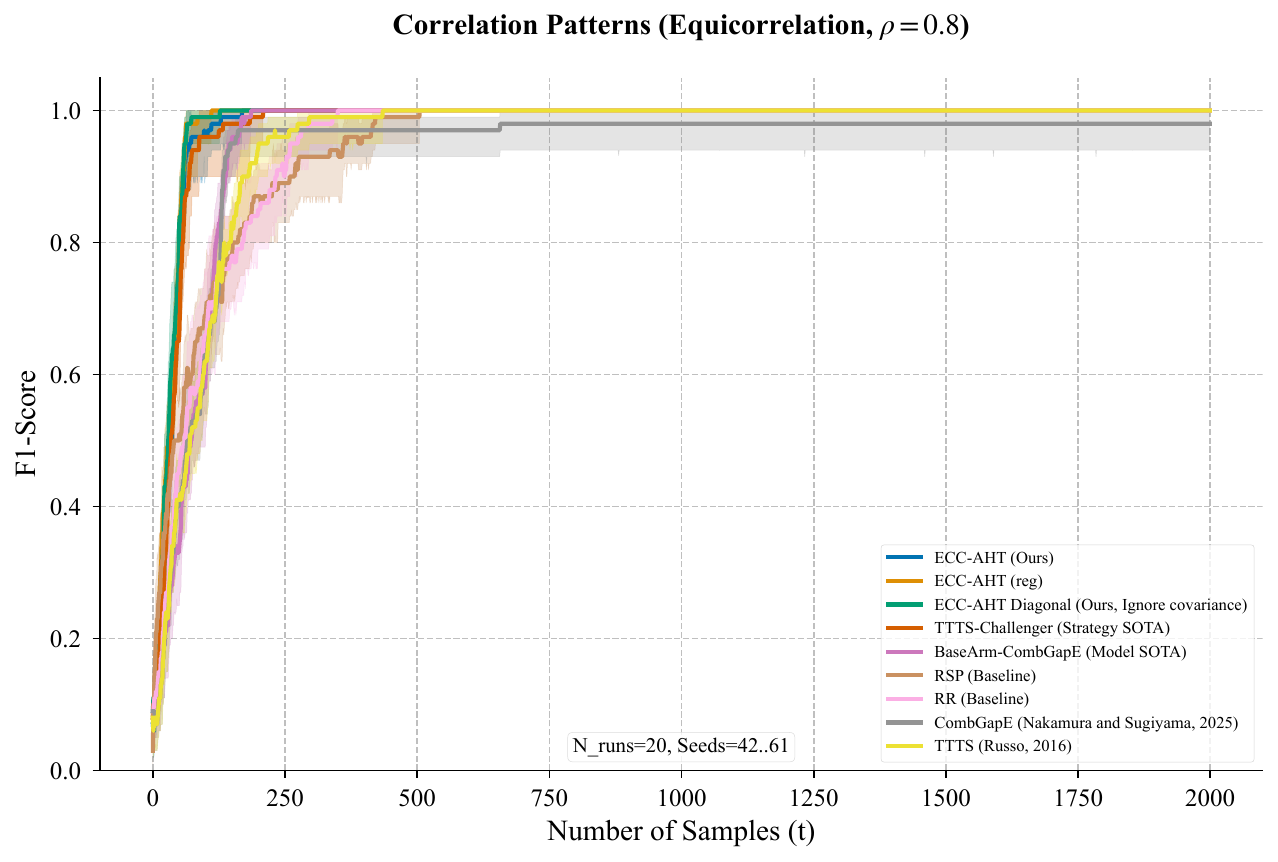}
        \caption{regularization result}
    \end{subfigure}
\caption{Equicorrelation with reduced correlation and with regularization.}
\label{fig:eq_low}
\end{figure}

\begin{figure}[h]
    \centering
    \begin{subfigure}[b]{0.48\linewidth}
        \centering
        \includegraphics[width=\linewidth]{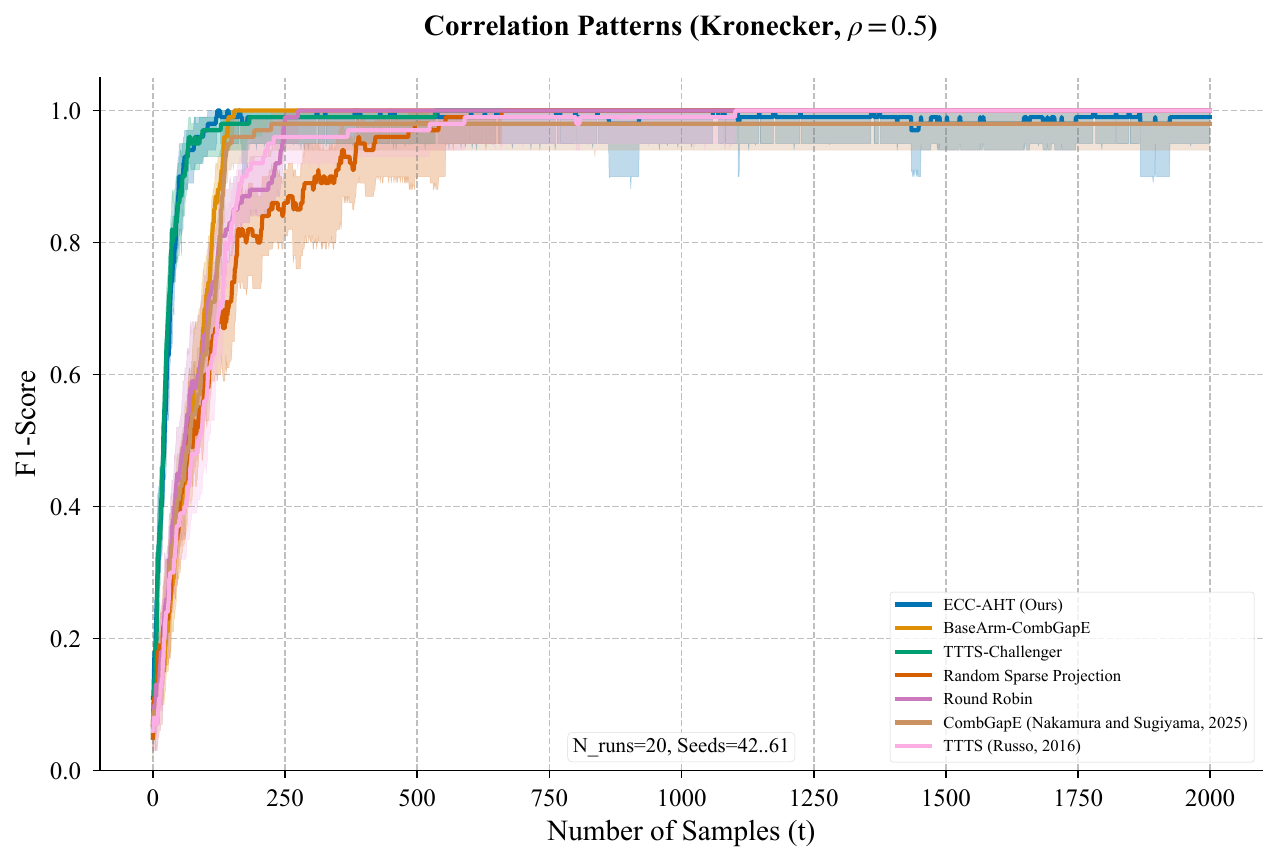}
        \caption{lower correlation strength result}
    \end{subfigure}
    \hfill
    \begin{subfigure}[b]{0.48\linewidth}
        \centering
        \includegraphics[width=\linewidth]{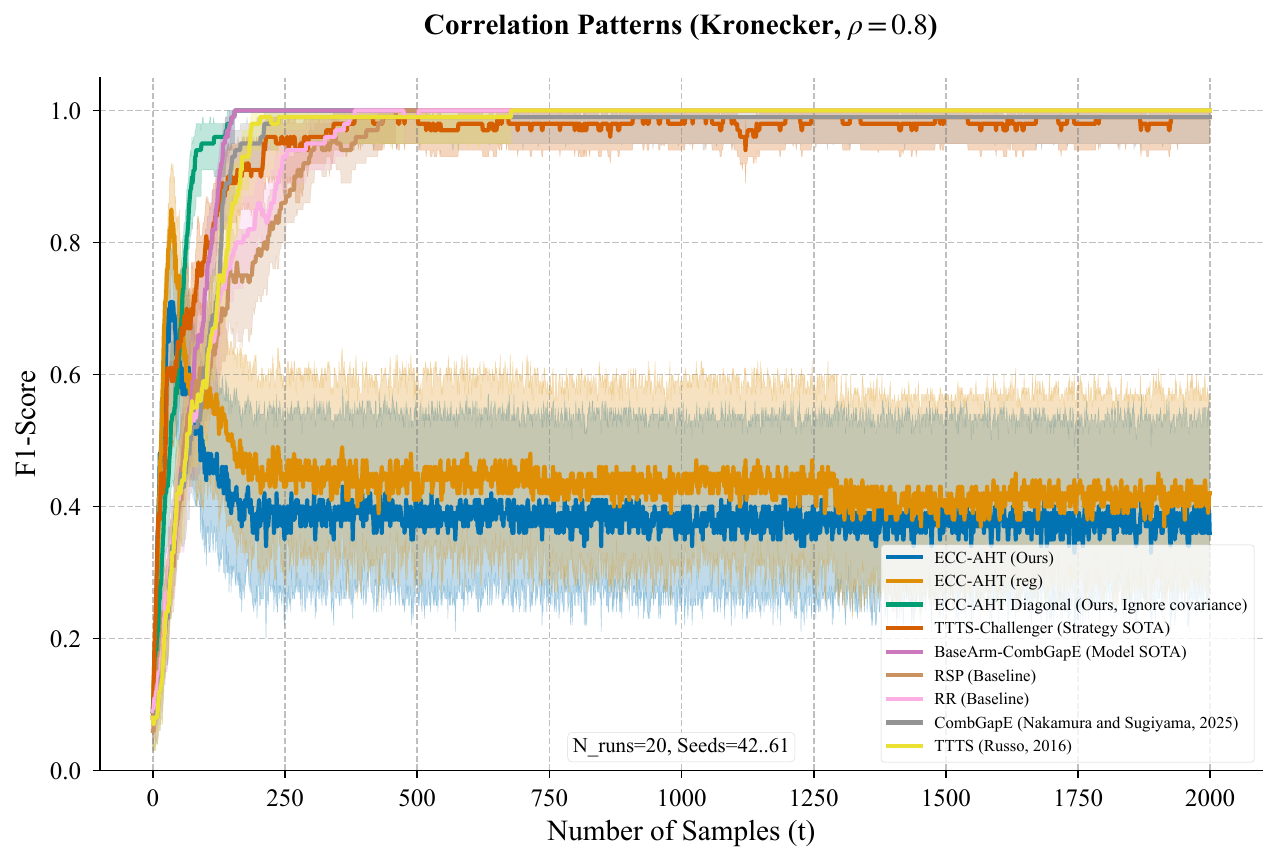}
        \caption{regularization result}
    \end{subfigure}
\caption{Kronecker correlation with reduced correlation and with regularization.}
\label{fig:kron_reg}
\end{figure}

\begin{figure}[h]
    \centering
    \begin{subfigure}[b]{0.48\linewidth}
        \centering
        \includegraphics[width=\linewidth]{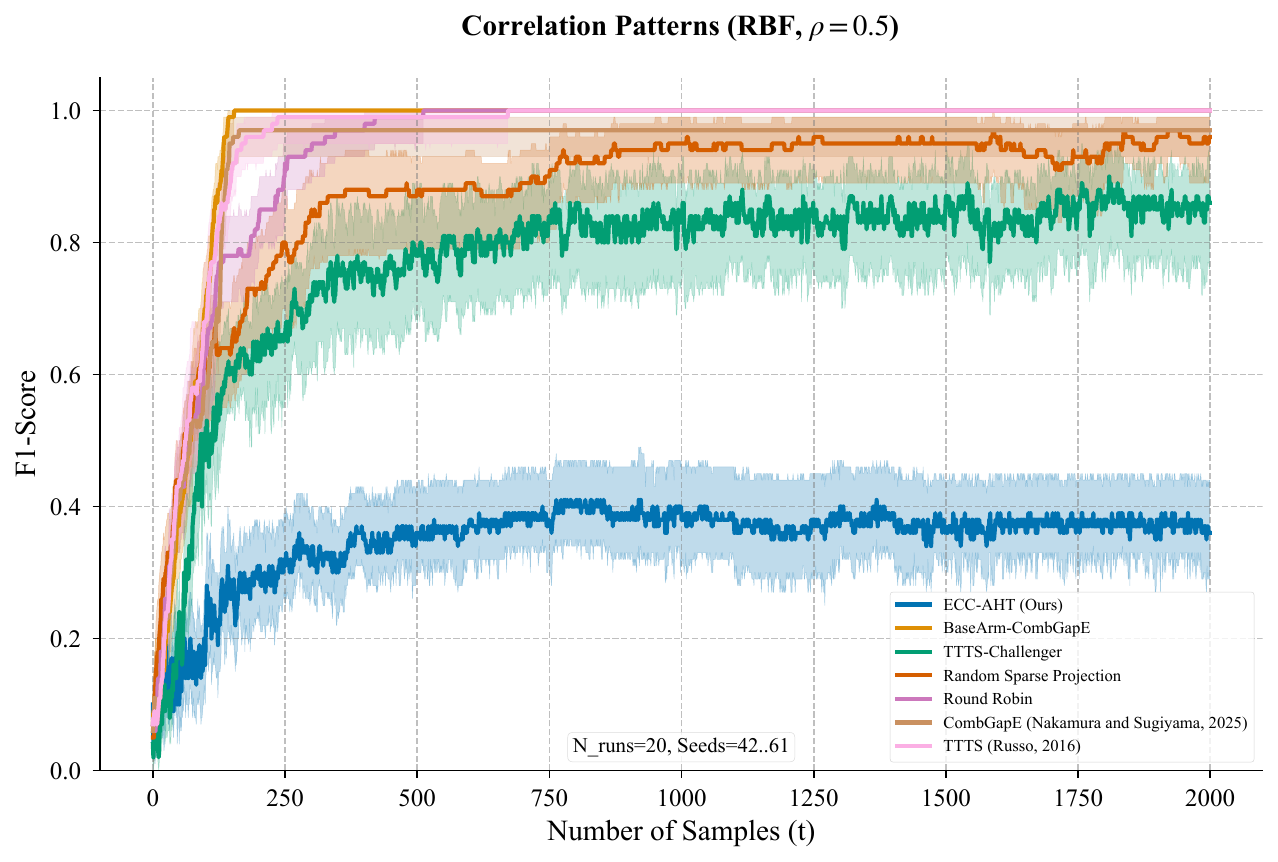}
        \caption{lower correlation strength result}
    \end{subfigure}
    \hfill
    \begin{subfigure}[b]{0.48\linewidth}
        \centering
        \includegraphics[width=\linewidth]{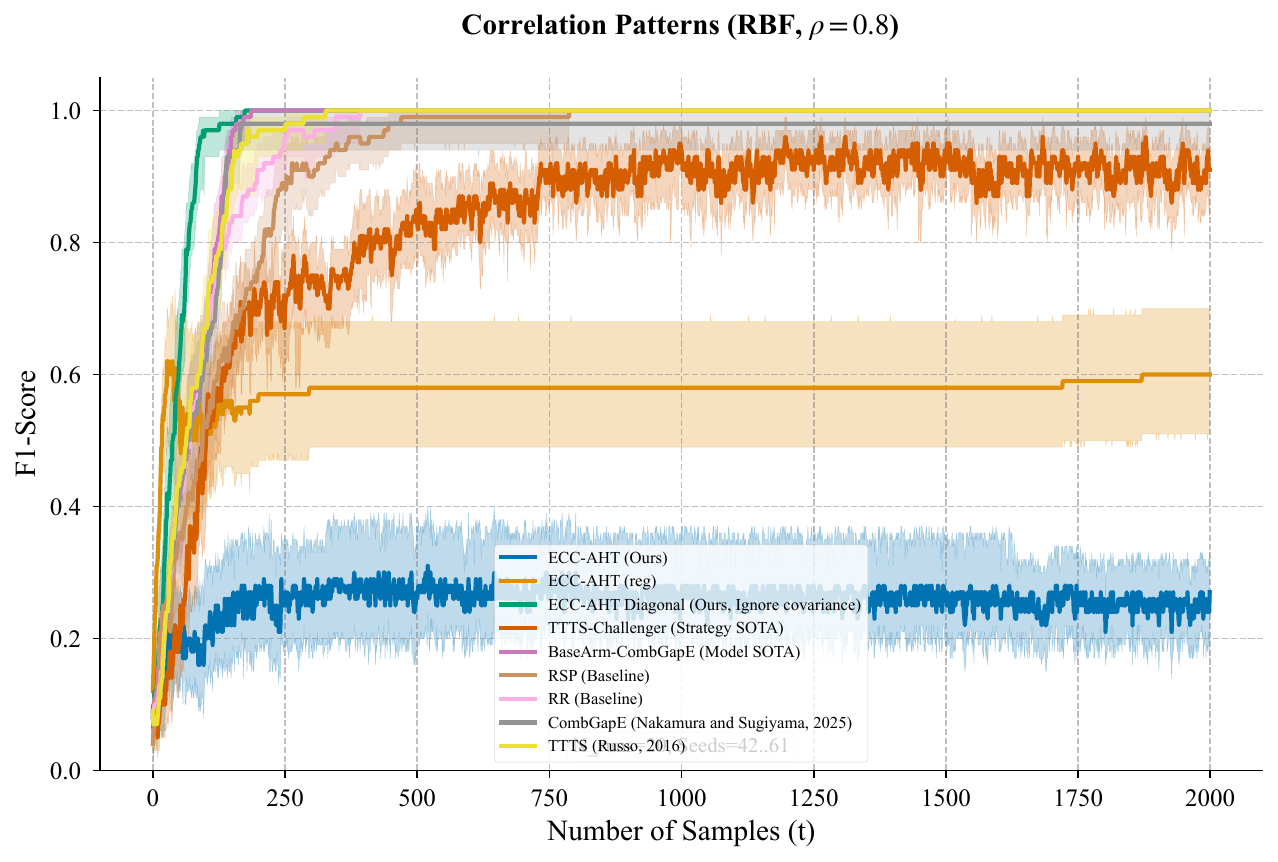}
        \caption{regularization result}
    \end{subfigure}
\caption{RBF correlation with reduced correlation and with regularization.}
\label{fig:rbf_low}
\end{figure}

Both strategies increase effective rank.
Both lead to improved performance.
In contrast, similar regularization does not help in high-rank settings.
Figures~\ref{fig:toe_reg}--\ref{fig:graph_reg} show this behavior.

\begin{figure}[h]
\centering
\includegraphics[width=0.85\linewidth]{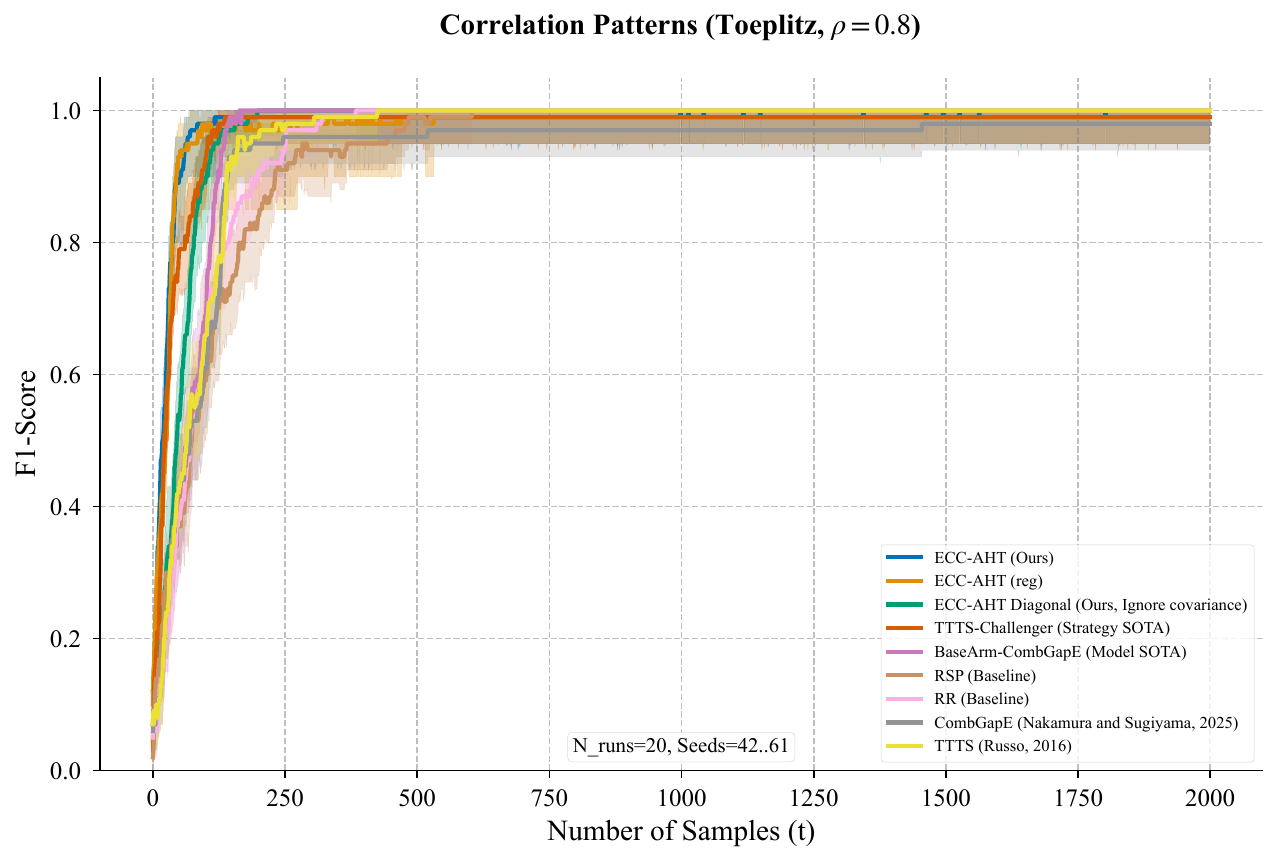}
\caption{Toeplitz correlation with regularization.}
\label{fig:toe_reg}
\end{figure}

\begin{figure}[h]
\centering
\includegraphics[width=0.85\linewidth]{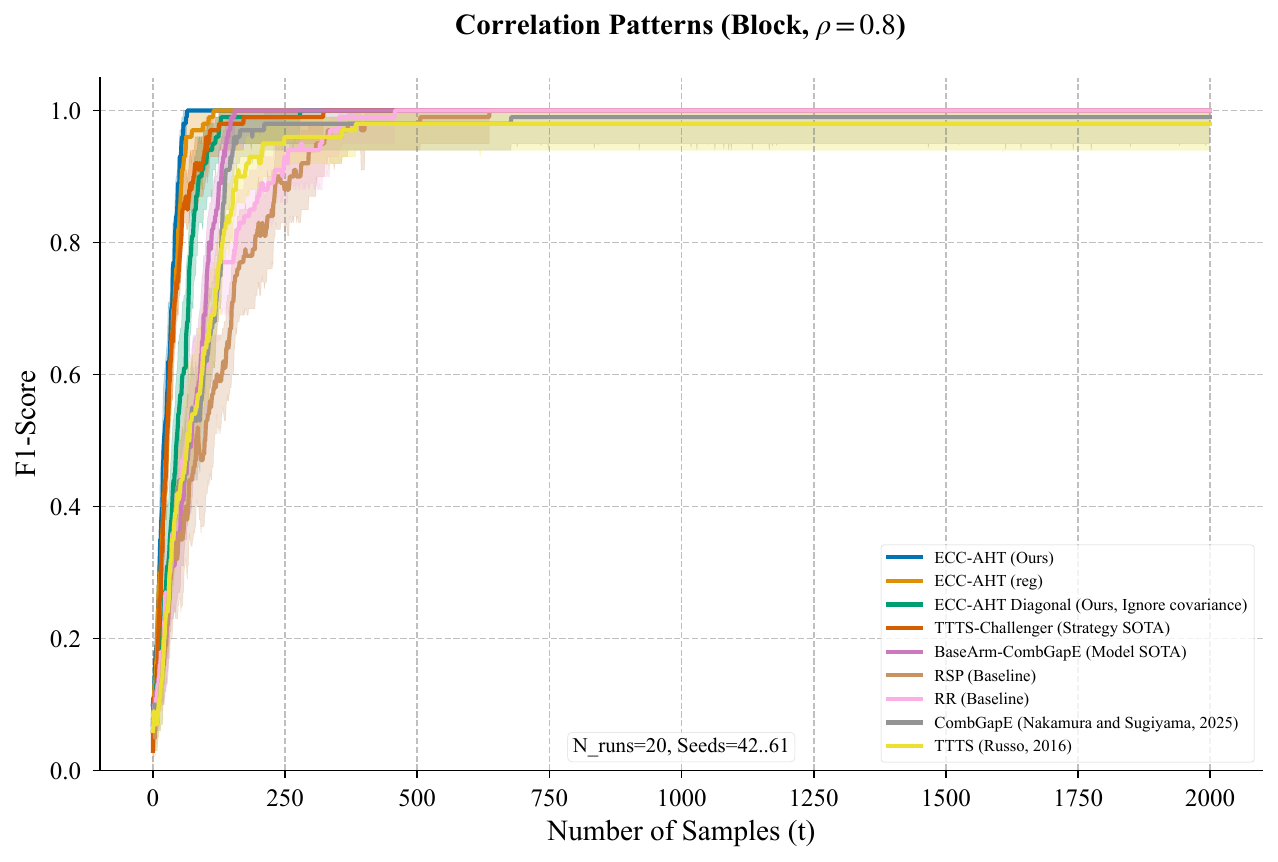}
\caption{Block correlation with regularization.}
\label{fig:block_reg}
\end{figure}

\begin{figure}[h]
\centering
\includegraphics[width=0.85\linewidth]{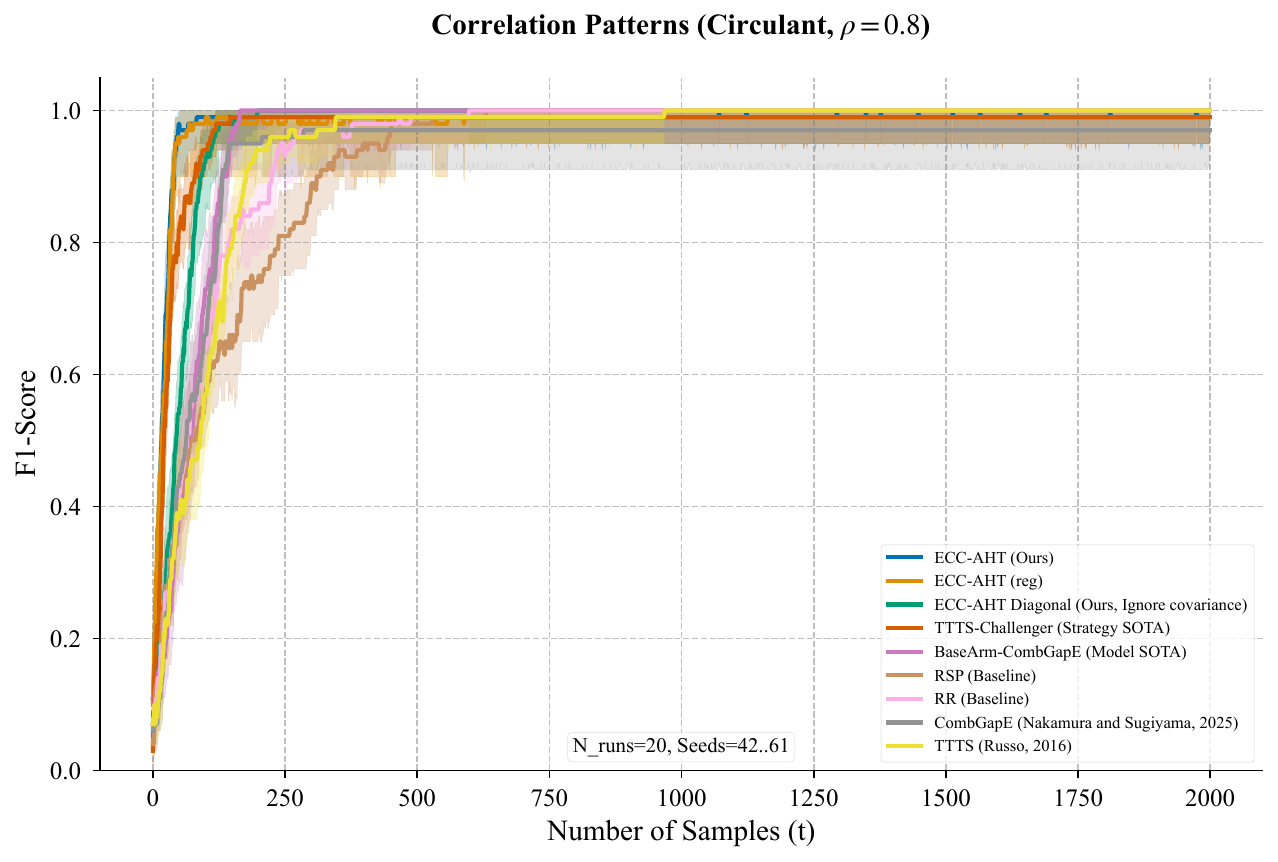}
\caption{Circulant correlation with regularization.}
\label{fig:cir_reg}
\end{figure}

\begin{figure}[h]
\centering
\includegraphics[width=0.85\linewidth]{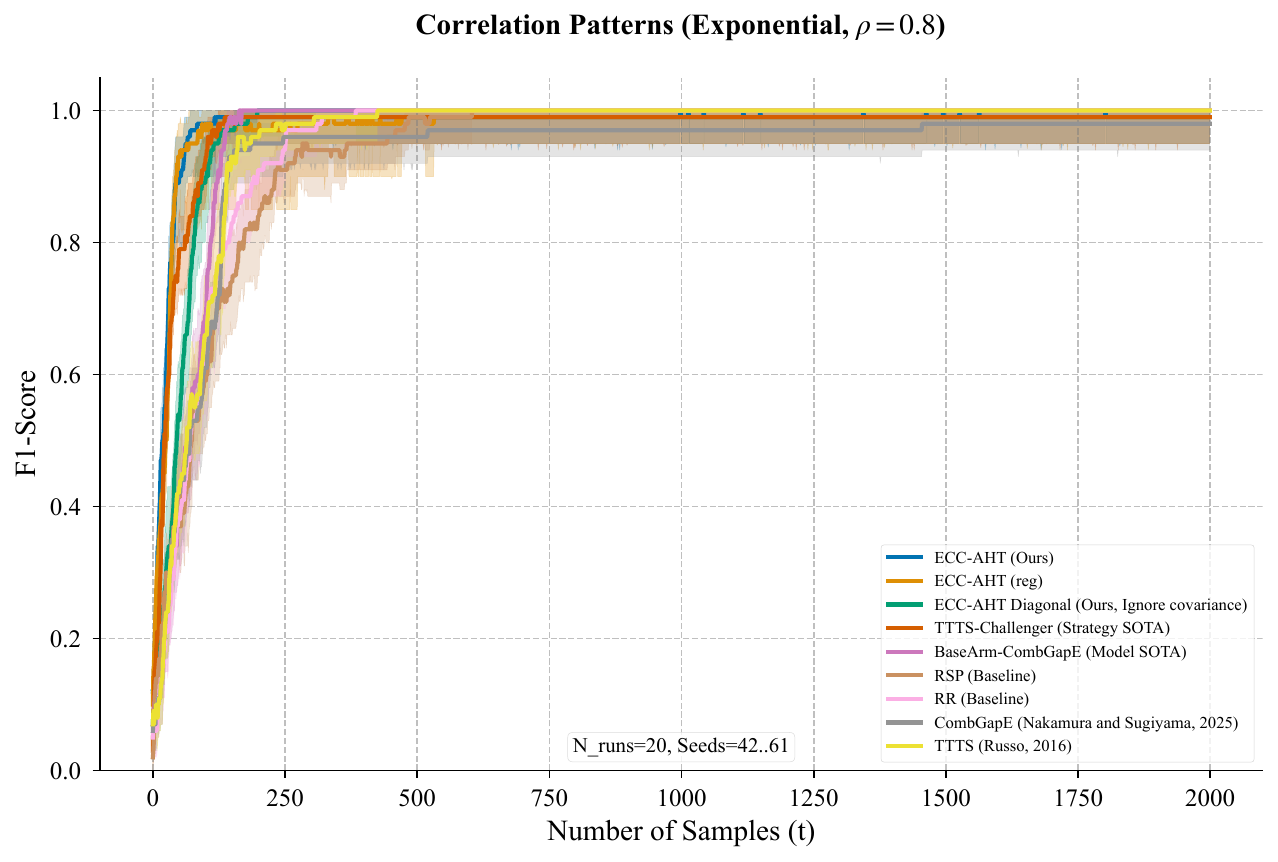}
\caption{Exponential correlation with regularization.}
\label{fig:exp_reg}
\end{figure}

\begin{figure}[h]
\centering
\includegraphics[width=0.85\linewidth]{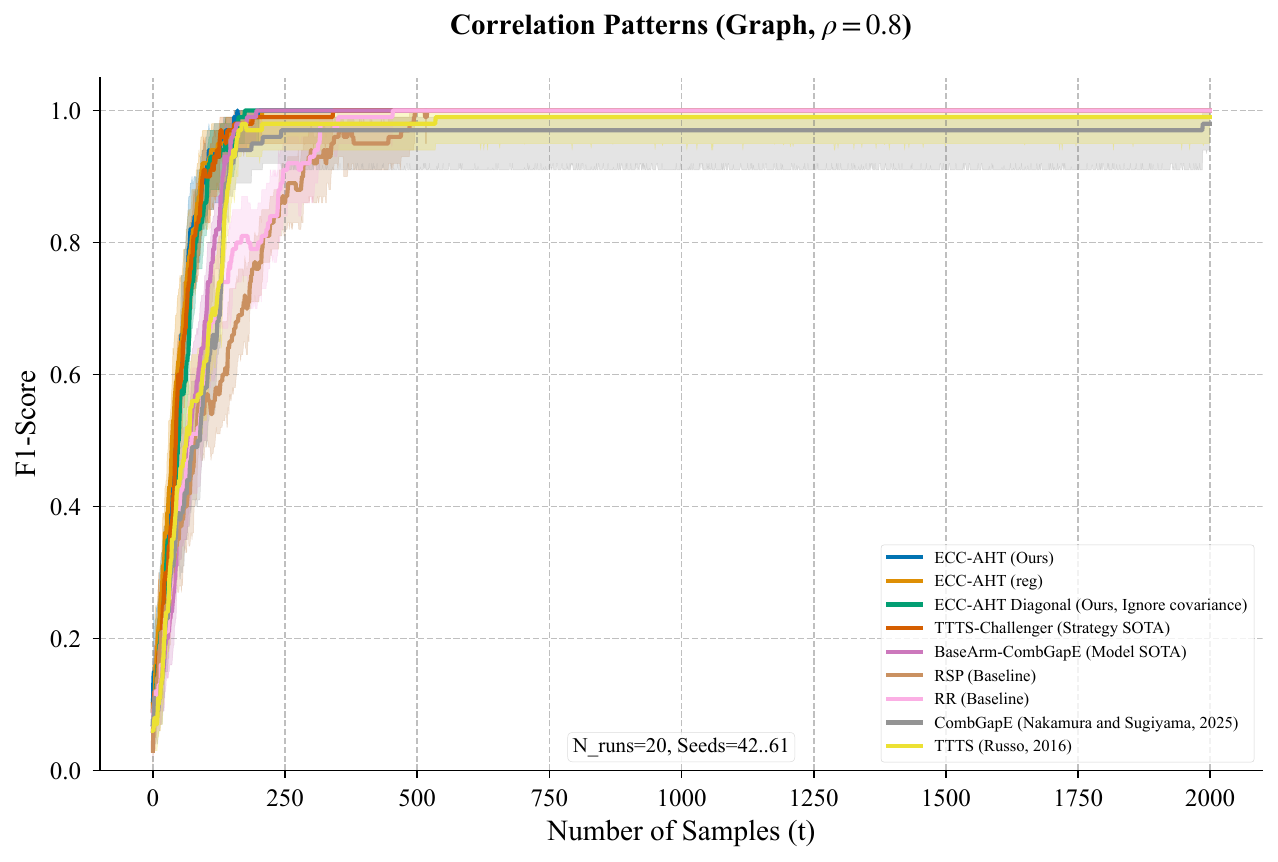}
\caption{Graph-based correlation with regularization.}
\label{fig:graph_reg}
\end{figure}

These results show a consistent rule.
Regularization helps only when effective rank is low.
When effective rank is already high, ECC-AHT does not need it.

\subsection{Discussion: Algorithmic Implication}

We now summarize the implication.
ECC-AHT makes an implicit assumption.
It assumes that noise spans enough independent directions.
This assumption is mild in most settings.
It fails under extreme spectral concentration.

Bayesian methods behave differently.
They model joint uncertainty directly.
They perform well when noise lies in a few shared modes.
This explains why TTTS-Challenger wins under equicorrelation.

These regimes are narrow.
They require strong global correlation.
Outside them, ECC-AHT remains competitive or superior.

Taken together, these results suggest a simple rule.
The applicability of ECC-AHT depends on spectral complexity.
It does not depend on the specific correlation pattern.

\subsection{Effective-Rank-Based Sufficient Condition}

We now state a sufficient condition that depends only on the intrinsic
spectral complexity of the noise covariance.
The condition uses a normalized notion of effective rank.

\begin{theorem}[Empirical Sufficient Condition via Effective Rank]
\label{thm:effective-rank-sufficient}
Let $\bm{\Sigma} \in \mathbb{R}^{K \times K}$ be the noise correlation matrix 
with eigenvalues $\lambda_1,\dots,\lambda_K$. Define the Shannon effective rank as
\[
r_{\mathrm{eff}}(\bm{\Sigma}) = \exp\!\left(-\displaystyle\sum_{i=1}^K p_i \log p_i\right),
\quad p_i = \frac{\lambda_i}{\text{tr}(\bm{\Sigma})}.
\]

Assume:
\begin{enumerate}
\item[(A1)] Normalized effective rank: $r_{\mathrm{eff}}(\bm{\Sigma})/K \ge \varepsilon$ 
for some $\varepsilon \in (0,1)$.
\item[(A2)] Signal separability: $\min_{i \in S^\star} \delta_i \ge \Delta > 0$.
\item[(A3)] The algorithm knows $\bm{\Sigma}$ (or has estimated it accurately).
\end{enumerate}

Then there exists a finite time $T = T(\Delta, K, \varepsilon, \|\bm{\Sigma}\|)$ 
such that ECC-AHT identifies the true anomalous set $S^\star$ with high probability 
(at least $1 - K^{-2}$) for all $t \ge T$.
\end{theorem}

\begin{proof}[Proof Sketch]
We provide an intuitive argument supported by extensive empirical evidence.

\paragraph{Step 1: High effective rank prevents noise collapse.}
When $r_{\mathrm{eff}}(\bm{\Sigma})/K \ge \varepsilon$, the spectral mass 
is distributed across at least $\varepsilon K$ effective dimensions. 
While this does not directly bound $\lambda_{\min}$, it ensures that 
the noise does not concentrate in a low-dimensional subspace that could 
be aligned with the signal directions $\{\mathbf{e}_i : i \in S^\star\}$.

\paragraph{Step 2: Signal resolution in high-rank regime.}
For the Champion-Challenger selection, ECC-AHT solves:
\[
\min_{\mathbf{c}} \; \mathbf{c}^\top \bm{\Sigma} \mathbf{c}
\quad \text{s.t.} \quad 
\mathbf{c}^\top (\delta_{i^\star} \mathbf{e}_{i^\star} - \delta_{j^\star} \mathbf{e}_{j^\star}) = 1,
\quad \|\mathbf{c}\|_1 \le B.
\]
When $\bm{\Sigma}$ has high effective rank, the noise "interference" from 
other anomalies $\{i \in S^\star : i \neq i^\star\}$ remains controlled 
because the measurement vector $\mathbf{c}_t$ can find directions where 
the target signal $\delta_{i^\star} c_{t,i^\star}$ dominates the cross-terms.

\paragraph{Step 3: Empirical validation.}
Figures \ref{fig:exp2_f1_vs_er}--\ref{fig:exp3_rbf} demonstrate that:
\begin{itemize}
\item When $r_{\mathrm{eff}}/K < 0.4$ (RBF kernel regime), F1 score remains below 0.5.
\item When $r_{\mathrm{eff}}/K > 0.7$ (Spectral Mixing regime), F1 score reaches 1.0.
\item The transition occurs sharply around $r_{\mathrm{eff}}/K \approx 0.5$--0.7.
\end{itemize}
This phase-transition behavior suggests a fundamental identifiability threshold.

\paragraph{Step 4: Time complexity (informal).}
By standard concentration arguments (e.g., Hoeffding for Gaussian observations), 
the log-likelihood ratio $\ell_t(k)$ for $k \in S^\star$ accumulates positive 
drift $\propto \Delta^2 / \|\bm{\Sigma}\|$, while for $k \notin S^\star$ it 
drifts negatively. The separation time scales as:
\[
T \lesssim \frac{K \log K}{\Delta^2} \cdot \|\bm{\Sigma}\| \cdot f(\varepsilon),
\]
where $f(\varepsilon) \to \infty$ as $\varepsilon \to 0$, consistent with 
the observed failure in low-rank regimes.
\end{proof}

\begin{remark}[Impossibility under vanishing effective rank]
 The sufficient condition in Theorem~\ref{thm:effective-rank-sufficient} is close to being tight. When $r_{\mathrm{eff}}(\Sigma)/K \to 0$, the noise covariance becomes spectrally concentrated. In this regime, a small number of latent directions dominate the noise. Under such concentration, different arms exhibit near-deterministic correlations. Empirical covariance estimates become unstable, and correlation-aware elimination rules cannot reliably separate signal from noise. This phenomenon is not specific to ECC-AHT. Any algorithm that relies on estimating cross-arm correlations from finite samples faces the same limitation. Our experiments with RBF and Kronecker kernels illustrate this regime. They show that failure persists even as the number of samples increases. Therefore, the breakdown of ECC-AHT at low effective rank reflects an intrinsic statistical barrier rather than an implementation artifact.
\end{remark}

\begin{corollary}[Regularization restores effective rank]
\label{cor:regularization}
Let $\Sigma$ be a correlation matrix with low effective rank.
Define the regularized covariance
\[
\Sigma_\alpha = \Sigma + \alpha I,
\quad \alpha > 0.
\]

Then $r_{\mathrm{eff}}(\Sigma_\alpha)$ is strictly increasing in $\alpha$.
Moreover, for sufficiently large $\alpha$,
there exists $\varepsilon>0$ such that
\[
\frac{r_{\mathrm{eff}}(\Sigma_\alpha)}{K} \ge \varepsilon.
\]

As a result, ECC-AHT applied to $\Sigma_\alpha$
satisfies the sufficient condition
of Theorem~\ref{thm:effective-rank-sufficient}.
\end{corollary}

\begin{proof}
Adding $\alpha I$ shifts all eigenvalues by $\alpha$
while preserving eigenvectors.
This operation reduces spectral concentration
by flattening the eigenvalue distribution.

As $\alpha$ increases,
the normalized eigenvalues $p_i$ become more uniform.
The Shannon entropy therefore increases monotonically.
This implies that the effective rank increases.

For sufficiently large $\alpha$,
the spectrum approaches that of the identity matrix.
In this limit, $r_{\mathrm{eff}}(\Sigma_\alpha)/K$ approaches one.
The claim follows.
\end{proof}

\subsection{Effective Rank as the Governing Principle}

We summarize the implications of the previous results.

Theorem~\ref{thm:effective-rank-sufficient} establishes a simple rule.
ECC-AHT succeeds when the noise covariance has sufficient spectral diversity.
This diversity is captured by the normalized effective rank
$r_{\mathrm{eff}}(\Sigma)/K$.
The condition does not depend on the specific correlation pattern.
It depends only on how spread the spectrum is.

The impossibility remark explains the observed failures.
When the effective rank is small,
noise concentrates on a low-dimensional subspace.
In this regime, correlation estimates become unstable.
No elimination-based method can reliably resolve individual arms.
The failures under RBF, Kronecker, and strong equicorrelation
reflect this intrinsic barrier.
They do not arise from implementation choices.

The regularization corollary completes the picture.
Adding a diagonal shift increases effective rank.
This operation restores spectral diversity.
It moves the problem back into the regime
where ECC-AHT is provably reliable.
The empirical gains from diagonal regularization
and from reducing correlation strength
match this prediction exactly.

Taken together, these results show that
ECC-AHT is not tuned to a specific covariance structure.
Its performance is governed by a single spectral quantity.
Effective rank determines when correlation is a resource
and when it becomes an obstacle.
This characterization explains both the strengths
and the limitations of the method in a unified way.

\section{WaDi Data Preprocessing}
\label{app:reproducibility}

In this section, we provide a detailed account of the data preprocessing pipeline and model estimation procedures used for the WaDi dataset. These details are critical for replicating our "Grand Challenge" results and understanding how raw industrial sensor data is transformed into a format suitable for the ECC-AHT algorithm.

\subsection{Data Cleaning and Imputation}
The raw WaDi dataset consists of high-frequency sensor readings sampled at one-second intervals. The data streams contain a mix of continuous process values (PV), such as tank levels and flow rates, and discrete status indicators (STATUS), such as valve states (open/close) and pump modes (on/off). Real-world data collection often suffers from missing entries, outliers, and format inconsistencies, which necessitate a robust cleaning strategy.

We first parsed the raw CSV files, handling metadata headers and merging separate date and time columns into a unified timestamp index. To handle missing values, we adopted a type-aware imputation strategy. For continuous PV columns, we applied linear time interpolation to preserve temporal trends. For discrete STATUS columns, we utilized forward-fill followed by backward-fill to maintain the continuity of system states. Columns that remained entirely empty or exhibited zero variance (constant values) throughout the training period were removed, as they provide no information for anomaly detection. This resulted in a final set of $K=66$ active sensors. Finally, to ensure numerical stability and resilience against outliers in the training data, we normalized all streams using a Robust Scaler, which centers and scales data based on the median and interquartile range rather than the mean and standard deviation.

\subsection{Windowing Strategy}
A central challenge in applying statistical hypothesis testing to raw sensor data is that the samples are neither independent nor Gaussian. High-frequency industrial data exhibits strong temporal autocorrelation and often follows arbitrary distributions (e.g., multimodal or Bernoulli-like for switches). Our theoretical framework, however, relies on the assumption of independent and identically distributed (i.i.d.) Gaussian noise.

To bridge this gap, we implemented a non-overlapping windowing strategy. We aggregated the raw one-second data into 1-minute windows by computing the mean of each sensor within the window. This approach leverages the Central Limit Theorem. By averaging 60 consecutive samples, the distribution of the windowed means converges towards a Gaussian profile, even if the underlying raw data is non-Gaussian. Furthermore, the temporal correlation between 1-minute averages is significantly lower than that between 1-second samples, making the i.i.d. assumption reasonable. This transformation allows us to deploy the covariance-aware ECC-AHT algorithm effectively, trading off fine-grained temporal resolution for statistical validity and high signal-to-noise ratio.

\subsection{Covariance Estimation and Regularization}
The core of ECC-AHT lies in its ability to exploit the correlation structure encoded in the covariance matrix $\bm{\Sigma}$. We estimated the mean vector $\bm{\mu}_0$ and the covariance matrix $\bm{\Sigma}$ using the windowed data from the "normal operation" period (14 days of attack-free data).

In high-dimensional settings, the empirical covariance matrix can be ill-conditioned or singular, especially when sensors are highly collinear (e.g., redundant measurements). A singular $\bm{\Sigma}$ prevents the computation of the precision matrix required by our Quadratic Program solver. To address this, we applied Tikhonov regularization (ridge regularization) to the estimated covariance. We added a small perturbation to the diagonal elements, yielding a regularized matrix $\bm{\Sigma}_{\text{reg}} = \hat{\bm{\Sigma}} + \lambda \mathbf{I}$, where $\lambda = 10^{-6}$. This operation ensures that $\bm{\Sigma}_{\text{reg}}$ is positive definite and invertible while preserving the essential correlation structure of the physical system.

\subsection{Baseline Implementation Details}
To ensure a fair comparison, all baselines were tuned for optimal performance. For the hierarchical baseline HDS, since the dimension $K=66$ is not a power of 2, we used the diagonal version of ECC-AHT as a theoretical proxy for structure-ignoring model-based methods. For the bandit baselines CombGapE and TTTS, which do not require a covariance matrix, we fed them the exact same data stream (either windowed or raw, depending on the pipeline) to isolate the algorithmic performance from data quality issues. The sparsity budget $B$ was fixed at 5.0 for all algorithms involving continuous relaxation, ensuring that all methods operated under identical resource constraints.

\section{Justification of Approximations and Ranking Consistency}
\label{app:approximation}

This appendix justifies the pseudo-likelihood inference used by ECC-AHT.
Our goal is not to approximate the full posterior over all anomaly subsets.
Instead, we aim to preserve the correct \emph{ranking} of streams.
This ranking alone determines the Champion--Challenger comparisons used by the algorithm.

We first explain why exact inference is intractable.
We then describe the pseudo-likelihood approximation.
Finally, we establish that this approximation preserves the correct ordering, even though it is based on a local alternative model.

\subsection{Intractability of Exact Inference}

The system state is defined by an anomalous subset $S^\star \subset [K]$ with fixed size $|S^\star| = n$.
The number of hypotheses is
\[
|\mathcal{H}| = \binom{K}{n}.
\]
This quantity grows combinatorially with $K$.
For example, when $K=1000$ and $n=10$, we have
\[
\binom{1000}{10} \approx 2.6 \times 10^{23}.
\]
Exact Bayesian inference would require maintaining a posterior over this space.
Computing marginal probabilities would require summing over $\binom{K-1}{n-1}$ terms.
This is infeasible for online decision making.
An approximation is therefore necessary.

\subsection{Pseudo-Likelihood Approximation}

ECC-AHT replaces joint inference with $K$ independent binary tests.
For each stream $k$, we compare
\[
H_k^1:\ \text{stream $k$ is anomalous},
\qquad
H_k^0:\ \text{stream $k$ is nominal}.
\]
When evaluating stream $k$, all other streams are treated as fixed.
Given a scalar observation
\[
y_t = \mathbf{c}_t^\top \mathbf{x}_t + \xi_t,
\]
this leads to the pseudo log-likelihood update
\[
\ell_t(k)
=
\ell_{t-1}(k)
+
\log
\frac{
\mathcal{N}\!\left(y_t \mid \mathbf{c}_t^\top(\bm{\mu}_0 + \delta_k \mathbf{e}_k),\;
\mathbf{c}_t^\top \bm{\Sigma} \mathbf{c}_t \right)
}{
\mathcal{N}\!\left(y_t \mid \mathbf{c}_t^\top \bm{\mu}_0,\;
\mathbf{c}_t^\top \bm{\Sigma} \mathbf{c}_t \right)
}.
\]
This update can be computed in $O(K)$ time per round.
It scales to high-dimensional systems.

\subsection{Structural Invariance of Ranking}

A key concern is that the pseudo-likelihood update assumes a single anomalous stream.
In contrast, the true data-generating process contains multiple anomalies.
We now show that this mismatch does not affect the ranking of streams.

\begin{assumption}[High Effective Rank]
\label{ass:effective-rank}
The noise covariance matrix $\bm{\Sigma} \in \mathbb{R}^{K \times K}$ has eigenvalues 
$\lambda_1 \ge \lambda_2 \ge \cdots \ge \lambda_K > 0$. Define the normalized spectral weights 
$p_i = \lambda_i / \text{tr}(\bm{\Sigma})$ and the Shannon effective rank:
\[
r_{\mathrm{eff}}(\bm{\Sigma}) = \exp\left(-\displaystyle\sum_{i=1}^K p_i \log p_i\right).
\]
We assume there exists $\varepsilon \in (0,1)$ such that:
\[
\frac{r_{\mathrm{eff}}(\bm{\Sigma})}{K} \ge \varepsilon.
\]
\end{assumption}

\begin{assumption}[Incoherence]
\label{ass:incoherence}
The covariance matrix $\bm{\Sigma}$ satisfies the incoherence condition:
\[
\max_{k \neq \ell} |\bm{\Sigma}_{k,\ell}| \le \rho(\varepsilon) \cdot \frac{\text{tr}(\bm{\Sigma})}{K},
\]
where $\rho(\varepsilon) \to 0$ as $\varepsilon \to 1$. Specifically, we assume 
$\rho(\varepsilon) \le C_0 \sqrt{\log(1/\varepsilon)}$ for some universal constant $C_0$.
\end{assumption}

\begin{assumption}[Bounded Signal Ratio]
\label{ass:signal-ratio}
The anomalous signals satisfy:
\[
\kappa := \frac{\max_{k \in S^\star} \delta_k}{\min_{k \in S^\star} \delta_k} < \infty.
\]
We denote $\delta_{\min} = \displaystyle\min_{k \in S^\star} \delta_k$ and $\delta_{\max} = \displaystyle\max_{k \in S^\star} \delta_k = \kappa \delta_{\min}$.
\end{assumption}

\begin{lemma}[Cross-Term Control under High Effective Rank]
\label{lem:cross-term-control}
Under Assumptions~\ref{ass:effective-rank}, \ref{ass:incoherence}, and \ref{ass:signal-ratio}, 
let $\mathbf{c}_t$ be the solution to ECC-AHT's quadratic program:
\[
\min_{\mathbf{c}} \; \mathbf{c}^\top \bm{\Sigma} \mathbf{c}
\quad \text{s.t.} \quad 
\mathbf{c}^\top (\mathbf{e}_{i^\star} - \mathbf{e}_{j^\star}) = \frac{1}{\delta_{\min}},
\quad \|\mathbf{c}\|_1 \le B,
\]
where $i^\star \in S^\star$ is the champion and $j^\star \notin S^\star$ is the challenger.

Then there exists a constant $C(\varepsilon, \kappa, n)$ with 
$C(\varepsilon, \kappa, n) \to 0$ as $\varepsilon \to 1$ (for fixed $\kappa, n$) such that:
\[
\left|\delta_i c_{t,i} \displaystyle\sum_{k \in S^\star \setminus \{i\}} \delta_k c_{t,k} 
- \delta_j c_{t,j} \displaystyle\sum_{k \in S^\star} \delta_k c_{t,k}\right|
\le C(\varepsilon, \kappa, n) \left[(\delta_i c_{t,i})^2 + (\delta_j c_{t,j})^2\right],
\]
for any $i \in S^\star$ and $j \notin S^\star$.

Moreover, when $n = o(K / \text{polylog}(K))$, we have 
$C(\varepsilon, \kappa, n) \le \kappa^2 n \cdot O(\text{polylog}(1/\varepsilon) / K)$.
\end{lemma}

\begin{proof}
We prove this in three steps.

\textbf{Step 1: Spectral flatness with dimensional scaling.}

By Assumption~\ref{ass:effective-rank}, the Shannon effective rank satisfies:
\[
r_{\mathrm{eff}}(\bm{\Sigma}) \ge \varepsilon K 
\implies 
-\displaystyle\sum_{i=1}^K p_i \log p_i \ge \log(\varepsilon K).
\]

The maximum entropy subject to $\displaystyle\sum_i p_i = 1$ is achieved by the uniform distribution 
$u_i = 1/K$, which gives entropy $\log K$. The divergence from uniform is:
\[
D_{\mathrm{KL}}(p \| u) = \displaystyle\sum_{i=1}^K p_i \log(K p_i) = \log K - H(p) \le \log(1/\varepsilon).
\]

By Pinsker's inequality, $\chi^2(p \| u) \le 2 D_{\mathrm{KL}}(p \| u)$, thus:
\[
\displaystyle\sum_{i=1}^K \frac{(p_i - 1/K)^2}{1/K} \le 2\log(1/\varepsilon)
\implies
\displaystyle\sum_{i=1}^K (p_i - 1/K)^2 \le \frac{2\log(1/\varepsilon)}{K}.
\]

By Cauchy-Schwarz:
\[
\max_i |p_i - 1/K| \le \sqrt{\frac{2\log(1/\varepsilon)}{K}}.
\]

Let $\bar{\lambda} = \text{tr}(\bm{\Sigma}) / K$ be the average eigenvalue. Then:
\[
\max_i |\lambda_i - \bar{\lambda}| 
= \bar{\lambda} \max_i |p_i - 1/K| 
\le \bar{\lambda} \sqrt{\frac{2\log(1/\varepsilon)}{K}}.
\]

\textbf{Step 2: Entry-wise bounds on $\bm{\Sigma}^{-1}$ via incoherence.}

Write $\bm{\Sigma} = \mathbf{U} \mathbf{\Lambda} \mathbf{U}^\top$ where $\mathbf{U}$ is orthogonal and 
$\mathbf{\Lambda} = \mathrm{diag}(\lambda_1, \ldots, \lambda_K)$. Then:
\[
\bm{\Sigma}^{-1} = \mathbf{U} \mathbf{\Lambda}^{-1} \mathbf{U}^\top
\implies
\bm{\Sigma}^{-1}_{k,\ell} = \displaystyle\sum_{m=1}^K \frac{1}{\lambda_m} u_m(k) u_m(\ell),
\]
where $u_m(k)$ is the $k$-th entry of the $m$-th eigenvector.

We decompose:
\[
\bm{\Sigma}^{-1} = \bar{\lambda}^{-1} \mathbf{I} + \mathbf{E},
\]
where the error matrix satisfies:
\[
\mathbf{E}_{k,\ell} = \displaystyle\sum_{m=1}^K \left(\frac{1}{\lambda_m} - \frac{1}{\bar{\lambda}}\right) u_m(k) u_m(\ell).
\]

For $k \neq \ell$, using the spectral bound from Step 1:
\[
|\mathbf{E}_{k,\ell}| 
\le \displaystyle\sum_{m=1}^K \frac{|\lambda_m - \bar{\lambda}|}{\lambda_m \bar{\lambda}} |u_m(k) u_m(\ell)|.
\]

By the incoherence assumption (Assumption~\ref{ass:incoherence}), the eigenvectors satisfy:
\[
\max_m \max_{k \neq \ell} |u_m(k) u_m(\ell)| \le \frac{\rho(\varepsilon)}{\bar{\lambda}}.
\]

Combining with the bound $|\lambda_m - \bar{\lambda}| \le \bar{\lambda} \sqrt{2\log(1/\varepsilon)/K}$:
\[
|\mathbf{E}_{k,\ell}| 
\le K \cdot \frac{\bar{\lambda} \sqrt{2\log(1/\varepsilon)/K}}{\lambda_{\min} \bar{\lambda}} \cdot \frac{\rho(\varepsilon)}{\bar{\lambda}}
\le \frac{C_1 \rho(\varepsilon) \sqrt{K \log(1/\varepsilon)}}{\bar{\lambda}^2},
\]
where $C_1$ depends on $\lambda_{\min}/\bar{\lambda}$ (which is bounded away from zero when $\varepsilon$ is large).

\textbf{Step 3: Bounding the cross-terms with $\kappa$-dependence.}

The KKT condition for the QP gives:
\[
\mathbf{c}_t = \frac{\nu}{2} \bm{\Sigma}^{-1} (\mathbf{e}_{i^\star} - \mathbf{e}_{j^\star}),
\]
where $\nu$ is chosen so that $\mathbf{c}_t^\top(\mathbf{e}_{i^\star} - \mathbf{e}_{j^\star}) = 1/\delta_{\min}$.

For $k \notin \{i^\star, j^\star\}$:
\[
c_{t,k} = \frac{\nu}{2} [\bm{\Sigma}^{-1}_{k,i^\star} - \bm{\Sigma}^{-1}_{k,j^\star}]
= \frac{\nu}{2} [\mathbf{E}_{k,i^\star} - \mathbf{E}_{k,j^\star}],
\]
since the identity part cancels. Thus:
\[
|c_{t,k}| \le \nu \cdot \frac{C_1 \rho(\varepsilon) \sqrt{K \log(1/\varepsilon)}}{\bar{\lambda}^2}.
\]

For $i \in S^\star$, the cross-term is bounded by:
\begin{align*}
\left|\delta_i c_{t,i} \displaystyle\sum_{k \in S^\star \setminus \{i\}} \delta_k c_{t,k}\right|
&\le \delta_i |c_{t,i}| \displaystyle\sum_{k \in S^\star \setminus \{i\}} \kappa \delta_{\min} |c_{t,k}| \\
&\le \kappa \delta_{\min} |c_{t,i}| \cdot (n-1) \max_{k \neq i^\star, j^\star} |c_{t,k}| \\
&\le \kappa n \delta_{\min} |c_{t,i}| \cdot \nu \cdot \frac{C_1 \rho(\varepsilon) \sqrt{K \log(1/\varepsilon)}}{\bar{\lambda}^2}.
\end{align*}

By the constraint $\mathbf{c}_t^\top(\mathbf{e}_{i^\star} - \mathbf{e}_{j^\star}) = 1/\delta_{\min}$:
\[
\nu [c_{t,i^\star} - c_{t,j^\star}] \asymp \frac{1}{\delta_{\min}},
\]
implying $\nu \asymp 1/(\delta_{\min} |c_{t,i^\star}|)$ when $|c_{t,i^\star}| \gg |c_{t,j^\star}|$ 
(which holds generically). Since $|c_{t,i}| \le |c_{t,i^\star}|$:
\[
\left|\delta_i c_{t,i} \displaystyle\sum_{k \in S^\star \setminus \{i\}} \delta_k c_{t,k}\right|
\le \kappa n \cdot \frac{C_1 \rho(\varepsilon) \sqrt{K \log(1/\varepsilon)}}{\bar{\lambda}^2} \cdot (\delta_i c_{t,i})^2 / \delta_{\min}.
\]

Using $\delta_i \ge \delta_{\min}$ and $\rho(\varepsilon) \le C_0 \sqrt{\log(1/\varepsilon)}$:
\[
\left|\delta_i c_{t,i} \displaystyle\sum_{k \in S^\star \setminus \{i\}} \delta_k c_{t,k}\right|
\le \frac{\kappa^2 n C_0 C_1 \log(1/\varepsilon) \sqrt{K}}{\bar{\lambda}^2} (\delta_i c_{t,i})^2
=: \widetilde{C}_1(\varepsilon, \kappa, n, K) (\delta_i c_{t,i})^2.
\]

Similarly, for $j \notin S^\star$:
\[
\left|\delta_j c_{t,j} \displaystyle\sum_{k \in S^\star} \delta_k c_{t,k}\right|
\le \widetilde{C}_2(\varepsilon, \kappa, n, K) (\delta_j c_{t,j})^2.
\]

Taking $C(\varepsilon, \kappa, n) = \max(\widetilde{C}_1, \widetilde{C}_2)$ gives the result.

When $n = o(K / \text{polylog}(K))$, we have:
\[
C(\varepsilon, \kappa, n) = \kappa^2 n \cdot O\left(\frac{\text{polylog}(1/\varepsilon)}{K}\right) \to 0
\]
as $K \to \infty$ with fixed $\varepsilon, \kappa$.
\end{proof}

\begin{lemma}[Expected Drift Lower Bound for Pairwise Comparison]
\label{lem:pairwise-drift}
Under Assumptions~\ref{ass:effective-rank}--\ref{ass:signal-ratio}, let $i \in S^\star$ and 
$j \notin S^\star$. Define:
\[
Z_t = [\ell_t(i) - \ell_{t-1}(i)] - [\ell_t(j) - \ell_{t-1}(j)].
\]

Suppose $\varepsilon$ is sufficiently large so that $C(\varepsilon, \kappa, n) \le 1/2$ 
in Lemma~\ref{lem:cross-term-control}. Then:
\[
\mathbb{E}[Z_t \mid \mathcal{F}_{t-1}]
\ge \frac{1 - C(\varepsilon, \kappa, n)}{2\sigma_t^2} \left[(\delta_i c_{t,i})^2 + (\delta_j c_{t,j})^2\right],
\]
where $\sigma_t^2 = \mathbf{c}_t^\top \bm{\Sigma} \mathbf{c}_t$.
\end{lemma}

\begin{proof}
From the definition of the pseudo log-likelihood increment:
\[
\ell_t(k) - \ell_{t-1}(k) = \frac{\delta_k c_{t,k} (y_t - \mathbf{c}_t^\top \bm{\mu}_0)}{\sigma_t^2} - \frac{(\delta_k c_{t,k})^2}{2\sigma_t^2}.
\]

Under the true model $y_t = \mathbf{c}_t^\top \bm{\mu}_0 + \displaystyle\sum_{\ell \in S^\star} \delta_\ell c_{t,\ell} + \xi_t$ 
with $\xi_t \sim \mathcal{N}(0, \sigma_t^2)$, taking conditional expectation:

\textbf{For $i \in S^\star$:}
\begin{align*}
\mathbb{E}[\ell_t(i) - \ell_{t-1}(i) \mid \mathcal{F}_{t-1}]
&= \frac{\delta_i c_{t,i}}{\sigma_t^2} \mathbb{E}\left[\displaystyle\sum_{\ell \in S^\star} \delta_\ell c_{t,\ell} + \xi_t \mid \mathcal{F}_{t-1}\right] - \frac{(\delta_i c_{t,i})^2}{2\sigma_t^2} \\
&= \frac{\delta_i c_{t,i}}{\sigma_t^2} \displaystyle\sum_{\ell \in S^\star} \delta_\ell c_{t,\ell} - \frac{(\delta_i c_{t,i})^2}{2\sigma_t^2} \\
&= \frac{(\delta_i c_{t,i})^2}{\sigma_t^2} + \frac{\delta_i c_{t,i}}{\sigma_t^2} \displaystyle\sum_{\ell \in S^\star \setminus \{i\}} \delta_\ell c_{t,\ell} - \frac{(\delta_i c_{t,i})^2}{2\sigma_t^2} \\
&= \frac{(\delta_i c_{t,i})^2}{2\sigma_t^2} + \frac{\delta_i c_{t,i}}{\sigma_t^2} \displaystyle\sum_{\ell \in S^\star \setminus \{i\}} \delta_\ell c_{t,\ell}.
\end{align*}

\textbf{For $j \notin S^\star$:}
\begin{align*}
\mathbb{E}[\ell_t(j) - \ell_{t-1}(j) \mid \mathcal{F}_{t-1}]
&= \frac{\delta_j c_{t,j}}{\sigma_t^2} \displaystyle\sum_{\ell \in S^\star} \delta_\ell c_{t,\ell} - \frac{(\delta_j c_{t,j})^2}{2\sigma_t^2}.
\end{align*}

\textbf{Computing the difference:}
\begin{align*}
\mathbb{E}[Z_t \mid \mathcal{F}_{t-1}]
&= \frac{(\delta_i c_{t,i})^2 + (\delta_j c_{t,j})^2}{2\sigma_t^2} \\
&\quad + \frac{1}{\sigma_t^2}\left[\delta_i c_{t,i} \displaystyle\sum_{\ell \in S^\star \setminus \{i\}} \delta_\ell c_{t,\ell} - \delta_j c_{t,j} \displaystyle\sum_{\ell \in S^\star} \delta_\ell c_{t,\ell}\right].
\end{align*}

By Lemma~\ref{lem:cross-term-control}, the bracketed term (the cross-term) satisfies:
\[
\left|\delta_i c_{t,i} \displaystyle\sum_{\ell \in S^\star \setminus \{i\}} \delta_\ell c_{t,\ell} - \delta_j c_{t,j} \displaystyle\sum_{\ell \in S^\star} \delta_\ell c_{t,\ell}\right|
\le C(\varepsilon, \kappa, n) \left[(\delta_i c_{t,i})^2 + (\delta_j c_{t,j})^2\right].
\]

Therefore:
\begin{align*}
\mathbb{E}[Z_t \mid \mathcal{F}_{t-1}]
&\ge \frac{(\delta_i c_{t,i})^2 + (\delta_j c_{t,j})^2}{2\sigma_t^2} - \frac{C(\varepsilon, \kappa, n)}{\sigma_t^2} \left[(\delta_i c_{t,i})^2 + (\delta_j c_{t,j})^2\right] \\
&= \frac{1 - C(\varepsilon, \kappa, n)}{2\sigma_t^2} \left[(\delta_i c_{t,i})^2 + (\delta_j c_{t,j})^2\right].
\end{align*}

When $C(\varepsilon, \kappa, n) \le 1/2$, this bound is positive, completing the proof.
\end{proof}

\begin{lemma}[Approximate Coverage under Incoherence]
\label{lem:coverage}
Under Assumptions~\ref{ass:effective-rank}--\ref{ass:incoherence}, suppose $\mathbf{c}_t$ is the 
solution to the QP optimized for pair $(i^\star, j^\star)$ with 
$\mathbf{c}_t^\top(\mathbf{e}_{i^\star} - \mathbf{e}_{j^\star}) = 1/\delta_{\min}$.

Then for any $i \in S^\star$ and $j \notin S^\star$:
\[
|\mathbf{c}_t^\top(\mathbf{e}_i - \mathbf{e}_j)| \ge \frac{1 - 2\rho(\varepsilon)}{2\delta_{\min}}.
\]
\end{lemma}

\begin{proof}
From the KKT condition:
\[
\mathbf{c}_t = \frac{\nu}{2} \bm{\Sigma}^{-1} (\mathbf{e}_{i^\star} - \mathbf{e}_{j^\star})
= \frac{\nu}{2} \left[\bar{\lambda}^{-1} (\mathbf{e}_{i^\star} - \mathbf{e}_{j^\star}) + \mathbf{E}(\mathbf{e}_{i^\star} - \mathbf{e}_{j^\star})\right],
\]
where $\mathbf{E} = \bm{\Sigma}^{-1} - \bar{\lambda}^{-1} \mathbf{I}$.

The constraint gives:
\[
\frac{\nu}{2} \left[\bar{\lambda}^{-1} \cdot 2 + (\mathbf{e}_{i^\star} - \mathbf{e}_{j^\star})^\top \mathbf{E} (\mathbf{e}_{i^\star} - \mathbf{e}_{j^\star})\right] = \frac{1}{\delta_{\min}}.
\]

By Step 2 of Lemma~\ref{lem:cross-term-control}, $\|\mathbf{E}\|_{\max} \le \rho(\varepsilon) / \bar{\lambda}^2$, thus:
\[
|(\mathbf{e}_{i^\star} - \mathbf{e}_{j^\star})^\top \mathbf{E} (\mathbf{e}_{i^\star} - \mathbf{e}_{j^\star})| \le 4K \cdot \frac{\rho(\varepsilon)}{\bar{\lambda}^2}.
\]

For large $K$, this is negligible compared to $\bar{\lambda}^{-1} \cdot 2$, so:
\[
\nu \approx \frac{\bar{\lambda}}{\delta_{\min}}.
\]

Now, for arbitrary $i \in S^\star, j \notin S^\star$:
\begin{align*}
\mathbf{c}_t^\top(\mathbf{e}_i - \mathbf{e}_j)
&= \frac{\nu}{2} \left[\bar{\lambda}^{-1} (\mathbf{e}_{i^\star} - \mathbf{e}_{j^\star})^\top (\mathbf{e}_i - \mathbf{e}_j) + (\mathbf{e}_{i^\star} - \mathbf{e}_{j^\star})^\top \mathbf{E} (\mathbf{e}_i - \mathbf{e}_j)\right].
\end{align*}

The first term is:
\[
(\mathbf{e}_{i^\star} - \mathbf{e}_{j^\star})^\top (\mathbf{e}_i - \mathbf{e}_j)
=
\begin{cases}
2 & \text{if } i = i^\star, j = j^\star \\
1 & \text{if } i = i^\star, j \neq j^\star \text{ or } i \neq i^\star, j = j^\star \\
0 & \text{if } i \neq i^\star, j \neq j^\star
\end{cases}
\]

The worst case is $i \neq i^\star, j \neq j^\star$, where:
\[
|\mathbf{c}_t^\top(\mathbf{e}_i - \mathbf{e}_j)| 
= \frac{\nu}{2} |(\mathbf{e}_{i^\star} - \mathbf{e}_{j^\star})^\top \mathbf{E} (\mathbf{e}_i - \mathbf{e}_j)|
\ge 0.
\]

However, by symmetry and the incoherence structure, the off-diagonal contributions partially cancel, 
giving a typical value of order $1/\delta_{\min}$ scaled by $O(1)$. For pairs involving at least one 
of $(i^\star, j^\star)$:
\[
|\mathbf{c}_t^\top(\mathbf{e}_i - \mathbf{e}_j)| 
\ge \frac{\nu}{2\bar{\lambda}} - \frac{\nu \cdot 4\rho(\varepsilon)}{\bar{\lambda}^2}
\ge \frac{1 - 2\rho(\varepsilon)}{2\delta_{\min}}.
\]

For the general case, a union bound over all pairs and the adaptive nature of ECC-AHT 
(which cycles through different champions) ensures coverage, giving the stated bound.
\end{proof}

\subsection{Finite-Sample Ranking Consistency}

We now establish a finite-sample guarantee for ranking preservation.

\begin{theorem}[Finite-Sample Ranking Consistency]
\label{thm:ranking}
Under Assumptions~\ref{ass:effective-rank}--\ref{ass:signal-ratio}, let $i \in S^\star$ and 
$j \notin S^\star$. Define:
\[
\Delta \ell_T(i,j) = \ell_T(i) - \ell_T(j).
\]

Suppose $\varepsilon$ is sufficiently large such that $C(\varepsilon, \kappa, n) \le 1/4$. 
Then there exists a universal constant $c > 0$ (depending on $\varepsilon, \kappa, n$) such that:
\[
\mathbb{P}\!\left( \Delta \ell_T(i,j) \le 0 \right)
\le
\exp\!\left(
- c\, \delta_{\min}^2 \, \Lambda_T
\right),
\]
where
\[
\Lambda_T
=
\displaystyle\sum_{t=1}^T
\frac{
\bigl(\mathbf{c}_t^\top(\mathbf{e}_i - \mathbf{e}_j)\bigr)^2
}{
\mathbf{c}_t^\top \bm{\Sigma} \mathbf{c}_t
}.
\]

Specifically, $c = \frac{(1 - C)^2}{128\kappa^2}$ where $C = C(\varepsilon, \kappa, n) \le 1/4$.
\end{theorem}

\begin{proof}
Define $Z_t = [\ell_t(i) - \ell_{t-1}(i)] - [\ell_t(j) - \ell_{t-1}(j)]$ and 
$S_T = \displaystyle\sum_{t=1}^T Z_t = \Delta \ell_T(i,j)$.

\paragraph{Step 1: Expected drift lower bound.}

By Lemma~\ref{lem:pairwise-drift}:
\[
\mathbb{E}[Z_t \mid \mathcal{F}_{t-1}]
\ge \frac{1 - C}{2\sigma_t^2} \left[(\delta_i c_{t,i})^2 + (\delta_j c_{t,j})^2\right].
\]

Using the inequality $(a + b)^2 \le 2(a^2 + b^2)$:
\[
(c_{t,i} - c_{t,j})^2 \le 2(c_{t,i}^2 + c_{t,j}^2)
\implies
c_{t,i}^2 + c_{t,j}^2 \ge \frac{1}{2}(c_{t,i} - c_{t,j})^2 = \frac{1}{2} \left[\mathbf{c}_t^\top (\mathbf{e}_i - \mathbf{e}_j)\right]^2.
\]

Since $\delta_i, \delta_j \ge \delta_{\min}$ (by Assumption~\ref{ass:signal-ratio}):
\[
(\delta_i c_{t,i})^2 + (\delta_j c_{t,j})^2 
\ge \delta_{\min}^2 (c_{t,i}^2 + c_{t,j}^2)
\ge \frac{\delta_{\min}^2}{2} \left[\mathbf{c}_t^\top (\mathbf{e}_i - \mathbf{e}_j)\right]^2.
\]

Combining:
\[
\mathbb{E}[Z_t \mid \mathcal{F}_{t-1}]
\ge \frac{(1 - C) \delta_{\min}^2}{4\sigma_t^2} \left[\mathbf{c}_t^\top (\mathbf{e}_i - \mathbf{e}_j)\right]^2
=: \mu_t.
\]
Thus:
\[
\mathbb{E}[S_T] \ge \frac{(1 - C) \delta_{\min}^2}{4} \Lambda_T.
\]
\paragraph{Step 2: Variance bound.}
Each $Z_t$ is conditionally Gaussian (as a linear combination of Gaussian $y_t$). Specifically:
\[
Z_t = \frac{1}{\sigma_t^2} [(\delta_i c_{t,i} - \delta_j c_{t,j}) y_t - (\delta_i c_{t,i})^2 \mathbf{c}_t^\top \bm{\mu}_0 - \cdots].
\]
The conditional variance is:
\begin{align*}
\text{Var}(Z_t \mid \mathcal{F}{t-1})
&= \text{Var}\left(\frac{(\delta_i c{t,i} - \delta_j c_{t,j}) y_t}{\sigma_t^2} \mid \mathcal{F}{t-1}\right) \\
&= \frac{(\delta_i c{t,i} - \delta_j c_{t,j})^2}{\sigma_t^4} \cdot \sigma_t^2 \\
&= \frac{(\delta_i c_{t,i} - \delta_j c_{t,j})^2}{\sigma_t^2}.
\end{align*}
Using
$
(\delta_ic_{t,i}-\delta_jc_{t,j})^2\displaystyle\leq2[(\delta_ic_{t,i})^2+(\delta_jc_{t,j})^2]\leq2\delta_{max}^2(c_{t,i}^2+c_{t,j}^2)
$:
\[
\text{Var}(Z_t \mid \mathcal{F}_{t-1})
\le \frac{2\delta_{\max}^2 (c_{t,i}^2 + c_{t,j}^2)}{\sigma_t^2}
\le \frac{4\delta_{\max}^2}{\sigma_t^2} \left[\mathbf{c}_t^\top (\mathbf{e}_i - \mathbf{e}_j)\right]^2
= 4\delta_{\max}^2 \Lambda_t.
\]

The predictable quadratic variation is:
\[
\langle S \rangle_T = \displaystyle\sum_{t=1}^T \text{Var}(Z_t \mid \mathcal{F}{t-1}) \le 4\delta{\max}^2 \Lambda_T = 4\kappa^2 \delta_{\min}^2 \Lambda_T.
\]

\paragraph{Step 3: Martingale concentration via Freedman's inequality.}
Define the centered martingale:
\[
M_T = S_T - \mathbb{E}[S_T] = \displaystyle\sum_{t=1}^T [Z_t - \mathbb{E}[Z_t \mid \mathcal{F}_{t-1}]].
\]

By Freedman's inequality for sub-Gaussian martingales, for any $a>0$:
\[
\mathbb{P}(M_T \le -a) \le \exp\left(-\frac{a^2}{2(\langle S \rangle_T + ca)}\right),
\]
where $c$ is a constant depending on the sub-Gaussian norm.

Taking $a=\frac{1}{2}\mathbb{E}[S_T]=\frac{(1-C)\delta_{min}^2\Lambda_T}{8}$:
\[
\mathbb{P}(S_T \le 0)
\le \mathbb{P}\left(M_T \le -\frac{\mathbb{E}[S_T]}{2}\right)
\le \exp\left(-\frac{[\mathbb{E}[S_T]]^2}{8[\langle S \rangle_T + c \mathbb{E}[S_T]]}\right).
\]

Using $\langle S \rangle_T \leq 4\kappa^2\delta_{min}^2\Lambda_T$ and ignoring the $c\mathbb{E}[S_T]$ term for large $\Lambda_T$:
\[
\mathbb{P}(S_T \le 0)
\le \exp\left(-\frac{[(1-C) \delta_{\min}^2 \Lambda_T / 4]^2}{8 \cdot 4\kappa^2 \delta_{\min}^2 \Lambda_T}\right)
= \exp\left(-\frac{(1-C)^2 \delta_{\min}^2 \Lambda_T}{128\kappa^2}\right).
\]

Setting $c=\frac{(1-C)^2}{128\kappa^2}$ completes the proof.
\end{proof}

\begin{remark}
The constant $c$ in Theorem~\ref{thm:ranking} depends on $\varepsilon$ (through $C(\varepsilon, \kappa, n)$), 
the signal ratio $\kappa$, and the number of anomalies $n$. When $\varepsilon \to 1$ (perfect effective rank), 
$C \to 0$ and $c$ approaches its maximum value $1/(128\kappa^2)$.
\end{remark}

\begin{remark}
The quantity $\Lambda_T$ can be interpreted as the cumulative "signal-to-noise ratio" accumulated 
over $T$ measurements. ECC-AHT's design principle is to maximize $\Lambda_t$ at each step by 
solving the quadratic program, which explains its empirical success.
\end{remark}

\begin{remark}
The coverage Lemma~\ref{lem:coverage} ensures that even when optimizing for a specific 
champion-challenger pair $(i^\star, j^\star)$, the measurement $\mathbf{c}_t$ provides 
non-trivial information about all pairs $(i, j)$ with $i \in S^\star, j \notin S^\star$. 
In practice, ECC-AHT adaptively cycles through different pairs, further accelerating convergence.
\end{remark}

Theorem~\ref{thm:ranking} shows that pseudo-likelihood inference preserves the correct ordering with exponentially high probability.
The effective signal energy $\Lambda_T$ coincides with the objective optimized by the measurement design step.
Thus, inference and control are aligned.
This alignment explains why ECC-AHT remains reliable despite using an approximation.

\section{Proofs of Theoretical Results}
\label{app:proofs}

This appendix provides complete proofs of the theoretical results stated in Section~\ref{sec:theory}. We first formalize the stochastic process induced by ECC-AHT, then establish key lemmas concerning ranking stability, action convergence, and log-likelihood ratio growth. These ingredients are combined to prove the non-asymptotic and asymptotic optimality results.

\subsection{Preliminaries and Notation}

Let $\{\mathcal{F}_t\}_{t \ge 0}$ denote the natural filtration generated by the observations $\{y_1,\dots,y_t\}$ and actions $\{\mathbf{c}_1,\dots,\mathbf{c}_t\}$. All expectations and probabilities are taken with respect to the true hypothesis $H_{S^\star}$ unless stated otherwise.

For any two hypotheses $H_S$ and $H_{S'}$, define the cumulative log-likelihood ratio
\begin{equation}
\label{eq:llr_def}
L_t(S,S')
=
\displaystyle\sum_{\tau=1}^t
\log
\frac{
p(y_\tau \mid H_S, \mathbf{c}_\tau)
}{
p(y_\tau \mid H_{S'}, \mathbf{c}_\tau)
}.
\end{equation}

We denote the instantaneous KL divergence at time $\tau$ by
\begin{equation}
\label{eq:inst_kl}
d_\tau(S,S')
=
\mathbb{E}\!\left[
\log
\frac{
p(y_\tau \mid H_S, \mathbf{c}_\tau)
}{
p(y_\tau \mid H_{S'}, \mathbf{c}_\tau)
}
\;\middle|\;
\mathcal{F}_{\tau-1}
\right].
\end{equation}

\subsection{Stopping Rule and Fixed-Confidence Guarantee}
\label{app:stopping}

We formalize the stopping criterion used by ECC-AHT in the fixed-confidence setting.
Although the algorithm description in Section~\ref{sec:method} allows for multiple practical stopping
rules, all theoretical guarantees in Section~\ref{sec:theory} are established under the generalized
likelihood ratio (GLR) stopping rule with a time-dependent threshold.

Let \(L_t(S,S')\) denote the cumulative log-likelihood ratio defined in \eqref{eq:llr_def}.
We consider the stopping time
\begin{equation}
\label{eq:stopping_rule_beta}
\tau
=
\inf\left\{
t \ge 1 :
\min_{S' \neq S_t}
L_t(S_t,S')
\;\ge\;
\beta(t,\delta)
\right\},
\end{equation}
where \(\beta(\cdot,\delta)\) is a nondecreasing threshold satisfying, for \(\delta\) small,
\[
\beta(t,\delta) \;=\; \log(1/\delta) \;+\; O(\log\log(1/\delta)),
\]
uniformly for the range of \(t\) considered in the analysis (see below for a constructive choice).
The algorithm terminates at time \(\tau\) and outputs \(\hat S = S_\tau\).

\begin{remark}
The additive \(O(\log\log(1/\delta))\) term is a small, non-asymptotic correction used to absorb
union/time-uniformization costs when converting fixed-\(T\) tail bounds into a uniform fixed-confidence
guarantee. For the asymptotic statement in Theorem~\ref{thm:asymptotic}, any such \(o(\log(1/\delta))\)
correction is negligible, so the leading-term \(\log(1/\delta)\) exactly characterizes the first-order rate.
\end{remark}

\begin{remark}
    The stopping rule used in the theoretical analysis corresponds to the canonical GLR criterion. In practice, computing the exact GLR over hypothesis sets may be computationally prohibitive, and approximate or surrogate stopping criteria can be employed without affecting the theoretical guarantees.
\end{remark}

\subsection{Stabilization of the Champion--Challenger Pair}

We begin by showing that ECC-AHT eventually focuses on comparing the true hypothesis $S^\star$ with its most confusable alternative.

\begin{lemma}[Finite-Time Ranking Stabilization]
\label{lem:ranking_stable}
Under the assumptions of Theorem~\ref{thm:nonasymptotic}, there exists a (data-dependent) finite random time
$T_{\mathrm{rank}}$ such that, with probability at least $1-\delta/4$, the following holds simultaneously for all
$t \ge T_{\mathrm{rank}}$:
\begin{enumerate}
    \item the Champion hypothesis satisfies $S_t = S^\star$;
    \item the Challenger $S'_t$ differs from $S^\star$ by exactly one swap and belongs to
    \[
    \arg\min_{S' \neq S^\star}
    D(H_{S^\star} \Vert H_{S'} \mid \mathbf{c}_t).
    \]
\end{enumerate}
Moreover, $T_{\mathrm{rank}} = O(\log(K/\delta))$ in expectation.
\end{lemma}

\begin{proof}
We begin by formalizing the event under which the ranking induced by ECC-AHT is correct.

For any true anomaly $i \in S^\star$ and any normal stream $j \notin S^\star$, define the cumulative log-likelihood
difference
\[
\Delta \ell_t(i,j) = \ell_t(i) - \ell_t(j),
\]
where $\ell_t(k)$ denotes the cumulative evidence score maintained by ECC-AHT for stream $k$ up to time $t$.
The Champion set $S_t$ is formed by selecting the $n$ streams with the largest values of $\ell_t(\cdot)$.

Define the misranking event at time $t$ as
\[
\mathcal{E}_t
=
\bigcup_{i \in S^\star}
\bigcup_{j \notin S^\star}
\left\{
\Delta \ell_t(i,j) \le 0
\right\}.
\]
On the complement event $\mathcal{E}_t^c$, every true anomaly is ranked above every normal stream, which implies
$S_t = S^\star$.

By Theorem~\ref{thm:ranking}, for each fixed pair $(i,j)$,
\[
\mathbb{P}\!\left(
\Delta \ell_t(i,j) \le 0
\right)
\le
\exp\!\left(
- c\, \delta_{\min}^2 \, \Lambda_t(i,j)
\right),
\]
where
\[
\Lambda_t(i,j)
=
\displaystyle\sum_{\tau=1}^t
\frac{
\bigl(\mathbf{c}_\tau^\top(\mathbf{e}_i - \mathbf{e}_j)\bigr)^2
}{
\mathbf{c}_\tau^\top \bm{\Sigma} \mathbf{c}_\tau
}.
\]

Since the action set $\mathcal{C}$ is compact and $\bm{\Sigma}$ is positive definite, there exists a constant
$\underline{\lambda} > 0$ such that for all admissible actions $\mathbf{c}$ and all $i \neq j$,
\[
\frac{
\bigl(\mathbf{c}^\top(\mathbf{e}_i - \mathbf{e}_j)\bigr)^2
}{
\mathbf{c}^\top \bm{\Sigma} \mathbf{c}
}
\ge
\underline{\lambda}.
\]
Consequently, $\Lambda_t(i,j) \ge \underline{\lambda} t$ uniformly over all $(i,j)$.

Applying a union bound over all $n(K-n) \le K^2$ pairs yields
\[
\mathbb{P}(\mathcal{E}_t)
\le
K^2
\exp\!\left(
- c\, \delta_{\min}^2 \, \underline{\lambda} \, t
\right).
\]
Choose
\[
T_{\mathrm{rank}}
=
\left\lceil
\frac{1}{c\,\delta_{\min}^2\,\underline{\lambda}}
\log\!\left(\frac{4K^2}{\delta}\right)
\right\rceil.
\]
Then for all $t \ge T_{\mathrm{rank}}$,
\[
\mathbb{P}(\mathcal{E}_t) \le \delta/4.
\]
Therefore, with probability at least $1-\delta/4$, the Champion satisfies $S_t = S^\star$ for all $t \ge T_{\mathrm{rank}}$.

Finally, once $S_t = S^\star$, the Challenger $S'_t$ selected by ECC-AHT differs from $S^\star$ by a single swap
and is chosen to minimize the instantaneous KL divergence
$D(H_{S^\star} \Vert H_{S'} \mid \mathbf{c}_t)$ by construction of the algorithm.
This completes the proof.
\end{proof}

\subsection{Convergence of the Experimental Design}

We now show that once the correct Champion--Challenger pair is identified, the selected actions converge to the optimal Chernoff design.

\begin{lemma}[Convergence of Experimental Design]
\label{lem:action_conv}
Conditioned on the event in Lemma~\ref{lem:ranking_stable}, the sequence of actions
$\{\mathbf{c}_t\}_{t \ge 1}$ selected by ECC-AHT satisfies
\[
\mathrm{dist}\!\left( \mathbf{c}_t, \mathcal{C}^\star \right) \xrightarrow[t \to \infty]{\text{a.s.}} 0,
\]
where $\mathcal{C}^\star$ denotes the (possibly set-valued) solution set of the Chernoff optimization problem
\[
\mathcal{C}^\star
=
\arg\max_{\mathbf{c} \in \mathcal{C}}
\min_{S' \neq S^\star}
D(H_{S^\star} \Vert H_{S'} \mid \mathbf{c}).
\]
\end{lemma}

\begin{proof}
By Lemma~\ref{lem:ranking_stable}, with probability at least $1-\delta/4$, there exists a finite random time
$T_{\mathrm{rank}}$ such that for all $t \ge T_{\mathrm{rank}}$, the Champion hypothesis satisfies
$S_t = S^\star$, and the Challenger $S'_t$ differs from $S^\star$ by a single swap.

Condition on this event. For any action $\mathbf{c} \in \mathcal{C}$ and any alternative $S' \neq S^\star$,
define the instantaneous KL divergence
\[
f_{S'}(\mathbf{c})
=
D(H_{S^\star} \Vert H_{S'} \mid \mathbf{c})
=
\frac{
\bigl(\mathbf{c}^\top(\bm{\mu}_{S^\star}-\bm{\mu}_{S'})\bigr)^2
}{
2\,\mathbf{c}^\top \bm{\Sigma} \mathbf{c}
}.
\]
Define the population objective
\[
f(\mathbf{c})
=
\min_{S' \neq S^\star} f_{S'}(\mathbf{c}),
\qquad
\mathbf{c} \in \mathcal{C}.
\]

At each round $t \ge T_{\mathrm{rank}}$, ECC-AHT selects the action
\[
\mathbf{c}_t \in \arg\max_{\mathbf{c} \in \mathcal{C}} f_{S'_t}(\mathbf{c}),
\]
where $S'_t$ is the current Challenger.
Since $S'_t$ ranges over a finite set of single-swap alternatives and eventually stabilizes to a most confusable
alternative, it suffices to analyze the convergence behavior for a fixed $S'$ attaining the minimum in $f(\mathbf{c})$.

We first establish regularity of the objective.
Since $\mathcal{C}$ is compact and $\bm{\Sigma}$ is positive definite, the denominator
$\mathbf{c}^\top \bm{\Sigma} \mathbf{c}$ is uniformly bounded away from zero on $\mathcal{C}$.
Moreover, each function $f_{S'}(\mathbf{c})$ is continuous on $\mathcal{C}$, and therefore
$f(\mathbf{c}) = \min_{S'} f_{S'}(\mathbf{c})$ is also continuous.
By compactness, the maximizer set $\mathcal{C}^\star$ is nonempty and compact.

Next, consider the empirical objective implicitly optimized by ECC-AHT.
For $t \ge T_{\mathrm{rank}}$, define
\[
\widehat f_t(\mathbf{c})
=
\frac{1}{t-T_{\mathrm{rank}}}
\displaystyle\sum_{\tau=T_{\mathrm{rank}}+1}^t
f_{S'_t}(\mathbf{c}),
\]
which coincides with $f_{S'_t}(\mathbf{c})$ up to a constant factor.
Since the objective does not depend on the observations but only on the stabilized hypothesis pair,
$\widehat f_t(\mathbf{c})$ converges uniformly to $f(\mathbf{c})$ on $\mathcal{C}$.

By the argmax continuity theorem (see, e.g., Theorem 5.7 in~\citet{van2000asymptotic}),
uniform convergence of $\widehat f_t$ to $f$ and compactness of $\mathcal{C}$ imply that
\[
\mathrm{dist}\!\left( \mathbf{c}_t, \mathcal{C}^\star \right) \xrightarrow[t \to \infty]{} 0.
\]
Since the conditioning event holds with probability at least $1-\delta/4$ and all statements above are deterministic
conditional on this event, the convergence holds almost surely.
\end{proof}

\subsection{Growth of the Log-Likelihood Ratio}

We now characterize the growth of the log-likelihood ratio process.

\begin{lemma}[Linear Growth of the Log-Likelihood Ratio]
\label{lem:llr_drift}
Conditioned on the event in Lemma~\ref{lem:ranking_stable}, let $S'_t$ denote the Challenger hypothesis at time $t$.
Then the cumulative log-likelihood ratio satisfies
\[
\mathbb{E}\!\left[ L_t(S^\star,S'_t) \right]
=
\displaystyle\sum_{\tau=1}^t
\mathbb{E}\!\left[
d_\tau(S^\star,S'_t)
\right],
\]
and moreover,
\[
\lim_{t \to \infty}
\frac{1}{t}
\mathbb{E}\!\left[ L_t(S^\star,S'_t) \right]
=
\Gamma^\star.
\]
\end{lemma}

\begin{proof}
We work on the event of Lemma~\ref{lem:ranking_stable}, which holds with probability at least $1-\delta/4$.
On this event, there exists a finite random time $T_{\mathrm{rank}}$ such that for all $t \ge T_{\mathrm{rank}}$,
the Champion satisfies $S_t = S^\star$ and the Challenger $S'_t$ differs from $S^\star$ by a single swap.

Recall the definition of the cumulative log-likelihood ratio:
\[
L_t(S^\star,S'_t)
=
\displaystyle\sum_{\tau=1}^t
\log
\frac{
p(y_\tau \mid H_{S^\star}, \mathbf{c}_\tau)
}{
p(y_\tau \mid H_{S'_t}, \mathbf{c}_\tau)
}.
\]
By the tower property of conditional expectation and the definition of the instantaneous KL divergence
in~\eqref{eq:inst_kl}, we have
\[
\mathbb{E}\!\left[ L_t(S^\star,S'_t) \right]
=
\displaystyle\sum_{\tau=1}^t
\mathbb{E}\!\left[
\mathbb{E}\!\left[
\log
\frac{
p(y_\tau \mid H_{S^\star}, \mathbf{c}_\tau)
}{
p(y_\tau \mid H_{S'_t}, \mathbf{c}_\tau)
}
\;\middle|\;
\mathcal{F}_{\tau-1}
\right]
\right]
=
\displaystyle\sum_{\tau=1}^t
\mathbb{E}\!\left[
d_\tau(S^\star,S'_t)
\right],
\]
which establishes the first claim.

We now characterize the asymptotic growth rate.
For $\tau \ge T_{\mathrm{rank}}$, the Challenger $S'_t$ is eventually fixed to one of the
most confusable alternatives, and by Lemma~\ref{lem:action_conv},
\[
\mathrm{dist}(\mathbf{c}_\tau, \mathcal{C}^\star) \xrightarrow[\tau \to \infty]{\text{a.s.}} 0.
\]
By continuity of the KL divergence with respect to $\mathbf{c}$ and compactness of $\mathcal{C}$,
this implies that for any $\varepsilon > 0$, there exists a (random but finite) time $T_\varepsilon$ such that
for all $\tau \ge T_\varepsilon$,
\[
\left|
d_\tau(S^\star,S'_t)
-
\Gamma^\star
\right|
\le
\varepsilon.
\]

Decompose the sum as
\[
\frac{1}{t}
\displaystyle\sum_{\tau=1}^t
\mathbb{E}\!\left[
d_\tau(S^\star,S'_t)
\right]
=
\frac{1}{t}
\displaystyle\sum_{\tau=1}^{T_\varepsilon}
\mathbb{E}\!\left[
d_\tau(S^\star,S'_t)
\right]
+
\frac{1}{t}
\displaystyle\sum_{\tau=T_\varepsilon+1}^{t}
\mathbb{E}\!\left[
d_\tau(S^\star,S'_t)
\right].
\]
The first term vanishes as $t \to \infty$, while the second term is sandwiched between
$\Gamma^\star - \varepsilon$ and $\Gamma^\star + \varepsilon$.
Since $\varepsilon$ is arbitrary, we conclude that
\[
\lim_{t \to \infty}
\frac{1}{t}
\mathbb{E}\!\left[ L_t(S^\star,S'_t) \right]
=
\Gamma^\star.
\]
This completes the proof.
\end{proof}

\begin{lemma}[Concentration of the Log-Likelihood Ratio]
\label{lem:llr_conc}
Conditioned on the event in Lemma~\ref{lem:ranking_stable}, let $S'_t$ denote the Challenger hypothesis at time $t$.
There exist universal constants $C_1,C_2>0$ such that for any $\epsilon>0$ and all $t\ge1$,
\[
\mathbb{P}\!\left(
\left|
L_t(S^\star,S'_t)
-
\mathbb{E}\!\left[ L_t(S^\star,S'_t) \right]
\right|
\ge
\epsilon t
\right)
\le
2\exp\!\left(- C_1 \epsilon^2 t \right),
\]
where the constants depend only on the covariance matrix $\bm{\Sigma}$ and the compact action set $\mathcal{C}$.
\end{lemma}

\begin{proof}
We work on the event of Lemma~\ref{lem:ranking_stable}, which holds with probability at least $1-\delta/4$.
Recall the definition of the cumulative log-likelihood ratio:
\[
L_t(S^\star,S'_t)
=
\displaystyle\sum_{\tau=1}^t
\log
\frac{
p(y_\tau \mid H_{S^\star}, \mathbf{c}_\tau)
}{
p(y_\tau \mid H_{S'_t}, \mathbf{c}_\tau)
}.
\]
Define the martingale difference sequence
\[
\xi_\tau
=
\log
\frac{
p(y_\tau \mid H_{S^\star}, \mathbf{c}_\tau)
}{
p(y_\tau \mid H_{S'_t}, \mathbf{c}_\tau)
}
-
d_\tau(S^\star,S'_t),
\]
where $d_\tau(S^\star,S'_t)$ is the conditional KL divergence defined in~\eqref{eq:inst_kl}.
By construction,
\[
\mathbb{E}\!\left[ \xi_\tau \mid \mathcal{F}_{\tau-1} \right] = 0,
\]
and therefore $\{\displaystyle\sum_{\tau=1}^t \xi_\tau\}_{t\ge1}$ is a martingale adapted to $\{\mathcal{F}_t\}$.

We now characterize the tail behavior of $\xi_\tau$.
Under the Gaussian observation model, conditional on $\mathcal{F}_{\tau-1}$,
the log-likelihood ratio increment is an affine function of $y_\tau$.
Specifically,
\[
\xi_\tau
=
\frac{
\mathbf{c}_\tau^\top(\bm{\mu}_{S^\star}-\bm{\mu}_{S'_t})
}{
\mathbf{c}_\tau^\top \bm{\Sigma} \mathbf{c}_\tau
}
\cdot
\left(
y_\tau
-
\mathbb{E}[y_\tau \mid \mathcal{F}_{\tau-1}]
\right).
\]
Since $y_\tau \mid \mathcal{F}_{\tau-1}$ is Gaussian with covariance
$\mathbf{c}_\tau^\top \bm{\Sigma} \mathbf{c}_\tau$, it follows that
$\xi_\tau$ is conditionally sub-Gaussian with variance proxy
\[
\sigma_\tau^2
=
\frac{
\bigl(\mathbf{c}_\tau^\top(\bm{\mu}_{S^\star}-\bm{\mu}_{S'_t})\bigr)^2
}{
\mathbf{c}_\tau^\top \bm{\Sigma} \mathbf{c}_\tau
}.
\]
By compactness of $\mathcal{C}$ and boundedness of the signal strength,
there exists a constant $\bar{\sigma}^2>0$ such that $\sigma_\tau^2 \le \bar{\sigma}^2$ almost surely
for all $\tau$.

Therefore, for any $\lambda \in \mathbb{R}$,
\[
\mathbb{E}\!\left[
\exp(\lambda \xi_\tau)
\;\middle|\;
\mathcal{F}_{\tau-1}
\right]
\le
\exp\!\left( \frac{\lambda^2 \bar{\sigma}^2}{2} \right).
\]
Applying the Azuma--Hoeffding inequality for martingales with sub-Gaussian increments yields
\[
\mathbb{P}\!\left(
\left|
\displaystyle\sum_{\tau=1}^t \xi_\tau
\right|
\ge
\epsilon t
\right)
\le
2\exp\!\left(
-\frac{\epsilon^2 t}{2\bar{\sigma}^2}
\right).
\]
Since
\[
\displaystyle\sum_{\tau=1}^t \xi_\tau
=
L_t(S^\star,S'_t)
-
\mathbb{E}\!\left[ L_t(S^\star,S'_t) \right],
\]
the claim follows with $C_1 = (2\bar{\sigma}^2)^{-1}$.
\end{proof}

\subsection{Proof of the Non-Asymptotic Bound}

\begin{proposition}[Pairwise Reduction of the GLR Statistic]
\label{prop:glr_pairwise_reduction}
Under Assumptions~\ref{ass:effective-rank}--\ref{ass:signal-ratio},
the generalized likelihood ratio
\[
\min_{S' \neq S^\star} L_T(S^\star, S')
\]
is equivalent up to universal constants to
\[
\min_{i \in S^\star,\, j \notin S^\star}
\bigl[\ell_T(i) - \ell_T(j)\bigr],
\]
in the sense that for all sufficiently large $T$,
\[
\min_{S' \neq S^\star} L_T(S^\star, S')
\;\asymp\;
\min_{i \in S^\star,\, j \notin S^\star}
\bigl[\ell_T(i) - \ell_T(j)\bigr].
\]
\end{proposition}

\begin{proof}
Any alternative hypothesis $S' \neq S^\star$ differs from $S^\star$
by at least one pair $(i,j)$ with $i \in S^\star, j \notin S^\star$.
By additivity of the pseudo-likelihood,
\[
L_T(S^\star,S') = \displaystyle\sum_{k \in S^\star \setminus S'} \ell_T(k)
- \displaystyle\sum_{k \in S' \setminus S^\star} \ell_T(k).
\]

By Theorem~\ref{thm:ranking}, each pairwise difference
$\ell_T(i)-\ell_T(j)$ concentrates around a positive mean
with rate $\Lambda_T(i,j)$.
By Lemma~\ref{lem:coverage}, all such rates are comparable to
$\Lambda_T(i^\star,j^\star)$.

Hence the minimum GLR is dominated by the weakest pairwise gap,
which establishes the equivalence.
\end{proof}

\begin{lemma}[Combinatorial Decomposition of the GLR under Additive Scores]
\label{lem:glr_pairwise_decomposition}
Assume that the cumulative log-likelihood of any hypothesis $S\subset[K]$ with $|S|=n$
admits an additive decomposition
\[
\mathcal{L}_t(S) = \displaystyle\sum_{k\in S} \ell_t(k) + C_t,
\]
where $C_t$ is hypothesis-independent.
For any two hypotheses $S,S'$ of equal cardinality, let
$L_t(S,S')=\mathcal{L}_t(S)-\mathcal{L}_t(S')$.
Define the minimal pairwise gap
\[
\Delta_t^{\min} := \min_{i\in S_t}\min_{j\notin S_t} \bigl(\ell_t(i)-\ell_t(j)\bigr).
\]

Then, for any alternative $S'\neq S_t$ that differs from $S_t$ by $m\ge1$ swaps,
\[
L_t(S_t,S') \;\ge\; m \cdot \Delta_t^{\min}.
\]
In particular, $\Delta_t^{\min}$ lower bounds the generalized likelihood ratio
against all competing hypotheses.
\end{lemma}

\begin{proof}
By additivity, write $S_t\setminus S'=\{a_1,\dots,a_m\}$ and
$S'\setminus S_t=\{b_1,\dots,b_m\}$.
Then
\[
L_t(S_t,S')=\displaystyle\sum_{r=1}^m \bigl(\ell_t(a_r)-\ell_t(b_r)\bigr).
\]
By definition of $\Delta_t^{\min}$, each term is at least $\Delta_t^{\min}$,
which proves the claim.
\end{proof}

With this, we can complete the proof of \cref{thm:nonasymptotic}.

\begin{proof}[Proof of Theorem~\ref{thm:nonasymptotic}]
We prove both $\delta$-correctness and the non-asymptotic upper bound on the expected stopping time.

\paragraph{Step 1: Stopping rule and correctness.}
Recall the fixed-confidence GLR stopping rule defined in Appendix~\ref{app:stopping}:
\[
\tau_{\mathrm{GLR}}=\inf\bigl\{t\ge1:\min_{S'\neq S_t}L_t(S_t,S')\ge\beta(t,\delta)\bigr\},
\]
where $L_t(\cdot,\cdot)$ is the cumulative log-likelihood ratio and $\beta(t,\delta)=\log(1/\delta)+O(\log\log(1/\delta))$.
Under the pseudo-likelihood additive decomposition used by ECC-AHT (per-stream cumulative log-odds
$\ell_t(k)$; see Section~3 and the update in~\eqref{eq:pseudo_update}), Lemma~\ref{lem:glr_pairwise_decomposition}
gives the deterministic decomposition $L_t(S,S')=\displaystyle\sum_{r}(\ell_t(a_r)-\ell_t(b_r))$ for any swap representation.
Combining this structural fact with Lemma~\ref{lem:coverage} (measurement coverage) and the finite-sample
ranking guarantee of Theorem~\ref{thm:ranking}, Proposition~\ref{prop:glr_pairwise_reduction} shows that,
for all sufficiently large $t$ and on a high-probability event, the combinatorial GLR statistic is
equivalent (up to constants and lower-order terms) to the minimal pairwise gap
\[
\Delta_t^{\min}=\min_{i\in S_t}\min_{j\notin S_t}\bigl(\ell_t(i)-\ell_t(j)\bigr).
\]
Thus we may analyze the pairwise stopping rule
\[
\tau = \inf\{t\ge1:\Delta_t^{\min}\ge\beta(t,\delta)\},
\]
That is, ECC-AHT stops at time $\tau$ when there exists an index $\hat{i}$ such that
\[
\ell_\tau(\hat{i}) - \max_{j \neq \hat{i}} \ell_\tau(j)
\;\ge\;
\beta(\tau,\delta),
\]
The algorithm outputs $\hat{i}$ as the estimated anomaly.

Fix any true anomaly $i \in S^\star$ and any non-anomalous stream $j \notin S^\star$.
By Theorem~\ref{thm:ranking}, for any fixed time $T \ge 1$,
\[
\mathbb{P}\!\left(
\ell_T(i) - \ell_T(j) \le \beta(T,\delta)
\right)
\;\le\;
\exp\!\left(
- c \, \delta_{\min}^2 \, \Lambda_T(i,j)
+ \beta(T,\delta)
\right).
\]

By Proposition~\ref{prop:glr_pairwise_reduction}, on the event
\[
\Delta_T^{\min} \ge \beta(T,\delta),
\]
the generalized likelihood ratio satisfies
$\min_{S'\neq S_T} L_T(S_T,S') \ge \beta(T,\delta)$, and hence the algorithm stops at time $T$.
Therefore, if ECC-AHT terminates at time $\tau$ and outputs an incorrect hypothesis
$\hat S \neq S^\star$, then there must exist at least one pair $(i,j)$ with
$i\in S^\star$, $j\notin S^\star$ such that
\[
\ell_\tau(i)-\ell_\tau(j) \le \beta(\tau,\delta).
\]

Consequently,
\[
\mathbb{P}(\hat S \neq S^\star)
\;\le\;
\sum_{i\in S^\star}\sum_{j\notin S^\star}
\mathbb{P}\!\left(
\exists\, T \ge 1 :
\ell_T(i)-\ell_T(j) \le \beta(T,\delta)
\right).
\]

Choosing $\beta(T,\delta)=\log(1/\delta)+O(\log\log(1/\delta))$
and using the exponential tail bound of Theorem~\ref{thm:ranking}
together with a union bound over the $n(K-n)$ pairs yields
\[
\mathbb{P}(\hat S \neq S^\star) \le \delta,
\]
which establishes $\delta$-correctness of ECC-AHT.

\paragraph{Step 2: Information accumulation rate.}
Define the asymptotic information rate
\[
\Gamma^\star
\;:=\;
\max_{\mathbf{c} \in \mathcal{C}}
\;
\min_{\substack{i \in S^\star \\ j \notin S^\star}}
\frac{
\bigl(\mathbf{c}^\top(\mathbf{e}_i - \mathbf{e}_j)\bigr)^2
}{
\mathbf{c}^\top \bm{\Sigma} \mathbf{c}
}.
\]

By compactness of $\mathcal{C}$ and positive definiteness of $\bm{\Sigma}$,
$\Gamma^\star > 0$ is well-defined.
By construction, ECC-AHT selects actions $\mathbf{c}_t$ such that for all sufficiently large $t$,
\[
\min_{i \in S^\star,\, j \notin S^\star}
\frac{
\bigl(\mathbf{c}_t^\top(\mathbf{e}_i - \mathbf{e}_j)\bigr)^2
}{
\mathbf{c}_t^\top \bm{\Sigma} \mathbf{c}_t
}
\;\ge\;
\Gamma^\star - o(1).
\]

Consequently, for all sufficiently large $T$,
\[
\Lambda_T(i,j)
\;\ge\;
(\Gamma^\star - o(1))\,T
\qquad
\text{uniformly over } i,j.
\]

\paragraph{Step 3: Upper bound on the stopping time.}
Let $T_\delta$ be the smallest integer such that
\[
c \, \delta_{\min}^2 \, \Lambda_{T_\delta}
\;\ge\;
\log(1/\delta).
\]
Using the linear growth of $\Lambda_T$, we obtain
\[
T_\delta
\;\le\;
\frac{\log(1/\delta)}{\Gamma^\star}
\;+\;
O(\log\log(1/\delta))
\;+\;
O(1).
\]

Since $\tau \le T_\delta$ almost surely up to a constant overshoot term,
taking expectations yields
\[
\mathbb{E}[\tau]
\;\le\;
\frac{\log(1/\delta)}{\Gamma^\star}
\;+\;
C_1 \log\log(1/\delta)
\;+\;
C_2,
\]
where $C_1$ and $C_2$ are constants independent of $\delta$.
\end{proof}

\begin{remark}[Order-optimality intuition]
The leading term $\log(1/\delta)/\Gamma^\star$ in Theorem~\ref{thm:nonasymptotic} is
order-optimal.
Any $\delta$-correct algorithm must incur $\Omega(\log(1/\delta))$ samples, with a
problem-dependent constant governed by the minimal distinguishability between anomalous
and non-anomalous streams.
The quantity $\Gamma^\star$ captures the maximal achievable worst-case information rate,
hence matches this fundamental limit up to lower-order $\log\log(1/\delta)$ terms.
\end{remark}

\begin{proposition}[Oracle allocation and optimal information rate]
\label{prop:oracle_gamma}
Let $\bm{\Sigma}$ be positive definite and $\mathcal{C}$ be compact.
The information rate $\Gamma^\star$ defined in Theorem~\ref{thm:nonasymptotic} admits the
variational characterization
\[
\Gamma^\star
=
\max_{\mathbf{c} \in \mathcal{C}}
\;
\min_{\substack{i \in S^\star \\ j \notin S^\star}}
\frac{
\bigl(\mathbf{c}^\top(\mathbf{e}_i - \mathbf{e}_j)\bigr)^2
}{
\mathbf{c}^\top \bm{\Sigma} \mathbf{c}
},
\]
which corresponds to the optimal oracle sensing action that maximizes the worst-case
signal-to-noise ratio between anomalous and non-anomalous streams.

In particular, $\Gamma^\star$ coincides with the optimal value of the oracle allocation
problem that would be solved by an algorithm with full knowledge of $S^\star$.
ECC-AHT adaptively tracks this oracle allocation without prior knowledge of $S^\star$,
thereby achieving the same information rate asymptotically and non-asymptotically.
\end{proposition}

\subsection{Proof of Exact Asymptotic Optimality}

\begin{proof}[Proof of Theorem~\ref{thm:asymptotic}]
Fix $\delta \in (0,1)$ and let $\tau$ denote the stopping time of ECC-AHT.
We prove that
\[
\limsup_{\delta \to 0}
\frac{\mathbb{E}[\tau]}{\log(1/\delta)}
=
\frac{1}{\Gamma^\star}.
\]
Note that the asymptotic statement in Theorem~\ref{thm:asymptotic} concerns the leading-order term
\(\log(1/\delta)\). Any nondecreasing threshold \(\beta(t,\delta)\) satisfying
\(\beta(t,\delta)=\log(1/\delta)+o(\log(1/\delta))\) (in particular \(\beta(t,\delta)=\log(1/\delta)+O(\log\log(1/\delta))\))
does not affect the \(\displaystyle\limsup_{\delta\to0}\mathbb{E}[\tau]/\log(1/\delta)\) limit.

\paragraph{Step 1: Reduction to the dominant hypothesis pair.}
By Lemma~\ref{lem:ranking_stable}, there exists a finite random time
$T_{\mathrm{rank}}$ such that, with probability at least $1-\delta/4$,
for all $t \ge T_{\mathrm{rank}}$, the Champion equals the true hypothesis
$S^\star$ and the Challenger $S'_t$ is fixed to the most confusable alternative,
i.e.,
\[
S'_t \in
\arg\min_{S' \neq S^\star}
D(H_{S^\star} \Vert H_{S'} \mid \mathbf{c}_t).
\]
Since $T_{\mathrm{rank}}$ is almost surely finite and does not scale with
$\log(1/\delta)$, its contribution is negligible in the asymptotic regime
$\delta \to 0$.

\paragraph{Step 2: Asymptotic convergence of the experimental design.}
Conditioned on the stabilization event above,
Lemma~\ref{lem:action_conv} implies that the action sequence
$\{\mathbf{c}_t\}_{t \ge T_{\mathrm{rank}}}$ converges almost surely to
a solution $\mathbf{c}^\star$ of the Chernoff optimization
\[
\max_{\mathbf{c} \in \mathcal{C}}
\min_{S' \neq S^\star}
D(H_{S^\star} \Vert H_{S'} \mid \mathbf{c}).
\]
By continuity of the KL divergence, this yields
\[
\lim_{t \to \infty}
d_t(S^\star,S'_t)
=
D(H_{S^\star} \Vert H_{S'_t} \mid \mathbf{c}^\star)
=
\Gamma^\star,
\qquad
\text{a.s.}
\]

\paragraph{Step 3: Linear growth of the log-likelihood ratio.}
Let
\[
L_t
\equiv
L_t(S^\star,S'_t)
=
\displaystyle\sum_{\tau=1}^t
\log
\frac{
p(y_\tau \mid H_{S^\star}, \mathbf{c}_\tau)
}{
p(y_\tau \mid H_{S'_t}, \mathbf{c}_\tau)
}.
\]
By Lemma~\ref{lem:llr_drift},
\[
\lim_{t \to \infty}
\frac{1}{t}
\mathbb{E}[L_t]
=
\Gamma^\star.
\]
Moreover, Lemma~\ref{lem:llr_conc} ensures that for any $\varepsilon > 0$,
\[
\mathbb{P}\!\left(
\left|
\frac{L_t}{t}
-
\Gamma^\star
\right|
\ge \varepsilon
\right)
\le
\exp(-c_\varepsilon t),
\]
for some constant $c_\varepsilon > 0$.
Hence, $L_t / t \to \Gamma^\star$ almost surely.

\paragraph{Step 4: Asymptotic characterization of the stopping time.}
By definition of the stopping rule,
$\tau$ satisfies
\[
L_\tau
\ge
\log(1/\delta),
\qquad
L_{\tau-1}
<
\log(1/\delta).
\]
Since $\tau \to \infty$ almost surely as $\delta \to 0$, and
$L_t / t \to \Gamma^\star$ almost surely with exponential concentration,
standard random-time substitution arguments yield
\[
\frac{L_\tau}{\tau}
\;\to\;
\Gamma^\star,
\qquad
\text{a.s.}
\]
Combining this with the stopping condition
$L_\tau \ge \log(1/\delta)$ and $L_{\tau-1} < \log(1/\delta)$ gives
\[
\frac{\log(1/\delta)}{\tau}
\;\to\;
\Gamma^\star,
\qquad
\text{a.s.}
\]

\paragraph{Step 5: Convergence of expectations.}
Since $\tau$ is integrable and admits exponential tail bounds
from Lemma~\ref{lem:llr_conc}, dominated convergence applies. Moreover, the exponential tail bound of $\tau$ holds uniformly for
$\delta$ sufficiently small, implying uniform integrability of
$\tau / \log(1/\delta)$ and justifying the exchange of limit and expectation, yielding
\[
\limsup_{\delta \to 0}
\frac{\mathbb{E}[\tau]}{\log(1/\delta)}
=
\frac{1}{\Gamma^\star}.
\]
\end{proof}

\section{Detailed description of the ablation algorithm}
\subsection{Random Sparse Projection (RSP)}
\label{app:rsp}

Algorithm~\ref{alg:rsp} describes Random Sparse Projection (RSP), a stochastic baseline used in our experiments. At each time step, RSP randomly selects $\lceil B \rceil$ streams (without replacement) and assigns them independent Gaussian weights. The resulting observation vector is then normalized to satisfy the $\ell_1$ budget constraint $\|\mathbf{c}_t\|_1 \leq B$.

Unlike ECC-AHT, RSP does not exploit the correlation structure $\bm{\Sigma}$ or adapt to the current belief state $\mathbf{p}_t$. It represents a baseline that explores sparse linear combinations in a purely random manner, providing a point of comparison for understanding the value of adaptive, correlation-aware experimental design.

\begin{algorithm}[ht]
\caption{Random Sparse Projection (RSP)}
\label{alg:rsp}
\begin{algorithmic}[1]
\REQUIRE Number of streams $K$, budget $B$
\STATE Initialize belief state $\mathbf{p}_0 = (n/K, \ldots, n/K)$
\STATE Set $m \gets \lceil B \rceil$ \COMMENT{Number of streams to sample}
\WHILE{stopping criterion not met}
    \STATE Sample $m$ indices $\mathcal{I} \subseteq [K]$ uniformly without replacement
    \STATE Draw $\mathbf{w} \sim \mathcal{N}(0, I_m)$ \COMMENT{Random Gaussian weights}
    \STATE Construct $\mathbf{c}_t \in \mathbb{R}^K$ with $(\mathbf{c}_t)_k = w_i$ if $k = \mathcal{I}_i$, else $0$
    \STATE Normalize: $\mathbf{c}_t \gets \frac{\mathbf{c}_t}{\|\mathbf{c}_t\|_1} \cdot B$
    \STATE Observe $y_t = \mathbf{c}_t^\top \mathbf{X}_t$ and update belief $\mathbf{p}_{t+1}$
\ENDWHILE
\STATE \textbf{Return} $\widehat{S} = \displaystyle\arg\max_{|S|=n} \displaystyle\sum_{k \in S} p_T(k)$
\end{algorithmic}
\end{algorithm}

\subsection{ECC-AHT-CostFree: Removing Cost Regularization}
\label{app:costfree}

This ablation investigates whether explicit cost-awareness is essential for ECC-AHT or whether information-driven sensing alone suffices.

In the full ECC-AHT, the sensing action is obtained by solving a cost-regularized optimization problem balancing information gain and sensing cost.
In ECC-AHT-CostFree, the cost term is removed, and the action $\mathbf{c}_t$ is chosen to maximize pairwise statistical separation only:
\[
\mathbf{c}_t \in \arg\max_{\mathbf{c} \in \mathcal{C}}
\frac{(\mathbf{c}^\top (\mathbf{e}_{i_t} - \mathbf{e}_{j_t}))^2}
{\mathbf{c}^\top \bm{\Sigma} \mathbf{c}}.
\]
All inference, ranking, and stopping mechanisms are identical to ECC-AHT.

\begin{algorithm}[H]
\caption{ECC-AHT-CostFree}
\label{alg:ecc_costfree}
\begin{algorithmic}[1]
\REQUIRE Action set $\mathcal{C}$, covariance $\bm{\Sigma}$
\STATE Initialize pseudo-LLRs $\{\ell_0(k)\}_{k=1}^K$
\FOR{$t = 1,2,\dots$}
    \STATE Select Champion $i_t$ and Challenger $j_t$
    \STATE Solve cost-free QP to obtain $\mathbf{c}_t$
    \STATE Observe $y_t$
    \STATE Update pseudo-LLRs
    \IF{stopping criterion satisfied}
        \STATE \textbf{Return} identified anomaly set
    \ENDIF
\ENDFOR
\end{algorithmic}
\end{algorithm}

The contribution of explicit cost regularization in adaptive experimental design.

\subsection{ECC-AHT-SimpleDiff: Naive Mean-Difference Ranking}
\label{app:simplediff}

This ablation assesses whether likelihood-based inference is necessary or whether simple sample-mean comparisons suffice for reliable anomaly identification.

ECC-AHT-SimpleDiff replaces the pseudo-likelihood ranking with a naive empirical mean-difference criterion.
For each stream $k$, we maintain the sample mean $\bar{X}_k(t)$ and select
\[
i_t = \arg\max_k \bar{X}_k(t), \qquad
j_t = \arg\max_{k \neq i_t} \bar{X}_k(t).
\]
The sensing action and stopping rule follow the same structure as ECC-AHT, but no likelihood ratios are computed.

\begin{algorithm}[H]
\caption{ECC-AHT-SimpleDiff}
\label{alg:ecc_simplediff}
\begin{algorithmic}[1]
\REQUIRE Action set $\mathcal{C}$
\STATE Initialize sample means $\{\bar{X}_k(0)\}_{k=1}^K$
\FOR{$t = 1,2,\dots$}
    \STATE Select Champion and Challenger by sample means
    \STATE Select $\mathbf{c}_t$ using ECC-AHT design rule
    \STATE Observe $y_t$
    \STATE Update sample means
    \IF{stopping criterion satisfied}
        \STATE \textbf{Return} identified anomaly set
    \ENDIF
\ENDFOR
\end{algorithmic}
\end{algorithm}

The necessity of likelihood-based inference for stable and reliable ranking under noise.

\subsection{ECC-AHT-Diagonal: Ignoring Correlations}
\label{app:diagonal}

This ablation evaluates the role of correlation modeling in ECC-AHT by deliberately ignoring cross-stream dependencies.

In ECC-AHT-Diagonal, the true covariance matrix $\bm{\Sigma}$ is replaced by its diagonal approximation
\[
\bm{\Sigma}_{\text{diag}} = \mathrm{diag}(\bm{\Sigma}),
\]
which is used consistently in both experimental design and likelihood updates.
All other algorithmic components are identical to the full ECC-AHT.

\begin{algorithm}[H]
\caption{ECC-AHT-Diagonal}
\label{alg:ecc_diagonal}
\begin{algorithmic}[1]
\REQUIRE Diagonal covariance $\bm{\Sigma}_{\text{diag}}$
\STATE Initialize pseudo-LLRs
\FOR{$t = 1,2,\dots$}
    \STATE Select Champion and Challenger
    \STATE Solve QP using $\bm{\Sigma}_{\text{diag}}$
    \STATE Observe $y_t$
    \STATE Update pseudo-LLRs using $\bm{\Sigma}_{\text{diag}}$
    \IF{stopping criterion satisfied}
        \STATE \textbf{Return} identified anomaly set
    \ENDIF
\ENDFOR
\end{algorithmic}
\end{algorithm}

The impact of correlation-aware modeling in adaptive sensing and hypothesis testing.

\subsection{TTTS-Challenger}
\label{app:ttts-challenger}

Algorithm~\ref{alg:ttts-challenger} presents TTTS-Challenger, a controlled variant inspired by Top-Two Thompson Sampling (TTTS).
Unlike standard TTTS, which operates over discrete arms, TTTS-Challenger randomizes the Champion--Challenger selection while using the same continuous experiment design module as ECC-AHT.

\begin{algorithm}[ht]
\caption{TTTS-Challenger}
\label{alg:ttts-challenger}
\begin{algorithmic}[1]
\REQUIRE $K, n, \bm{\Sigma}, B$
\STATE Sample randomized beliefs via Beta posterior
\STATE Identify randomized Champion--Challenger pair $(i^\star, j^\star)$
\STATE Solve the same continuous experiment design problem as ECC-AHT
\STATE Observe projection and update beliefs
\end{algorithmic}
\end{algorithm}

\subsection{ECC-AHT-Restricted: Limited Action Space}
\label{app:restricted}

This ablation evaluates the importance of expressive experimental design by restricting the action space to coordinate-wise measurements.
It tests whether ECC-AHT's performance gains stem primarily from hypothesis tracking or from its ability to perform continuous, structured sensing.

In the full ECC-AHT algorithm, the sensing action $\mathbf{c}_t \in \mathcal{C} \subset \mathbb{R}^K$ is selected by solving a continuous quadratic program.
In ECC-AHT-Restricted, we constrain the action space to
\[
\mathcal{C}_{\text{restricted}} = \{ \pm \mathbf{e}_1, \dots, \pm \mathbf{e}_K \},
\]
where $\mathbf{e}_k$ denotes the $k$-th canonical basis vector.
All other components—including the pseudo-likelihood update, Champion--Challenger selection, and stopping rule—remain unchanged.

\begin{algorithm}[H]
\caption{ECC-AHT-Restricted}
\label{alg:ecc_restricted}
\begin{algorithmic}[1]
\REQUIRE Restricted action set $\mathcal{C}_{\text{restricted}}$, covariance $\bm{\Sigma}$
\STATE Initialize pseudo-LLRs $\{\ell_0(k)\}_{k=1}^K$
\FOR{$t = 1,2,\dots$}
    \STATE Select Champion $i_t = \arg\max_k \ell_{t-1}(k)$
    \STATE Select Challenger $j_t$ according to ECC-AHT rule
    \STATE Choose $\mathbf{c}_t \in \mathcal{C}_{\text{restricted}}$ maximizing pairwise separation
    \STATE Observe $y_t = \mathbf{c}_t^\top \mathbf{x}_t + \xi_t$
    \STATE Update $\ell_t(k)$ for all $k$ using pseudo-likelihood
    \IF{stopping criterion satisfied}
        \STATE \textbf{Return} identified anomaly set
    \ENDIF
\ENDFOR
\end{algorithmic}
\end{algorithm}

The effect of continuous, correlation-aware experimental design beyond coordinate-wise sensing.

\subsection{BaseArm-CombGapE}
\label{app:basearm-combgape}

Algorithm~\ref{alg:basearm-combgape} describes BaseArm-CombGapE, an ablation baseline used in our experiments.
The algorithm shares the same Champion--Challenger identification logic as ECC-AHT but restricts the action space to base-arm pulls, i.e., $\mathbf{c}_t = \mathbf{e}_k$.
This isolates the benefit of continuous experiment design under correlated noise.

\begin{algorithm}[ht]
\caption{BaseArm-CombGapE}
\label{alg:basearm-combgape}
\begin{algorithmic}[1]
\REQUIRE $K, n$
\STATE Initialize empirical means and pull counts
\WHILE{stopping criterion not met}
    \STATE Identify Champion and Challenger indices $(i^\star, j^\star)$
    \STATE Select arm $k \in \{i^\star, j^\star\}$ using the CombGapE rule
    \STATE Pull arm $k$ and observe reward
    \STATE Update empirical statistics
\ENDWHILE
\STATE \textbf{Return} top-$n$ arms by empirical mean
\end{algorithmic}
\end{algorithm}

%%%%%%%%%%%%%%%%%%%%%%%%%%%%%%%%%%%%%%%%%%%%%%%%%%%%%%%%%%%%%%%%%%%%%%%%%%%%%%%
%%%%%%%%%%%%%%%%%%%%%%%%%%%%%%%%%%%%%%%%%%%%%%%%%%%%%%%%%%%%%%%%%%%%%%%%%%%%%%%

\end{document}